%% file: arkiv.tex
\documentclass{article}
\PassOptionsToPackage{authoryear,round}{natbib}
\PassOptionsToPackage{table}{xcolor}

\usepackage[T1]{fontenc}

\usepackage{microtype}
\usepackage{graphicx}
\usepackage{subcaption}
\usepackage{booktabs} 
\usepackage{algorithm}
\usepackage{algorithmic}
\usepackage{comment}
\usepackage[most]{tcolorbox}
\usepackage{xparse}
\usepackage{booktabs}
\usepackage{tabularx}
\usepackage{array}
\usepackage{placeins}
\usepackage{tcolorbox}
\usepackage{minted}
\usepackage{xcolor} 
\usepackage{listings}

\lstdefinestyle{promptlist}{
    basicstyle=\ttfamily\footnotesize,
    breaklines=true,
    breakatwhitespace=false,
    columns=fullflexible,
    keepspaces=true,
    showstringspaces=false,
    upquote=true
}
\definecolor{citationblue}{RGB}{0,32,128}
\definecolor{pystringred}{HTML}{BA2121}
\usepackage[breaklinks,colorlinks=true,citecolor=citationblue]{hyperref}       

\usepackage[preprint]{arxiv}

\usepackage{amsmath}
\usepackage{amssymb}
\usepackage{mathtools}
\usepackage{amsthm}
\usepackage{enumitem}
\usepackage{aliascnt}

\usepackage[capitalize,noabbrev]{cleveref}

\makeatletter
\renewcommand{\paragraph}{%
  \@startsection{paragraph}{4}%
  {\z@}{0ex}{-1em}%
  {\normalfont\normalsize\bfseries}%
}
\makeatother

\setlist[itemize,1]{itemsep=0.2\lineskip,topsep=0.2\lineskip,leftmargin=2.3em}
\setlist[itemize,2]{itemsep=0.2\lineskip,topsep=0.2\lineskip,leftmargin=1.4em}
\setlist[itemize,3]{itemsep=0.2\lineskip,topsep=0.2\lineskip,leftmargin=1.4em}

\setlist[enumerate,1]{itemsep=0.2\lineskip, topsep=0.2\lineskip, leftmargin=2.3em}
\setlist[enumerate,2]{itemsep=0.2\lineskip,topsep=0.2\lineskip,leftmargin=2em}
\setlist[enumerate,3]{itemsep=0.2\lineskip,topsep=0.2\lineskip,leftmargin=2em}

\input{math_commands}

\input{symbols}

\theoremstyle{plain}
\newtheorem{theorem}{Theorem}[section]

\newaliascnt{lemma}{theorem}
\newtheorem{lemma}[lemma]{Lemma}
\aliascntresetthe{lemma}

\newaliascnt{proposition}{theorem}

\aliascntresetthe{proposition}

\newaliascnt{corollary}{theorem}

\aliascntresetthe{corollary}

\theoremstyle{definition}
\newaliascnt{definition}{theorem}
\newtheorem{definition}[definition]{Definition}
\aliascntresetthe{definition}

\newaliascnt{assumption}{theorem}
\newtheorem{assumption}[assumption]{Assumption}
\aliascntresetthe{assumption}

\crefname{assumption}{assumption}{assumptions}
\Crefname{assumption}{Assumption}{Assumptions}

\crefrangeformat{assumption}{assumptions~#3#1#4--#5#2#6}
\Crefrangeformat{assumption}{Assumptions~#3#1#4--#5#2#6}

\theoremstyle{remark}
\newaliascnt{remark}{theorem}
\newtheorem{remark}[remark]{Remark}
\aliascntresetthe{remark}
\newcommand{\Attn}{\textsf{Attn}}
\newcommand{\Unif}{\mathrm{Unif}}

\newcommand{\wembed}[1]{\vw_{\texttt{#1}}}
\newcommand{\vembed}[1]{\vv_{\texttt{#1}}}
\newcommand{\qembed}[1]{\vq_{\texttt{#1}}}
\newcommand{\kembed}[1]{\vk_{\texttt{#1}}}
\newcommand{\eembed}[1]{\ve_{\texttt{#1}}}
\newcommand{\symI}{\texttt{<I:>}}
\newcommand{\symO}{\texttt{<O:>}}
\newcommand{\ttC}{\texttt{C}}
\newcommand{\Str}{\textsf{Str}}
\newcommand{\Aref}[1]{Appendix \ref{#1}}
\newcommand{\symn}{\texttt{<\textbackslash n>}}
\newcommand{\MASK}{\mathrm{MASK}}
\newcommand{\RMSNorm}{\mathrm{RMSNorm}}

\newcommand{\haodong}[1]{[{\color{cyan}\textbf{Haodong: #1}}]}

\newcommand{\eos}{\texttt{<EOS>}}

\usepackage[textsize=tiny]{todonotes}

\title{Frontier Language Models Struggle to Copy: \\
Text Can Be Better Viewed in 2D}

\author{%
  Haodong Wen\footnotemark[1] \quad
  Yiran Zhang\footnotemark[1] \quad
  Yingfa Chen\footnotemark[1] \quad
  Kaifeng Lyu \\ 
  Tsinghua University \\
  \texttt{\{whd25,zhangyir22,yingfa-c24\}@mails.tsinghua.edu.cn}\\
  \texttt{klyu@mail.tsinghua.edu.cn}
}

\begin{document}

\maketitle

{
\renewcommand{\thefootnote}{\fnsymbol{footnote}}
\footnotetext[1]{Equal contribution}
}

\renewcommand{\thefootnote}{\arabic{footnote}}
\setcounter{footnote}{0}

\begin{abstract}
  While large language models (LLMs) can solve advanced reasoning problems in seconds, we show that even frontier models fail to perform a much simpler operation: exactly copying an input string that lies well within their context windows.
  We attribute this failure to positional encodings in Transformer architectures, whose inductive bias favors copying through a shortcut based on matching local contexts rather than carefully locating the corresponding input positions.
  To address this issue, we introduce 2D-RoPE, which organizes text into a 2D grid rather than a 1D sequence and assigns each token a row ID and a column ID.
  Under this view, copying becomes simply retrieving input tokens at a fixed column offset, which makes the task easy to learn.
  In synthetic copy experiments, shallow Transformers with 2D-RoPE achieve perfect copying at input lengths hundreds of times longer than those seen during training, whereas standard positional encodings fall far behind.
  We further show that the advantage of 2D-RoPE language models on copy tasks consistently holds in large-scale pretraining on DCLM with model sizes up to $1.4$B parameters. 
  Overall, our results suggest that viewing text in 2D can benefit language modeling, and we hope this encourages future work to further explore the potential of 2D positional encodings. Our code is available at \href{https://github.com/hhhhhh-925/copy-2dRoPE}{https://github.com/hhhhhh-925/copy-2dRoPE}.
\end{abstract}

\input{intro}
\input{related_work}
\input{preliminary}
\input{rope_discussion}
\input{theoretical_analysis}

\input{experiments}

\input{discussion}
\input{conclusion}

\bibliography{ref}
\bibliographystyle{plainnat}

\newpage

\tableofcontents

\newpage

\appendix

\input{appendix}


\end{document}

%% file: math_commands.tex
\usepackage{amsmath,amsfonts,bm}

\def\1{\bm{1}}

\def\vtheta{{\bm{\theta}}}

\def\vb{{\bm{b}}}

\makeatletter
\newcommand{\ve}{\@ifnextchar\bgroup{\velong}{{\bm{e}}}}
\newcommand{\velong}[1]{{\bm{#1}}}
\makeatother

\def\vh{{\bm{h}}}

\def\vk{{\bm{k}}}

\def\vq{{\bm{q}}}

\def\vv{{\bm{v}}}
\def\vw{{\bm{w}}}

\def\vz{{\bm{z}}}

\def\mA{{\bm{A}}}

\def\mH{{\bm{H}}}

\def\mK{{\bm{K}}}

\def\mQ{{\bm{Q}}}
\def\mR{{\bm{R}}}

\def\mV{{\bm{V}}}
\def\mW{{\bm{W}}}
\def\mX{{\bm{X}}}
\def\mY{{\bm{Y}}}
\def\mZ{{\bm{Z}}}

\DeclareMathAlphabet{\mathsfit}{\encodingdefault}{\sfdefault}{m}{sl}
\SetMathAlphabet{\mathsfit}{bold}{\encodingdefault}{\sfdefault}{bx}{n}

\newcommand{\E}{\mathbb{E}}
\newcommand{\R}{\mathbb{R}}

\newcommand{\softmax}{\mathrm{softmax}}

\DeclareMathOperator*{\argmax}{arg\,max}
\DeclareMathOperator*{\argmin}{arg\,min}



%% file: symbols.tex
\newcommand{\cD}{\mathcal{D}}

\newcommand{\cL}{\mathcal{L}}

\newcommand{\cS}{\mathcal{S}}

\newcommand{\cV}{\mathcal{V}}

\newcommand{\linf}{\ell_{\infty}}
\newcommand{\vbeta}{{\bm \beta}}

\newcommand{\norm}[1]{\left\|#1\right\|}

\newcommand{\normtwo}[1]{\norm{#1}_2}

\newcommand{\diag}{\mathrm{diag}}


%% file: intro.tex
\section{Introduction}\label{sec:intro}

While large language models (LLMs) have achieved tremendous success on a wide range of complex reasoning problems, 
we show that they still struggle with a seemingly simple task: given a string as input, reproduce the exact same string as output, possibly with minor adaptations to match a required output format.
We refer to this as the \textit{copy} task and specifically consider the following two variants:
\begin{itemize}
    \item \textbf{Binary Copy.} We give the model a simple string consisting of two types of tokens (such as $0$ and $1$) and the model is instructed to output the exact same sequence of tokens.
    \item \textbf{Python List Conversion.} We generate a comma-separated list of data points from synthetic physical experiments, feed it to the LLM, and ask the LLM to convert the data into a Python list that can be used to generate plotting code for the data.
\end{itemize}
Although solving these copy tasks may not require the same level of intelligence as an LLM, it is arguably a fundamental capability that any intelligent agent should possess.
For example, an agent may need to copy parameters from a configuration file for function calling, or organize unstructured user input into a specific format for downstream processing.
After all, if current LLMs are already intelligent enough to achieve gold-medal performance in Olympiad-level mathematics and competitive programming~\citep{openai2025competitive,deepmind_gemini_imo_2025,deepmind_gemini_icpc_2025}, why shouldn't we expect them to perform this very basic copying task perfectly?


However, as shown in~\Cref{fig:api-acc}, all the frontier LLMs we evaluate, including GPT-5.5, Gemini 3.1 Pro, and DeepSeek V4 Pro, fail on a substantial fraction of input strings that are well within their context lengths.
Their copying accuracy drops further on longer inputs, often falling well below $50\%$.




\begin{figure}[t]
    \includegraphics[width=0.9\linewidth]{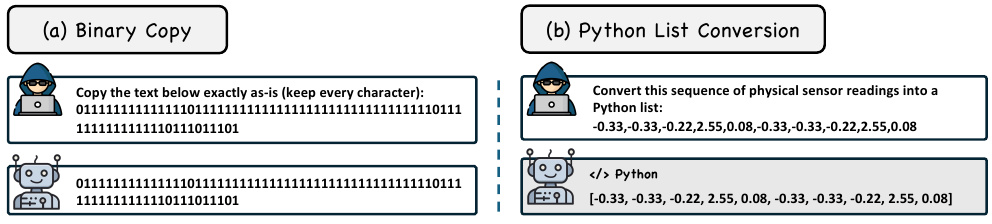}
    \centering
    \includegraphics[width=\textwidth]{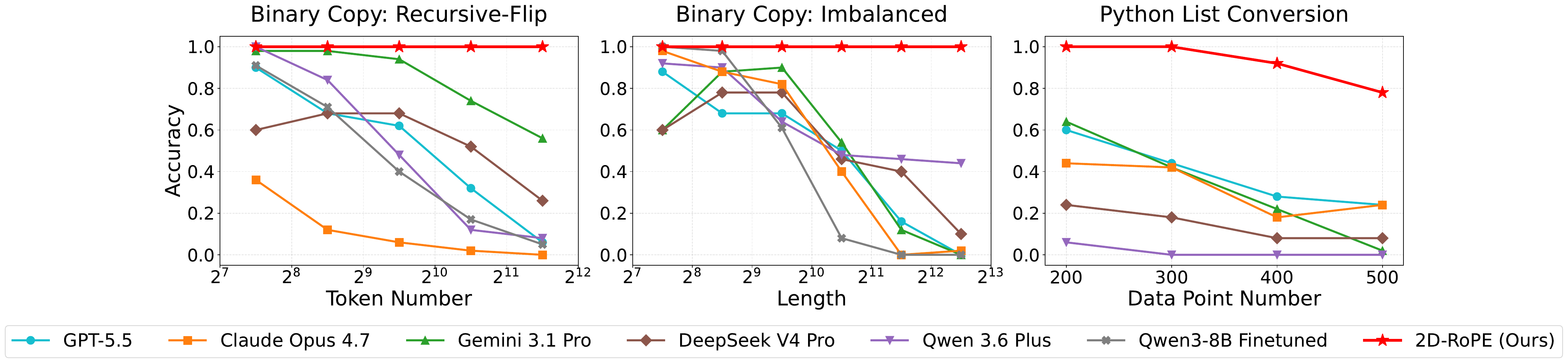}
    \caption{Frontier language models struggle to perform two representative copy tasks, the binary copy and Python list conversion tasks, while our 2D-RoPE models achieve significantly better performance. See~\Cref{sec:copy-benchmarks} for details on the benchmark construction.
    }
    \label{fig:api-acc}
\end{figure}

\paragraph{Understanding the Failure of Copying.}
Our experimental design is motivated by the following conjecture: LLMs may not copy a string by directly using the absolute index $i$ to retrieve the $i$-th character from the input. Instead, copying may be implemented in a way closer to the induction head mechanism~\citep{olsson2022context,chen2024unveiling}: the model searches for a previous occurrence of a similar local context and then predicts the token that appears right after this occurrence.

Under this view, repeated substrings create ambiguous local matches and can therefore interfere with exact copying. Indeed, in our experiments, the tested strings are not arbitrary but are deliberately chosen to contain repeated substrings. For example, in the Python List Conversion task, part of the data is generated from periodic sine or triangular waves.
Our further experiments in~\Cref{fig:repeat-acc} show that copying accuracy is negatively correlated with the degree of repetition in the input strings. 

In~\Cref{sec:LLM-faliure-why-rope}, we develop expressivity theory and mechanistic interpretability results to support this conjecture, and show that the copying failure is closely tied to the positional encoding used in Transformers. Although RoPE has appealing properties such as shift-invariance, it does not provide a sufficiently strong inductive bias for the model to align the source and target tokens by their relative positions.
As a result, the learned attention patterns are not cleanly aligned with the corresponding source positions, but are instead easily attracted by locally similar substrings.

\paragraph{Our Method: 2D Positional Encoding.}
This analysis points to a natural design goal: a positional encoding with a stronger inductive bias towards copying from the corresponding position in the input. 
To this end, a key observation is that the copy task has an implicit 2D structure.
When copying an input string $x_1,\ldots,x_n$ into an output string $y_1,\ldots,y_n$, if we only view the whole context as a 1D sequence, for example:
\[
    \texttt{<Input>}, \texttt{<:>}, x_1, \ldots, x_n, \symn, \texttt{<Output>}, \texttt{<:>}, y_1, \ldots, y_n,
\]
then, producing an output token $y_k$ requires the model to retrieve $x_k$ from an offset that depends on the input length $n$.
This is unnatural for relative positional encodings such as RoPE, which more readily represent fixed-offset retrieval than a length-dependent retrieval.
It is thus no wonder that models may fall back to an easier shortcut such as matching similar substrings.
Instead, if we arrange the input and output strings into two rows, for example:
\begin{align*}
    &\texttt{<Input>}, \texttt{<:>}, x_1, \ldots, x_n, \symn, \\
    &\texttt{<Output>}, \texttt{<:>}, y_1, \ldots, y_n,
\end{align*}
then, producing $y_k$ reduces to retrieving a token from the same column, or from a column shifted by a fixed constant, depending on the prompt template.

Motivated by this key observation, we introduce 2D-RoPE. The construction is inspired by how RoPE is applied to 2D images in vision models~\citep{jeevan2022resource,heo2024rotary,chu2024visionllama}. Given any text, 2D-RoPE uses the line break token $\symn$ (also called the newline token) as the row separator and views the text as a 2D grid, where each token is assigned a 2D position ID pair $(\text{row ID},\text{column ID})$.
RoPE attention is then applied to this 2D grid to encode relative positions along both axes.
Under this representation, copying a string can be implemented by a single attention layer that retrieves each token from the corresponding position in the previous row, which is arguably a simpler structure for the model to learn.

Empirically, 2D-RoPE leads to strong length generalization on the copy task. Here, length generalization refers to the ability to generalize from short training sequences to longer test sequences. In synthetic copy experiments, we train 2D-RoPE models on the binary copy task and show that one-layer models achieve perfect copying with \textbf{up to $1000\times$ length generalization}, while $12$-layer models maintain perfect performance up to $100\times$ length generalization. 


Theoretically, we first prove that a one-layer Transformer equipped with 2D-RoPE can represent the binary copy problem and achieve length generalization. Furthermore, we show that the global minima of a one-layer 2D-RoPE, trained on sequences of length at most $L$, can generalize to sequence lengths polynomial in $L$.



Beyond this synthetic setting, we further ask whether 2D-RoPE retains its
advantage when trained as a real language model rather than on the copy task
alone. We therefore pretrain 2D-RoPE models ranging from $350$M to $1.4$B
parameters on DCLM~\citep{dclm}. After finetuning, these models exhibit
both length generalization and out-of-distribution generalization on copy tasks, while
keeping the performance on common-sense reasoning tasks comparable to RoPE.

Finally, since the 2D structure in natural documents may not always be explicitly marked by line breaks, we introduce Auto-2D-RoPE, which learns a data-dependent transformation to automatically assign 2D coordinates to each token. Our experiments show that Auto-2D-RoPE maintains length generalization on binary copy even without line breaks, whereas the original 2D-RoPE does not.

\begin{itemize}
    \item We show that frontier LLMs fail on the copy task, and develop expressivity theory and mechanistic interpretability results to explain this failure (\Cref{sec:LLM-faliure-why-rope}).
    \item We introduce 2D-RoPE, a positional encoding that views text as 2D rather than 1D. Empirically, we show that training 2D-RoPE models specifically on the copy task leads to strong length generalization. We further theoretically show that this 2D-RoPE improves model expressivity and leads to a better loss landscape for learning the copy task with length generalization. More broadly, this may shed light on how positional encodings shape a model's learning behavior beyond the specific setting studied here~(\Cref{sec:main}).
    \item We further conduct LLM pretraining experiments on DCLM~\citep{dclm} dataset with model sizes ranging from $350$M to $1.4$B. Our results show that 2D-RoPE can exploit the 2D structure of text and significantly improve performance on the copy tasks with comparable common sense reasoning performance~(\Cref{sec:llm-experiments}).
    \item To reduce the explicit reliance of 2D-RoPE on line breaks, we introduce Auto-2D-RoPE, which automatically determines the 2D coordinates for each token and length-generalizes on the copy task even without line breaks~(\Aref{app:adaprope}).
\end{itemize}
Overall, our results suggest that viewing text in 2D can benefit language modeling, and we hope this encourages future work to further explore the potential of 2D positional encodings.

%% file: related_work.tex
\section{Related Works}
\paragraph{Positional Encoding.} 
The attention mechanism relies on position encoding methods to model the positional information of tokens~\citep{vaswani2017attention}.
Most state-of-the-art LLMs employ RoPE~\citep{su2024roformer}, which has demonstrated strong language modeling performance and length generalization.
\citet{kazemnejad2023impact} show that attention with just a causal mask can reconstruct absolute and relative position information.
More recently, \citet{yang2026path} propose a position encoding scheme based on accumulating Householder transformations and show that it can solve $NC^1$-complete problems under $AC^0$-reductions. 

\paragraph{Structured and Multidimensional Positional Encodings.}
Prior work has explored positional encodings that assign multiple structural coordinates to each token.  BiPE~\citep{hetwo2024} decomposes the input sequence into segments, where the segment boundaries are constructed with full stops and line breaks. It then applies absolute positional embeddings within each segment and uses RoPE or ALiBi across segments, which is connected to bilevel nondeterministic finite automata (BiNFA) in theory. For coding-related tasks, HiRoPE~\citep{zhang2024hirope} constructs hierarchical RoPE coordinates from the syntactic structure of source code. 
Although these methods introduce hierarchical positional structure, their constructions do not provide the inductive bias required for copying.
For BiPE, the absolute positional encoding within one segment does not provide the relative alignment along the inner coordinate required for copying.
HiRoPE, on the other hand, relies on the hierarchical structure inherent in coding data and therefore does not directly apply to copying task. 
In computer vision, RoPE has been applied along the height and width coordinates of pixels or image patches~\citep{jeevan2022resource, heo2024rotary, chu2024visionllama}. To the best of our knowledge, prior work has not systematically studied how positional encoding affects the copying abilities of language models.

\paragraph{Length Generalization.} Many early studies show that transformers sometimes succeed and sometimes fail at length-generalization, correlated with multiple factors such as positional encoding, data format, and other training hyperparameters~\citep{bhattamishra2020ability,anil2022exploring,kazemnejad2023impact,awasthi2023improving,jelassi2023length,wang2024length,zhou2023algorithms,zhou2024transformers,changlanguage,jelassi2024repeat}. Specifically on the copy task, Transformers trained
to copy short strings often do not generalize well to longer strings when the input string includes repeated substrings~\citep{zhou2024transformers,morwanifeature}. Theoretically, \citet{huangformal} introduce a formal framework, proving that Transformers can solve a class of problems expressible by the C-RASP formalism~\citep{yangcounting}. Unfortunately, the repeated copy is provably difficult in the sense of C-RASP expressiveness.

\paragraph{The Representation Power of Transformer.} There is a long line of work studying the representation power of Transformer~\citep{perez2021attention,yao2021self,chiang2022overcoming,sanford2023representational}. \citet{merrillexpressive,feng2023towards} analyze the representation power of Transformer with Chain-of-Thought. \citet{wenrnns} compared the expressive ability of Transformer with Recurrent Neural Networks. Later, \citet{chen2024theoretical} showed the limitation of multi-layer Transformers' expressive power. 

\paragraph{Copy Capability of Transformers.}
Copying from context has been studied in several related settings. Some works use synthetic copy tasks to study Transformer length generalization~\citep{kazemnejad2023impact,huangformal,jelassi2024repeat}, while other benchmarks include copying as part of broader LLM evaluations~\citep{chen2025icleval,wang2024needle,hsieh2024ruler,wang2025stringllm}. However, these works do not isolate exact copying of user-provided strings as a standalone failure mode of frontier LLMs, nor do they provide a positional-encoding-based remedy. In contrast, we directly evaluate modern LLMs on exact copy tasks, connect the failure to positional alignment, and propose 2D-RoPE to better align source and target tokens.

\paragraph{Comparison with \citet{jelassi2024repeat}.} \citet{jelassi2024repeat} is most closely related to our work, which studies synthetic copy tasks and gives a two-layer Transformer construction based on context matching. Our focus is different in several aspects. First, we show that frontier LLMs fail to copy user-provided strings exactly, even within their context lengths. Second, their construction and analysis focus on the uniform data distribution and do not imply perfect copying for arbitrary strings with repeated substrings or provide the length-generalization guarantee considered in our work. Finally, their construction is context-based (see~\Cref{subsec:why}), whereas our results suggest that repeated substrings make context matching unreliable.

\paragraph{Comparison with \citet{bai2021segatron}.}
Segatron-XL assigns each token three hierarchical positions: its position within the sentence, the position of that sentence within the paragraph, and the position of that paragraph within the document. It encodes relative differences along these three dimensions using the relative attention mechanism of Transformer-XL~\citep{dai2019transformer}. If the input and output strings are placed in separate segments with aligned local token positions, corresponding tokens can have a fixed relative coordinate that does not depend on the sequence length. Segatron-XL may therefore provide a positional inductive bias similar to that of 2D-RoPE for copying. However, Segatron-XL was introduced within the Transformer-XL architecture and was evaluated only for language modeling; its behavior on exact copying and length generalization was not studied. In contrast, our work shows that even frontier language models struggle with exact copying and provides 2D-RoPE to solve it. We further introduce Auto-2D-RoPE, which learns two-dimensional position structure directly from text.

%% file: preliminary.tex
\section{Language Models Fail to Copy}\label{sec:LLM-faliure-why-rope}
In this section, we first introduce our copy benchmarks in detail, and then present our analysis to diagnose the failure of LLMs on the copy tasks.

\begin{figure}[t]
    \centering
    \includegraphics[width=\linewidth]{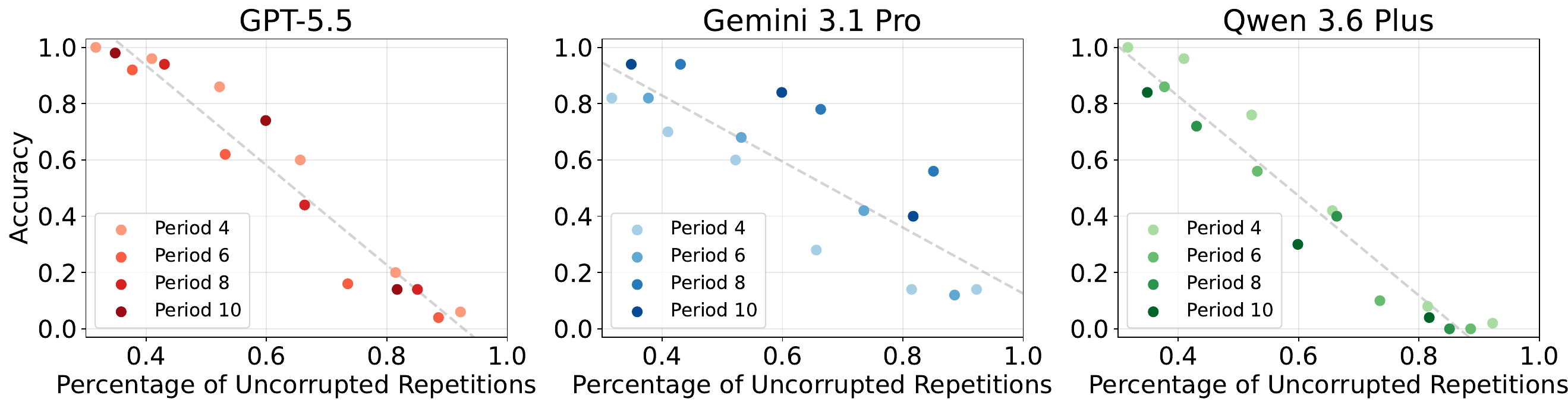}
    \caption{\textbf{Repeat-structure copy test for frontier LLMs.} Each input is
built by repeating a short base block (period $4$--$10$) until length $1000$, and
then independently flipping each token with some probability. The $x$-axis,
$(1-p)^m$, is the fraction of base blocks that survive uncorrupted, so a larger
value means the input contains more exact repetitions; the $y$-axis is the exact
copying accuracy. Across all three models and all four periods, accuracy
decreases as repetitions are preserved more strongly, indicating that repeated
substrings make exact copying harder. See \Aref{sec:repeat-test} for details.}
    \label{fig:repeat-acc}
\end{figure}


\subsection{Our Copy Benchmarks}
\label{sec:copy-benchmarks}
To test the copy capabilities of LLMs, we construct the following copy tasks. See example prompts in \Cref{fig:api-acc} and full prompt templates used in our experiments in \Aref{appendix:prompts}.

\paragraph{Binary Copy Task.} In binary copy, 
we fix a vocabulary $\cV$ of size $|\cV| =2$, which means that the input string $s$ only consists of two types of tokens.
Given an input string $s$, the model is required to output an exact copy of $s$. For example, 
the input string can be $s=``\texttt{011010101010100}"$, with the vocabulary fixed to $\cV =\{\texttt{0},\texttt{1}\}$.
We generate such input strings from three different distributions:
\begin{itemize}
\item \textbf{Uniform Generation:} Each token is sampled independently and uniformly from $\cV$.
\item \textbf{Imbalanced Generation:} We fix a probability set $p \in \{0.05, 0.15, 0.3, 0.5, 0.7, 0.85, 0.95\}$. For each string, we first sample $p$ uniformly from this set, and then each input token is sampled as the first token in $\cV$ with probability $p$ and the second token with probability $1-p$.
\item \textbf{Recursive-Flip Generation:} We initialize $s$ with a single token sampled uniformly from $\cV$. Then we iteratively update $s \leftarrow s + c + s$ for several rounds, where $c$ is a randomly sampled token from $\cV$ in each round.
After $K$ rounds, we obtain a string $s$ with length $2^{K+1} -1$, which contains a complicated recursive structure as well as many random flips. When we need to generate a string with an arbitrary length $n$, we simply truncate $s$ to its first $n$ tokens.
\end{itemize}

\paragraph{Python List Conversion Task.} 
We generate a comma-separated list of data points from synthetic physical experiments and ask the LLM to convert the data into a Python list. This mimics the scenario where an agent needs to copy user inputs into a specific format for downstream processing with Python code, such as visualizing the data.
Our evaluation includes 9 synthetic experiments in total, and each data-point list is generated randomly from one of the $9$ experiments.
These experiments can be further categorized into three types: smooth periodic oscillations, piecewise or reset-driven oscillations, and nonlinear or pulse-like signal responses. See \Aref{app:api-details} for details. 

\paragraph{Qwen3-8B Finetuning Baseline.} We also finetune Qwen3-8B~\citep{yang2025qwen3} on the same
imbalanced binary copy data and exactly the same setups as we used for our other pretrained models in \Cref{sec:llm-experiments},
denoted as \emph{Qwen3-8B Finetuned} in \Cref{fig:api-acc}.
Still, we found that the copy failure persists even after finetuning.

%% file: rope_discussion.tex
\subsection{Evaluation of Frontier LLMs}
\Cref{fig:api-acc} shows that frontier LLMs fail to perform perfect copying on both binary copy and Python list conversion tasks.
For binary copy, we test the models on Recursive-Flip and Imbalanced data.
For each length interval $[2^k,2^{k+1}-1)$ with $k\in\{7,8,9,10,11\}$, we generate $50$ samples with lengths sampled uniformly from that interval, and report the average copying accuracy within each interval. 
For Recursive-Flip, the vocabulary is either \{``\texttt{0}'', ``\texttt{1}''\} or \{``\texttt{000}'', ``\texttt{111}''\}, depending on whether the model's tokenizer tends to merge consecutive $01$ characters into a single token.
For Imbalanced data, we fix the same vocabulary for all models: \{``\texttt{ a}'', ``\texttt{ b}''\}, since all the tokenizers used in our evaluation recognize these two tokens as unique tokens.
For the Python List Conversion task, we evaluate four list lengths, $200,300,400,500$, with $50$ samples for each length.
See details of the evaluation in \Aref{app:api-details}.


There are two notable features in the results. First, as the input length increases, the accuracy of all models decreases.
This suggests that the models may only have learned an ad-hoc copy rule that works for some specific input length range, rather than a rule that can be generalized to all lengths.
Throughout the paper, we use the term \emph{length generalization} to refer to the ability of a model to perform well on inputs with lengths beyond the training range.

Second, frontier LLMs perform especially poorly on strings with many repeated substrings.
To see this, we sample a short binary base string, repeat it multiple times to generate a fixed-length long string, and independently flip each token with probability $1-p$. As we increase $p$ from $0$ to $1$, each repeated substring becomes less likely to be corrupted, and thus the input contains more repeated substrings.
As shown in \Cref{fig:repeat-acc}, the copying accuracy consistently decreases as the number of repeated substrings increases in this way.
Details of this experiments can be found in \Aref{sec:repeat-test}.

\subsection{Why Do LLMs Fail to Copy?}\label{subsec:why}

Motivated by the above observations, we now state our conjecture on why frontier LLMs fail to copy.
There are two natural algorithms for solving the copy task:
\begin{itemize}
    \item \emph{Position-based algorithm}: to output the $k$-th output token $y_k$, derive some representation of the index $k$ and use it to retrieve the corresponding input token $x_k$.
    \item \emph{Context-based algorithm}: first identify the local context around the current output token $y_k$, such as $y_{k-3}y_{k-2}y_{k-1}$, and then find a similar local context in the input, and then copy the token associated with that matched context.
\end{itemize}
The position-based algorithm sounds the most natural, but implementing it in a Transformer requires careful arithmetic over positions, since the relative position between $x_k$ and $y_k$ depends on the input length $n$, which may vary across input strings.
Without handling this dependence properly, the model may learn to copy only for some input lengths but not others.

The context-based algorithm is more flexible and can be implemented with a few attention layers, but it is NOT a correct algorithm in general: if there are many repeated substrings in the input, the context-based algorithm may not be able to locate the correct input position.

It is clear from the imperfect copying accuracy that LLMs do not exactly implement either of the two algorithms above, but what do they do instead?
We conjecture that LLMs instead implement a mixture of the two algorithms.
In other words, LLMs exhibit an inductive bias towards a mixture of position-based and context-based copying algorithms.
This conjecture is supported by the two failure modes: copying accuracy decreases as either the input length or the degree of repetition increases.
\vspace{-0.1in}
\subsection{Diagnosis via Expressivity Theory and Synthetic Experiments}\label{sec:LLM-synthetic-experiment}

In this subsection, we aim to understand the failure of LLMs on copy via a joint analysis of expressivity theory and synthetic experiments, as a support of our above conjecture. To analyze the copy task more precisely, we fix a template parameterized by a
positive integer pair $(a,b)$. Under this template, a full copy example is the
sequence
\[
(\texttt{A}_1,\cdots,\texttt{A}_a,s,\texttt{B}_1,\cdots,\texttt{B}_b,\; s),
\]
where $s \in \{0, 1\}^+ = \bigcup_{k\ge 1} \{0, 1\}^k$ is a binary string to be copied, and LLMs are asked to
autoregressively reproduce the second occurrence of $s$. The tokens
$\texttt{A}_1,\cdots,\texttt{A}_a$ form a fixed prefix and $\texttt{B}_1,\cdots,\texttt{B}_b$ a fixed delimiter (with
$\texttt{B}_1=\texttt{<\textbackslash n>}$) that marks the end of the source string and the
start of the copy. All of $\texttt{A}_1,\cdots,\texttt{A}_a,\texttt{B}_1,\cdots,\texttt{B}_b$ are distinct tokens
outside $\{0,1\}$, so that they can serve as positional anchors without being
confused with the binary content. In our synthetic experiments
(\Cref{sec:synthetic-experiments}), we take $a=2$ and $b=3$.



\subsubsection{Evidence from Expressivity Theory}
First of all, we ask: can Transformers express the position-based copying algorithm?
The answer is mixed.
Given the apparent simplicity of the copy task, one might expect a single Transformer layer to be expressive enough to solve it. This expectation is natural: single attention layers are closely related to associative recall~\citep{ramsauer2021hopfield,bricken2021attention,smart2025context}, and positional encodings can in principle support retrieval tokens at a fixed relative offset~\citep{vaswani2017attention,weiss2021thinking}. Perhaps surprisingly, this is not the case for binary copy. We show that even for input strings with length at most $5$, 
one-layer Transformers cannot exactly perform copying.
\begin{theorem}
\label{thm:1Lnegative}
    No one-layer Transformer with shift-invariant positional encoding can perform the binary copy task for all input strings with length $n \le 5$.
\end{theorem}
Note that shift-invariant positional encodings cover a wide range of positional encodings, including RoPE, NoPE, ALiBi, etc.
This follows from a simple proof by contradiction.
Intuitively, one-layer transformer fails to copy since copy requires retrieval at a length-dependent offset rather than a fixed offset. 
Thus, the failure of one-layer models reflects a lack of capability for processing length-dependent tasks.
The formal statement of this theorem and its proof are in \Aref{sec:1Lnegative}.

Beyond one layer, we show that two-layer Transformers with RoPE can express binary copy.
\begin{theorem}\label{thm:2L-RoPE-represent}
Let $a,b\geq 2$. There exists a constant embedding dimension $d$ such that for any norm constraint $\rho > 10^4$, there exist RoPE frequencies $\vbeta$ and a parameter vector $\vtheta$ with $\normtwo{\vtheta} \le \rho$ such that the two-layer Transformer with embedding dimension $d$, RoPE frequencies $\vbeta$, parameters $\vtheta$ can perfectly perform copying for all lengths $1 \le L \le O(\rho^2/\log\rho)$.
\end{theorem}
The key intuition here is that with multiple layers, the model can construct intermediate positional features and use them to align each output position with its corresponding input position.
The detailed setup and proof of this theorem is in \Aref{subsec:ropepositive}. Notice that in our synthetic experiment for binary copy in \Cref{sec:synthetic-experiments}, we set $a=2,b=3$, which is within the scope of our theoretical arguments above.

Putting these two theorems together, we can conclude that: (1)
the inductive bias induced by standard positional encodings, including RoPE, 
are not strong enough to support perfect copying for one-layer Transformers;
(2) Transformers with multiple layers can express the position-based copying algorithm, suggesting that training, rather than expressivity alone, biases Transformers away from position-based algorithms.





\subsubsection{Synthetic Experiments} \label{sec:synthetic-experiments}

\paragraph{Setup.} We further analyze what Transformers learn through synthetic experiments by training on the binary copy task. Here we consider two settings: Transformers with $1$ layer and $2$ attention heads, and Transformers with $12$ layers and $12$ attention heads.
For each setting, we consider $4$ position encodings: RoPE, ALiBi~\citep{presstrain}, NoPE, and RNoPE~\citep{yang2025rope}.
The template of the data are for $a=2,b=3$ as defined at the beginning of \cref{sec:LLM-synthetic-experiment}.

\paragraph{Uniform Data.}
We first train these models on uniform data with lengths from $1$ to $100$. 
While they perform perfectly on uniform data, they perform poorly on other data distributions, such as Imbalanced (see \cref{fig:rope-2distribution}).
This suggests that the models may have learned a context-based algorithm.

\paragraph{Imbalanced Data.}
Instead, when we train these models on imbalanced data for sufficient iterations, they can perform almost perfectly on imbalanced data as well as the other data distributions we construct.
However, all the position encodings mentioned above fail to generalize to longer inputs.
See \Aref{app:synthetic-details} for more details of experimental setups.

\begin{figure}[tb]
    \centering
    \includegraphics[width=\linewidth]{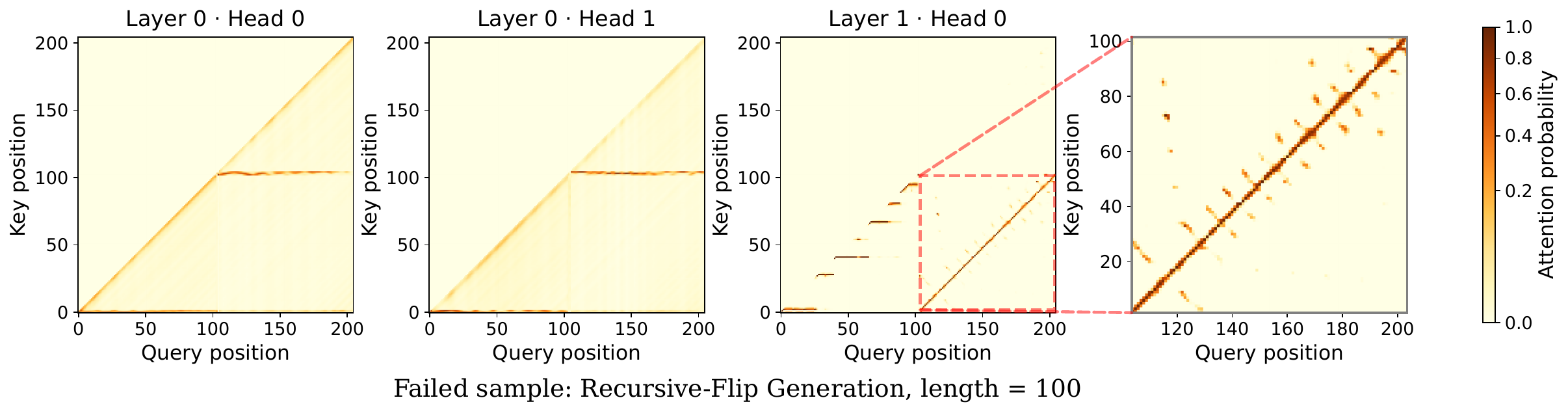}
    \caption{
    Attention map of a RoPE model on a failed sample with sequence length within training length.
    In Layer $0$, attention heads exhibit both \textbf{horizontal patterns} (attending to special tokens before and after the input string) and \textbf{diagonal patterns} (semantic/induction-style behavior).
    In Layer $1$, RoPE would attend to wrong tokens near themselves periodically.
    }
    \label{fig:rope-attn-copy}
\end{figure}

\paragraph{Attention Patterns.} 
We further show that these models use a mixture of position-based and context-based algorithms when copying.
To see this, we analyze the attention patterns of a trained RoPE model on binary copy under the imbalanced distribution. We consider a small model with $2$ heads in the first layer and $1$ head in the second layer, trained on sequences with lengths from $1$ to $100$. \Cref{fig:rope-attn-copy} shows the attention maps of  all layers and heads on representative copy examples.


We first inspect the last layer, since exact copying requires each output position to align with its corresponding input position. However, even for a failed sample with length $100$, which is within the training range, the learned attention is not a clean position-based alignment. As shown in the zoomed attention map, the final-layer head attends not only to the correct diagonal positions, but also periodically assigns large attention scores to nearby wrong tokens. This suggests that the final copy decision is affected by ambiguous local patterns.



In Layer $0$, we observe two types of attention patterns. 
First, some attention scores form \textbf{horizontal stripes} on the special tokens around the boundary between the input and output strings, namely the delimiters $\texttt{<\textbackslash n>}$, $\texttt{<I:>}$, and $\texttt{<O:>}$. A position-based algorithm requires each token to locate itself relative to fixed anchors, and recover a position-related signal from this relative information. Since these delimiters sit at fixed locations in the prompt, attending to them is one natural way to obtain such a signal. The horizontal stripes are consistent with this mechanism: they suggest that the model uses the anchor tokens to recover a position-related signal, which is the behavior we would expect from a position-based algorithm.
Second, Layer $0$ also contains \textbf{diagonal patterns},
which attend to tokens at a fixed relative offset from the query. This is a typical signature of an induction head~\citep{olsson2022context,singhneeds,jelassi2024repeat}: the model locates a previous occurrence of the current local context, such as $y_{k-3}y_{k-2}y_{k-1}$, and copies whatever token follows it. Such a mechanism is context-based, and therefore becomes ambiguous precisely when the input contains repeated substrings, since the same local context may then appear at several positions.



Taken together, the above analysis confirms the conjecture of Section~\ref{sec:LLM-faliure-why-rope} that
the trained RoPE model does not implement a purely position-based copy rule.
Instead, it relies on a mixture of a position-based algorithm, which anchors on
the fixed delimiters $\texttt{<\textbackslash n>}$, $\texttt{<I:>}$, and
$\texttt{<O:>}$ (the horizontal stripes), and a context-based algorithm, which
performs local matching (the diagonal patterns). The first component is tied to
the seen prompt lengths and does not extend reliably to longer inputs, while the
second is confused by repeated local structure. This mixture therefore explains
both failure modes observed in \Cref{fig:api-acc} that copying accuracy degrades
as the input grows longer, and as the degree of repetition increases. We further show that for a failure case where the test length is beyond the training length, the above attention patterns are similarly observed. (\Aref{app:discussion-attn-pattern})

%% file: theoretical_analysis.tex
\section{2D Positional Information Encoding with 2D-RoPE}
\label{sec:main}





Given that current LLMs exhibit inductive bias towards a mixture of position-based and context-based algorithms when copying and can fail, a natural question is whether we can design a positional encoding that explicitly favors the position-based algorithm.
In this section, we propose a novel architecture based on 2D-RoPE~(Section~\ref{sec:2drope}). Then, we describe an adaptive variant of 2D-RoPE~(Section~\ref{sec:auto-2d-rope}). Then, we show the superior performance of 2D-RoPE in copying~(Section~\ref{sec:empirical-results-2drope}). Finally, we present a theoretical analysis of 2D-RoPE~(Section~\ref{sec:theoretical-analysis-2drope}) in expressivity.

\subsection{2D-RoPE Architecture}
\label{sec:2drope}


\paragraph{2D Position ID Generation.}
Recall the diagnosis in ~\Cref{sec:LLM-faliure-why-rope} that to copy perfectly, a RoPE-based model must
retrieve the source token $x_k$ from an offset that grows with the input length
$n$, while standard relative encodings such as RoPE favor fixed-offset retrieval
and thus do not provide a strong enough inductive bias to learn this
length-dependent retrieval reliably. The idea behind 2D-RoPE is to remove this
length dependence by assigning each token a two-dimensional position ID instead of
a single one. As we sketch below and make precise in the rest of this section, the
right choice of coordinates turns the length-dependent retrieval into a
fixed-offset one. More specifically, 2D-RoPE supports retrieval from the same column in the previous row, which is
far easier for a single attention layer to express and to learn. The following paragraphs introduce the construction in detail.


Specifically, the key idea of 2D-RoPE is to assign each token a pair of position IDs instead of a single position ID. We use the line break token $\symn$ as a row separator. Let the 2D position of token $i$ be $(r_i,c_i)$, where $r_i$ is the row index and $c_i$ is the column index within that row. We set the first token position to $(r_1,c_1)=(0,0)$. For each subsequent token $i>1$,
\[
(r_i,c_i)
=
\begin{cases}
(r_{i-1},\, c_{i-1}+1), & \text{if token } i-1 \text{ is not a line break token},\\[3pt]
(r_{i-1}+1,\, 0), & \text{if token } i-1 \text{ is a line break token}.
\end{cases}
\]
Thus, tokens on the same line share the same row index, and their column index records the offset within the line. In a copy task, this representation makes the source token $x_k$ and the target token $y_k$ share the same column index, so copying can be reduced to retrieval from the same column.
This assumes that the input and output are separated by a line break. 
For the case where the separator is not a line break, we propose Auto-2D-RoPE to automatically identify the separator (see~\Aref{app:adaprope}).

\paragraph{From 1D RoPE to 2D-RoPE.}
Standard RoPE~\citep{su2024roformer} applies a positional rotation to the query and key vectors according to the 1D token position. In 2D-RoPE, we apply the same idea separately to the row and column coordinates. Concretely, we split the head dimension into two equal-sized parts: one encodes relative row positions, and the other encodes relative column positions. This gives the model direct access to both row differences and column differences. For copy, the column coordinate is especially useful, because source and target tokens that should be aligned have the same column ID, although their 1D distance depends on the input length.

We now give the formal definition. Let $d$ be the head dimension, divisible by $4$. For an input $\mX\in\R^{T\times d}$, a 2D-RoPE attention layer is defined as
\[
\Attn(\mX)
=
\cS\left(
\MASK\left(
R(\mX\mQ)\bigl(R(\mX\mK)\bigr)^\top
\right)
\right)\mX\mV,
\]
where $\cS$ is the row-wise softmax, $\MASK$ is the causal attention mask, and $\mQ,\mK,\mV$ are trainable projection matrices.

It remains to define the rotation function $R$. For any matrix $\mZ\in\R^{T\times d}$, the $i$-th row $\mZ_i$ is mapped to $\mZ_i \mR_{r_i,c_i}^\top,$
where 
$$\mR_{r_i,c_i} = \diag\bigl(
\mR_1(r_i),\mR_2(r_i),\ldots,\mR_{d/4}(r_i),
\mR_1(c_i),\mR_2(c_i),\ldots,\mR_{d/4}(c_i)
\bigr).$$ 

For the $j$-th block, the rotation is
\[
R_j(t) = \begin{pmatrix} \cos(t\beta_j) & -\sin(t\beta_j) \\
\sin(t\beta_j) & \cos(t\beta_j) \end{pmatrix},
\]
which is the standard RoPE rotation matrix with frequency parameter $\beta_j$.
In our synthetic experiments in \Cref{sec:LLM-synthetic-experiment}, we use $\beta_j=\theta^{-4j/d}$ where $\theta = 100$. In our LLM experiments in \Cref{sec:llm-experiments}, we set $\theta$ to $1000$. For the original RoPE architecture, see \Aref{sec:appendix_prel}.
We also discuss differences from prior vision 2D-RoPE variants in \Aref{app:2drop-imp}.

\subsection{How to Establish 2D Position Encodings in Text} 

Previous vision models have explored applying RoPE to pixels or image patches
using their height and width coordinates, so that attention can encode relative
spatial displacement on the image grid~\citep{jeevan2022resource, heo2024rotary,
chu2024visionllama}. In language modeling, however, the input is typically treated
as a 1D token sequence, and there is no fixed 2D grid analogous to image patches.
The 2D structure must therefore be inferred from the text itself rather than read
off from a layout.

Fortunately, many text formats already carry such a structure implicitly. Lists use
line breaks to separate items, tables use rows and columns to organize aligned
entries, and code uses newlines and indentation to express block structure. These
line breaks are thus not merely formatting tokens, but often signal how different
parts of the text are organized. Motivated by this observation, we use the
line-break token to induce a 2D structure in text: each token receives a row ID,
determined by the number of preceding line breaks, and a column ID, determined by
its offset within the current line. The 2D coordinates are therefore not
externally given as in images, but constructed from the textual structure itself.



\subsection{From Line Break to Adaptive 2D Position: Auto-2D-RoPE}
\label{sec:auto-2d-rope}
However, the above construction also reveals a limitation of 2D-RoPE. When the input contains few or no line breaks, the row coordinate changes rarely, and 2D-RoPE largely reduces to a 1D positional encoding along the column direction. This motivates a more flexible variant that automatically infers 2D positional structure from the input sequence, rather than relying on a fixed separator token. We introduce Auto-2D-RoPE for this purpose. 

At each layer, Auto-2D-RoPE uses the hidden representation of each token to predict two update coefficients through a shared linear projection. These coefficients define a learnable affine update of the token's 2D coordinate, and the final 2D position IDs are obtained by composing these updates along the sequence. In this way, the model can learn when to behave like moving to the next column and when to behave like resetting the column and moving to a new row. The learned coordinates are then used in the same 2D-RoPE rotation as before. Empirical results in \Cref{sec:synthetic-experiments} and \Cref{sec:llm-ablation} show that Auto-2D-RoPE outperforms standard PEs in binary copy and even remains effective when we replace the line-break token $\symn$ with another separator $*$, where rule-based 2D-RoPE can no longer identify row boundaries and largely loses its length-generalization advantage. We provide the full architectural details in \Aref{app:adaprope}.

\subsection{Empirical Results of 2D-RoPE}
\label{sec:empirical-results-2drope}


\paragraph{Results of Binary Copy Test.} In \Cref{fig:api-acc}, the 2D-RoPE model maintains near-perfect accuracy on Recursive-Flip strings, suggesting that explicitly aligning source and target positions provides a more reliable inductive bias for copying than relying on local-context matching alone. 
The model in the figure has $1.4$B parameters and is finetuned on the imbalanced binary copy task.

\paragraph{Results of Python List Conversion Test.} In the right subfigure in \Cref{fig:api-acc}, our fine-tuned 2D-RoPE model in \Cref{sec:llm-experiments} is trained only on lists with $50$ to $150$ data points and periods between $2$ and $8$. Although the training data are easy, the 2D-RoPE model maintains substantially higher accuracy than all tested LLMs. 
This result is consistent with the binary copy experiments: when the input contains repeated local patterns, as in periodic or oscillatory sequences, exact copying is difficult for standard position encoding while 2D-RoPE provides an explicit positional alignment between the source data points and their copied output positions.

\paragraph{Results of Synthetic Experiments on Binary Copy.}
\begin{figure}[tb]
    \centering
    \begin{subfigure}[t]{0.95\linewidth}
        \centering
        \includegraphics[width=\linewidth]{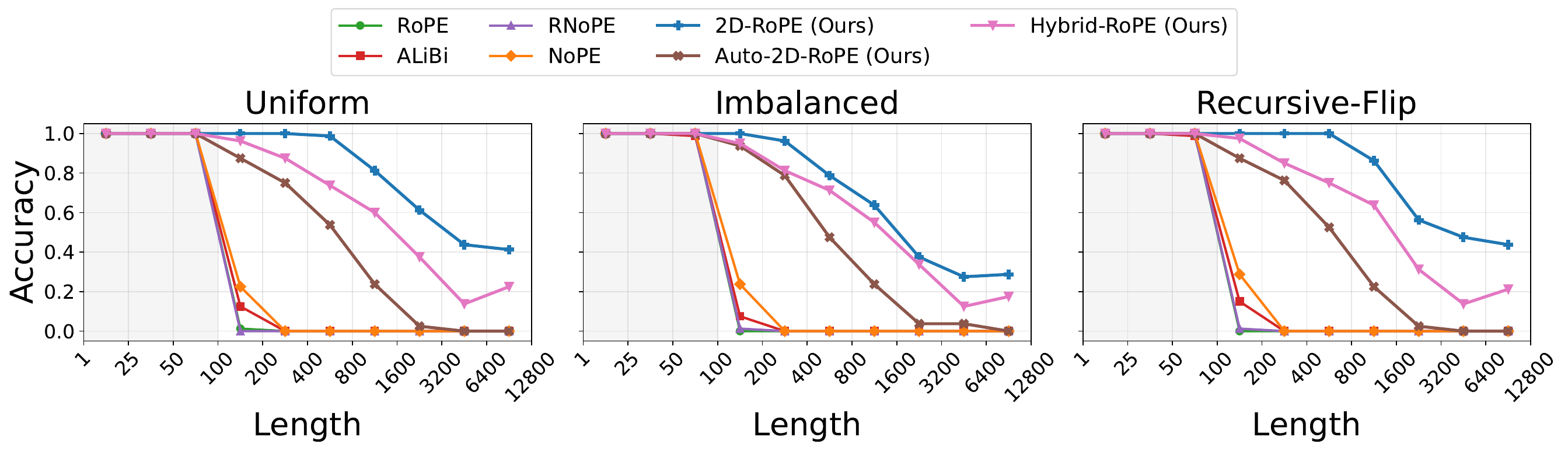}
        \caption{Length generalization results for different positional encodings with 12 layers and 12 heads.}
        \label{fig:scratch-12}
    \end{subfigure}
    
    \vspace{0.5em}
    
    \begin{subfigure}[t]{0.95\linewidth}
        \centering
        \includegraphics[width=\linewidth]{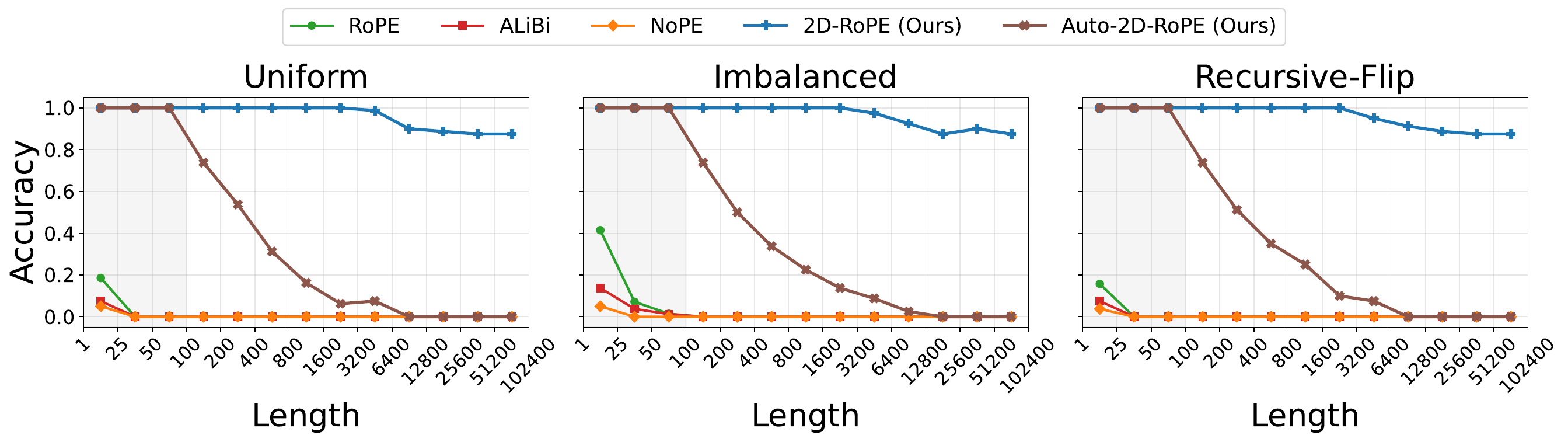}
        \caption{Length generalization results for different positional encodings with 1 layer and 2 heads.}
        \label{fig:copy-LG-from-scratch}
    \end{subfigure}
    
    \caption{\textbf{2D-RoPE models exhibit stronger length generalization for synthetic experiments on
binary copy.} Models are trained only on lengths $1$--$100$ and tested on much
longer inputs. The $x$-axis is the sequence length (log scale) and the $y$-axis is
the copying accuracy; the three columns correspond to the Uniform, Imbalanced, and
Recursive-Flip distributions, and each curve is a different positional encoding.
\textbf{(a)} $12$ layers and $12$ heads; \textbf{(b)} $1$ layer and $2$ heads.}
    \label{fig:length-generalization-combined-full}
\end{figure}

\Cref{fig:length-generalization-combined-full} shows that 2D positional structure provides a strong inductive bias for binary copy.
2D-RoPE-based methods exhibit much stronger length generalization than standard positional encoding.
In the $12$-layer setting, positional encodings with 2D structure substantially improve over standard positional encodings.
The advantage is even clearer in the $1$-layer setting as the 2D-RoPE model maintains perfect accuracy up to more than $1000\times$ its training context length.



\input{2drope_theory_new}

%% file: 2drope_theory_new.tex
\subsection{Theoretical Analysis of 2D-RoPE}
\label{sec:theoretical-analysis-2drope}

The experiments above suggest that 2D-RoPE provides a much stronger inductive bias for copying than standard 1D positional encodings: the source token to be copied and the corresponding target position have fixed relative positions, even though their 1D distance depends on the input length. This complements the diagnosis in \Cref{sec:LLM-faliure-why-rope}. 
In this subsection, we show that 2D-RoPE also has positive theoretical guarantees. We prove that a one-layer 2D-RoPE Transformer can explicitly represent the position-based copy rule with bounded parameters. 
Second, under the $\ell_{\infty}$ regularization-path analysis, the global minima of the one-layer 2D-RoPE in training loss provably learn a solution that generalizes far beyond the training length.

\subsubsection{One-Layer 2D-RoPE Can Represent Binary Copy}

We study the expressive power of one-layer 2D-RoPE. In contrast to the impossibility result for one-layer shift-invariant positional encodings, 2D-RoPE makes the correct source position geometrically simple: when predicting the $i$-th copied token, the query position in the output row and the source token $s_i$ in the input row have the same column coordinate. Thus, a single attention layer does not need to infer the input length or compute a length-dependent 1D offset. It only needs to attend to the token in the previous row with the same column.

The following theorem formalizes this intuition. The parameter norm controls the range of lengths for which the random 2D-RoPE frequencies separate the correct same-column position from all wrong positions. The length bound is polynomial in the norm, and the exponent grows with the RoPE dimension. We still consider format $(\texttt{A}_1,\cdots,\texttt{A}_a,s,\texttt{B}_1,\cdots,\texttt{B}_b)$ with $a,b\geq 1$ as in \Cref{sec:LLM-synthetic-experiment}.

\begin{theorem}
\label{thm:2dropepositive}
Assume that the 2D-RoPE dimension $d\geq 100+4a+4b$ is divisible by $4$. For any $\rho>d^2$, let the 2D-RoPE frequencies $\beta_i\overset{i.i.d.}{\sim}\Unif[0,2\pi]$ for all $i\in[d/2]$. Then, with probability at least $1-\frac{d}{\rho}$ over the randomness of $\vbeta$, there exists a one-layer Transformer with 2D-RoPE whose $\ell_2$ parameter norm $\|\vtheta\|_{2}\leq \rho$ such that correctly copies every binary string $s$ with length $|s|\leq \rho^{d/12}$.
\end{theorem}

The construction behind \Cref{thm:2dropepositive} is purely position-based and
comes from one simple geometric fact: under the 2D coordinates, the source token
$s_i$ to be copied always sits in the previous row at a fixed relative column, no
matter how long the string is. To leverage this offset, we make the query and key
vectors the same for every token, so that each attention score depends only on
the relative 2D coordinate of two positions, not on the tokens themselves.
Attention is thus based on position alone, with no context matching. We then
rotate the query to cancel the offset, so that the correct source has zero
relative rotation, which gives it the highest attention score. Every wrong
position differs from it in either column or row, and so has a nonzero relative
rotation. By sampling the RoPE frequencies at random, we ensure that with high
probability, none of these rotations up to the target length happens to sum back
to the maximum. This leaves a score gap that puts more than half of the attention
mass on $s_i$. Finally, only $0$ and $1$ have nonzero value vectors, and they
point in opposite directions, while all other tokens have zero value. So once
most of the attention is on $s_i$, the output logits predict $s_i$ correctly.

The detailed theoretical setup of this subsection is in \Aref{subsec:setup}. A more detailed version of \Cref{thm:2dropepositive} and its proof are in \Aref{subsec:2dropeexpress}. 

\subsubsection{One-Layer 2D-RoPE Provably Learns Binary Copy with Length Generalization}\label{sec:margin-main}
Beyond the expressivity results, we theoretically show that one-layer 2D-RoPE can learn binary copy with length generalization. For technical convenience, we consider the setting in which $b=a+2$ and $\texttt{B}_1=\symn$ (remind that the input is $(\texttt{A}_1,\cdots,\texttt{A}_a,s,\texttt{B}_1,\cdots,\texttt{B}_b)$ and Transformer is asked to output string $s$). Under this assumption, when the model generates a token, the corresponding target token has the same column index as the current position, which simplifies the theoretical analysis.

We study a global minimizer of the population loss for training length $L$ under an $\ell_\infty$-norm constraint:
\begin{align*}
    \vtheta_{M,L}^*
    \in
    \arg \min_{\vtheta:\|\vtheta\|_{\infty}\leq M}
    \cL_L(\vtheta),
\end{align*}
where $\vtheta$ denotes the learnable parameter of 2D-RoPE, as expanded and explained in \Aref{subsec:setup}.

Previous work has analyzed how solutions evolve as the norm constraint $M$ increases in the limit $M\to\infty$ under various settings. This is a perspective often referred to as the regularization path, including homogeneous neural networks~\citep{wei2019regularization} and simplified attention models~\citep{ataee2023max}. In contrast, our analysis establishes that, for every sufficiently large but finite $M$, a global minimizer of the 2D-RoPE population loss exhibits strong length generalization. 

Previous analyses of regularization paths have primarily focused on $\ell_2$-norm regularization~\citep{wei2019regularization,ataee2023max}. We instead consider an $\ell_\infty$ constraint, motivated by recent results on the implicit bias of Adam and AdamW~\citep{xie2024implicit,fan2025implicit,zhang2024implicit}, as all experiments in this paper use AdamW. In particular, \citet{xie2024implicit} showed that AdamW can converge to KKT points of an $\ell_\infty$-constrained optimization problem, while \citet{zhang2024implicit,fan2025implicit} showed that, for linear models on separable data, Adam converges to max-margin solutions with respect to the $\ell_\infty$ norm. Similar $\ell_\infty$ regularization paths have also appeared in theoretical analyses of grokking and algorithm learning~\citep{mohamadi2024you}. Thus, the $\ell_\infty$ constraint is not only technically convenient but also aligned with the implicit bias of the optimizer used in our experiments.

We now state the main theorem of this subsection, which guarantees length generalization for $\vtheta_{M,L}^*$ whenever $M$ is sufficiently large.

\begin{theorem}
\label{thm:margin}
Assume that the 2D-RoPE frequency $\beta_i\overset{i.i.d.}{\sim}\Unif[0,2\pi]$ for $\forall i \in[d/2]$.
Under \Cref{ass:prob,ass:Lbound,ass:zerovalue}, with probability at least $1-\frac{3}{d}$ over the randomness of $\beta$, there exists $M_0>0$ such that for all $M>M_0$, the one-layer 2D-RoPE Transformer with parameter $\vtheta_{L,M}^*$ can correctly copy all binary strings with length at most $L^{\sqrt{d}/2}$.
\end{theorem}

The assumptions and proof are in \Aref{app:global-min-analysis}. The theorem shows that, under a sufficiently large $\ell_\infty$ norm constraint, the resulting global minimizers do not merely fit strings within the training-length range. Instead, they learn a copying algorithm that correctly generalizes to binary strings far longer than those seen during training.

\paragraph{Proof Sketch.}
The main idea is to introduce the function
\[
f_L(\mQ,\mK,\mV)
:=
\min_{s,i,|s|\leq L}
\left\{
g_{s,i}^+(\mQ,\mK,\mV)
-
g_{s,i}^-(\mQ,\mK,\mV)
\right\},
\]
which acts as a surrogate for studying the margin between the attention logits assigned to correct and incorrect tokens. We first show that a small value of $f_L(\mQ,\mK,\mV)$ leads to a lower bound on the population loss. We then construct a feasible parameter configuration with a large value of $f_L(\mQ,\mK,\mV)$ and a small population loss. Comparing these two bounds shows that every constrained global minimizer must have a sufficiently large value of $f_L(\mQ,\mK,\mV)$ when $M$ is sufficiently large. Finally, concentration properties of the random 2D-RoPE frequencies allow us to convert this lower bound on $f_L(\mQ,\mK,\mV)$ into correct copying at lengths far beyond the training range; see detailed intuitions and proof in \Aref{app:global-min-analysis}.

%% file: experiments.tex
\section{Experiments on LLM Pretraining}\label{sec:llm-experiments}

\begin{table}[t]
    \footnotesize
    \centering
    \caption{Large-scale pretraining results of 2D-RoPE, H-RoPE, and RoPE. Each model is pretrained on DCLM, and then finetuned on binary copy with imbalanced distribution.}
    \setlength{\tabcolsep}{4pt}
    \begin{tabular}{lcc|c|cccc|cccc}
        \toprule
        & & PT & Common & \multicolumn{4}{c|}{Copy (Imbalanced)} & \multicolumn{4}{c}{Copy (Recursive-Flip)} \\
        Model & Param & Data & sense & 1K & 2K & 4K & 8K & 1K & 2K & 4K & 8K \\
        \midrule
        RoPE & 350M & 7B 
            & \textbf{41.4} 
            & 97.2 & 14.8 & 0.0 & 0.0
            & 14.4 & 7.8 & 0.0 & 0.0 \\
        H-RoPE (ours) & 350M & 7B
            & 41.2 
            & \textbf{99.6} & 57.8 & 0.0 & 0.0 
            & \textbf{96.9} & \textbf{92.2} & 35.1 & 0.0 \\
        2D-RoPE (ours) & 350M & 7B
            & 40.4 
            & 97.4 & \textbf{69.2} & \textbf{33.4} & \textbf{2.4} 
            & \textbf{96.9} & 89.3 & \textbf{56.4} & \textbf{36.4} \\
        \midrule
        RoPE & 730M & 15B
            & 44.0 
            & 97.0 & 11.0 & 0.0 & 0.0
            & 15.5 & 7.8 & 1.1 & 0.0 \\
        H-RoPE (ours) & 730M & 15B
            & \textbf{45.5} 
            & 95.4 & 11.4 & 0.0 & 0.0 
            & 14.4 & 5.8 & 0.0 & 0.0 \\
        2D-RoPE (ours) & 730M & 15B
            & 45.2 
            & \textbf{99.4} & \textbf{74.6} & \textbf{12.8} & 0.0 
            & \textbf{92.8} & \textbf{76.7} & \textbf{45.7} & \textbf{10.9} \\
        \midrule
        RoPE & 1.4B & 28B 
            & 45.9 
            & 99.6 & 26.8 & 0.0 & 0.0
            & 21.6 & 8.7 & 0.0 & 0.0 \\
        H-RoPE (ours) & 1.4B & 28B 
            & \textbf{47.2} 
            & \textbf{100.0} & 73.6 & 3.6 & 0.0 
            & \textbf{100.0} & \textbf{98.1} & 55.3 & 5.5 \\
        2D-RoPE (ours) & 1.4B & 28B
            & 46.8 
            & \textbf{100.0} & \textbf{100.0} & \textbf{92.0} & \textbf{61.8} 
            & \textbf{100.0} & \textbf{98.1} & \textbf{92.6} & \textbf{87.3} \\
        \midrule
        RoPE & 730M & 100B 
            & 46.0 
            & 99.6 & 29.2 & 0.0 & 0.0
            & 24.7 & 8.7 & 1.1 & 1.8 \\
        H-RoPE (ours) & 730M  & 100B 
            & \textbf{46.2} 
            & 99.2 & 52.1 & 0.0 & 0.0
            & \textbf{100.0} & 98.1 & 28.7 & 1.8\\
        2D-RoPE (ours) & 730M & 100B  
            & \textbf{46.2} & \textbf{100.0} & \textbf{96.3} & \textbf{72.8} & \textbf{15.8}
            & \textbf{100.0} & \textbf{100.0} & \textbf{100.0} & \textbf{100.0}\\
        \bottomrule
    \end{tabular}
    \label{tab:llm-results}
\end{table}

To validate the effectiveness of 2D-RoPE in empirical language modeling, we first pretrain Transformer language models from scratch and compare 2D-RoPE against RoPE. Then, we finetune our different models on the synthetic binary copy dataset and the Python List Conversion datasets. 

\paragraph{Model Configuration.}

We experiment with the Qwen3~\citep{yang2025qwen3} architecture, which is one of the most widely used open-source LLMs. To ensure fair comparison, we keep everything unchanged and only replace RoPE with 2D-RoPE. We experiment with different models sizes (350M, 730M, and 1.4B parameters). See \Aref{app:llm-model-config-details} for more details.


\paragraph{Pretraining Configuration.}

The models are pretrained with the DCLM corpus~\citep{dclm} with 2K context length using the typical optimizer, LR scheduler, and hyperparameters (see \Aref{app:llm-training-config-details} for more details). We use a data-to-parameter ratio of 20~(i.e., Chinchilla law~\citep{chinchilla}). Furthermore, we experiment with an overtrained setting that is common in state-of-the-art LLMs, where a 730M model is pretrained on 100B tokens.

\paragraph{Finetuning Configuration.}
Each model is finetuned on the Imbalanced binary copy task~(described in Section~\ref{sec:LLM-faliure-why-rope}), using 200K examples and 2K context length~(1K maximum input/output tokens).
For each model, we sweep learning rate~(LR) values of $\{3\times 10^{-5},5\times 10^{-5},1\times 10^{-4}\}$ and arrive at $5\times 10^{-5}$ because it has best overall performance.
See \Aref{app:lr-ablation} for more details.

\paragraph{Hybrid RoPE.} 
Prior works have achieved performance gains in various settings by combining RoPE with other position encoding methods~\citep{nape,swan-gpt,fox}. We also explore whether 2D-RoPE can be combined with RoPE to balance the advantages of both methods. To this end, we construct a Hybrid-RoPE~(H-RoPE) model between RoPE and 2D-RoPE, in which RoPE and 2D-RoPE are used alternately across Transformer blocks. 
\subsection{LLM Evaluation Details}\label{sec:LLM-ablation}

We evaluate the models on the following benchmarks. 
\paragraph{Common-Sense Reasoning (CSR).}
To evaluate the common-sense reasoning abilities of LLMs, we use the average accuracy in the following set of tasks: ARC-easy, ARC-challenge~\citep{arc}, HellaSwag~\citep{hellaswag}, WinoGrande~\citep{winogrande}, MMLU~\citep{mmlu}, LAMBADA~\citep{lambada}, PIQA~\citep{piqa}, and BoolQ~\citep{boolq}. We use LM-Evaluation-Harness~\citep{eval-harness} for evaluation.




\paragraph{Binary Copy.} 
We evaluate on the Imbalanced and Recursive-Flip binary copy tasks as described in Section~\ref{sec:copy-benchmarks}. The first task is an in-domain copy task~(the models are finetuned on it), and the second task is an out-of-domain task that tests the transferability of our method on other copying tasks. 



\subsection{LLM Results}


As shown in Table~\ref{tab:llm-results}, LLMs with 2D-RoPE have superior length generalization in the binary copy tasks, in both in-domain~(Imbalanced) and out-of-domain~(Recursive-Flip) settings. This result shows that 2D-RoPE can exhibit its advantage in copying in practical large-scale LLM training. This advantage holds for both standard chinchilla $\times20$ training and the overtraining setting ($730$M model with $100$B DCLM data). 
Finally, models with 2D-RoPE outperform RoPE models in CSR except for the smallest-scale model, implying that learning copying does not result in degraded general language modeling capabilities at scale. 

\subsection{Pretraining Loss}
\label{app:pretraining-loss}

\begin{figure}
    \centering
    \includegraphics[width=0.95\linewidth]{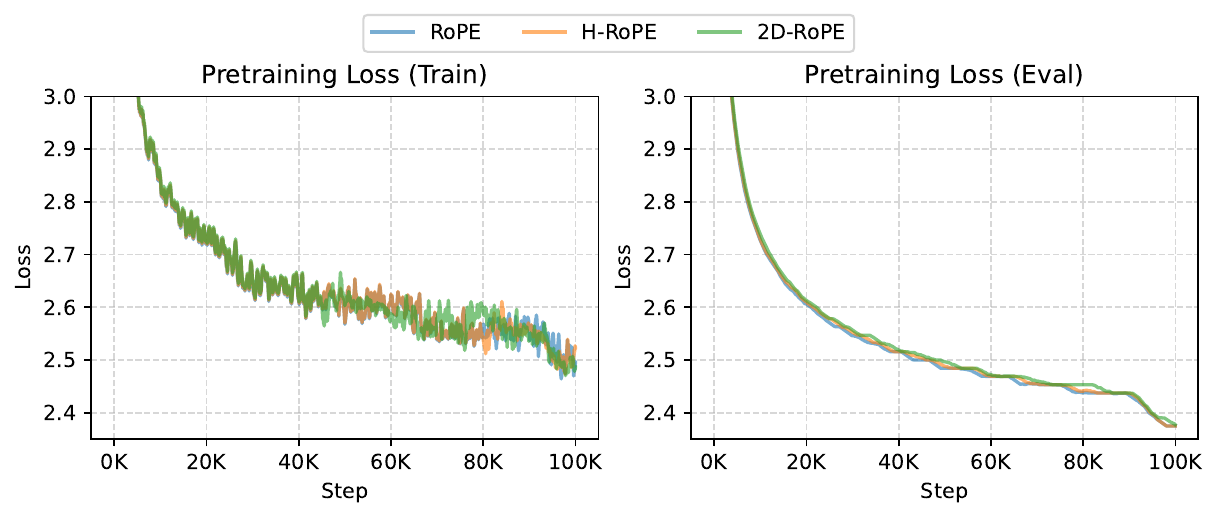}
    \caption{The training loss and evaluation loss during pretraining of the 2D-RoPE, H-RoPE, and RoPE models with 730M parameters (see Section~\ref{sec:llm-experiments}). These models were trained on 100B tokens from DCLM, and evaluated on FineWeb-Edu.}
    \label{fig:pretraining-loss}
\end{figure}

\begin{table}[tb]
    \centering
    \small
    \caption{The training and validation loss during pretraining of the three overtrained models in Section~\ref{sec:llm-experiments}. The training data is DCLM and the validation data is FineWeb-Edu.}
    \begin{tabular}{l|cc}
        \toprule
        Model & Training Loss & Validation Loss \\
        \midrule
        RoPE    & 2.5080 & \textbf{2.3808} \\
        H-RoPE  & 2.5000 & \textbf{2.3808} \\
        2D-RoPE & \textbf{2.4943} & 2.3876 \\ 
        \bottomrule
    \end{tabular}
    \label{tab:llm-pretraining-loss}
\end{table}

Here, we report the training and validation loss during pretraining (in Section~\ref{sec:llm-experiments}). We compute validation loss on FineWeb-Edu, which is another widely used pretraining corpus. The results are shown in Figure~\ref{fig:pretraining-loss} and Table~\ref{tab:llm-pretraining-loss}. One can see that the three models have roughly the same training and validation loss, which indicates that the performance of 2D-RoPE on next token prediction is comparable to RoPE.




\subsection{Additional Experiments}\label{sec:llm-ablation}
\begin{figure}[tb]
    \centering
        \begin{subfigure}[t]{0.45\linewidth}
        \centering
        \includegraphics[width=\linewidth]{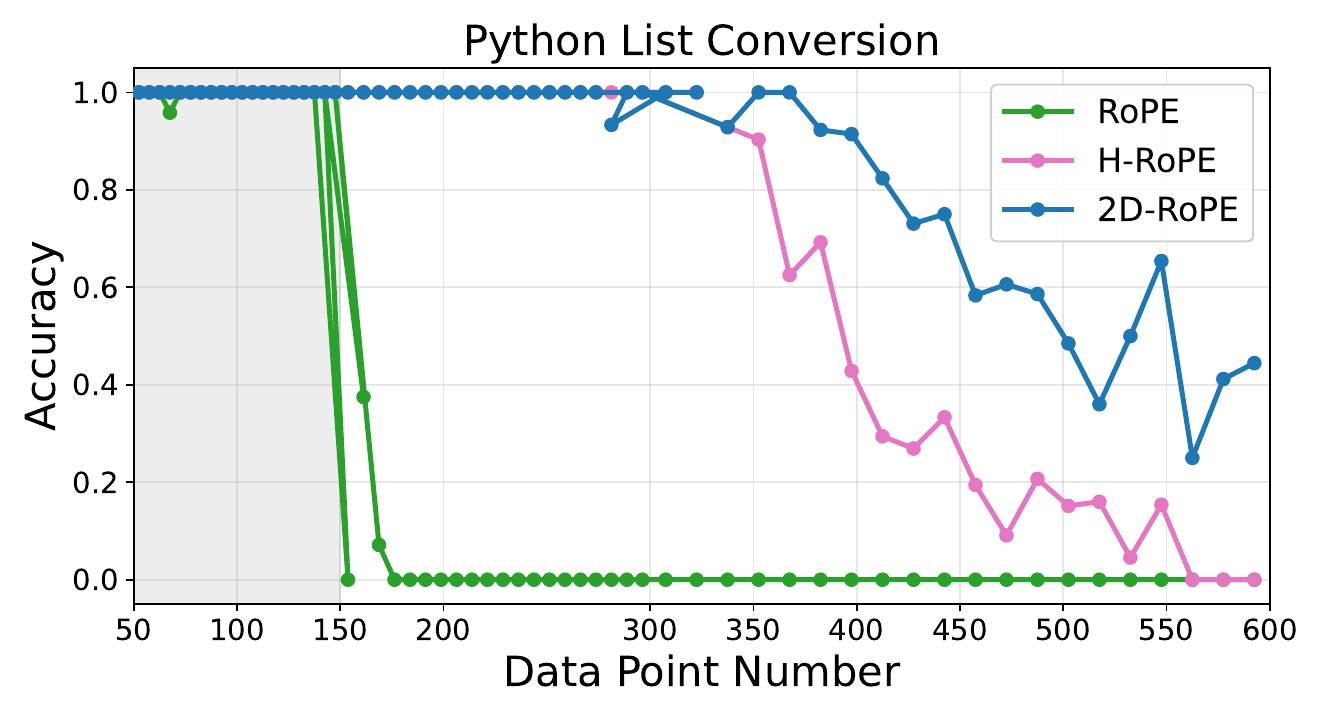}
        \label{fig:python-acc}
    \end{subfigure}
       \hfill
    \begin{subfigure}[t]{0.45\linewidth}
        \centering
        \includegraphics[width=\linewidth]{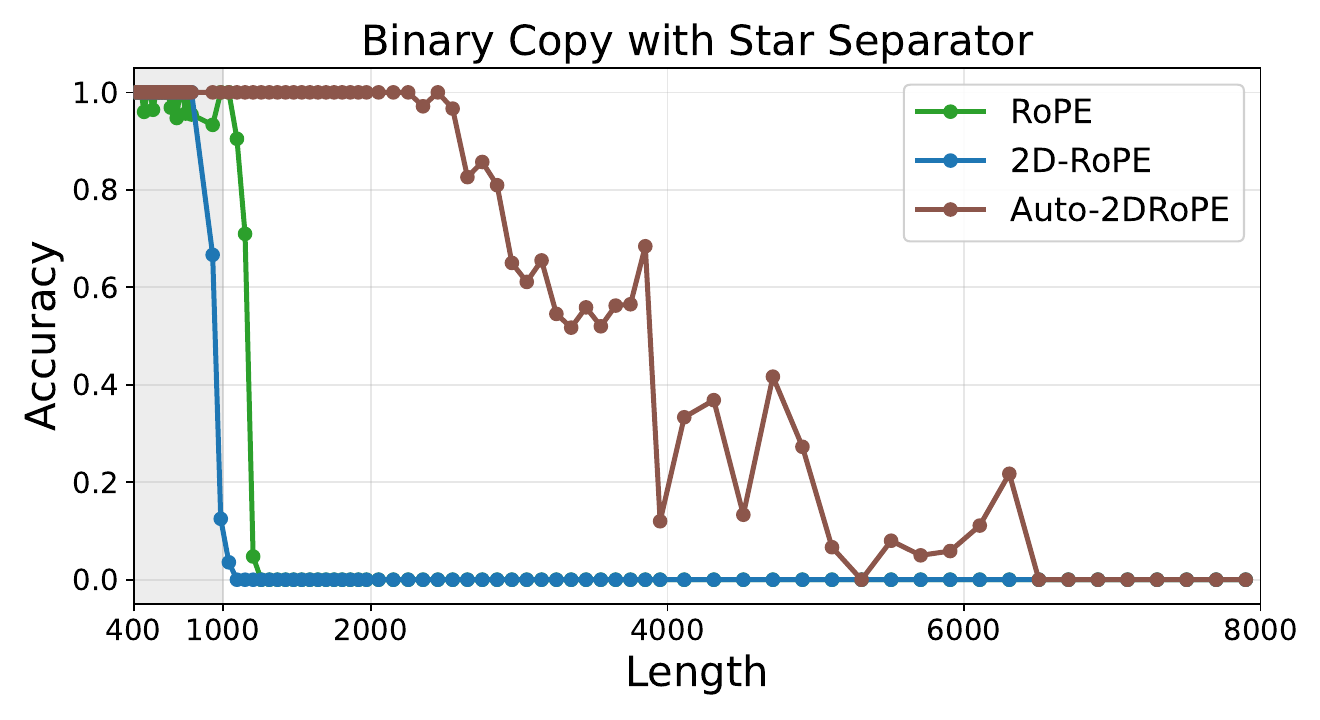}
        \label{fig:newline-acc}
    \end{subfigure}
    \caption{Fine-tuning results on copy-related tasks using $350$M models pretrained on $7$B tokens from DCLM. The gray region indicates the training context length range during finetuning.}
    \label{fig:copy-task-acc}
\end{figure}


\paragraph{Python List Conversion Task.} 
To demonstrate the practical importance of the ability to copy, we finetune the 350M model on the Python List Conversion task~(described in Section~\ref{sec:copy-benchmarks}).
As shown in the left subfigure in \Cref{fig:copy-task-acc}, after finetuning, 2D-RoPE achieves perfect generalization to sequences up to $3$ times longer than those seen during training. In comparison, H-RoPE also generalizes to about three times the training length, but underperforms 2D-RoPE. Meanwhile, RoPE shows limited length generalization. 


\paragraph{Auto 2D-RoPE: Beyond the Line-Break Separator.}
2D-RoPE uses line-break token $\symn$ to construct 2D positional coordinates, which means that it almost reduces to RoPE when the input text contains no, or very few, line breaks. To address this concern, we give an alternative called Auto 2D-RoPE as mentioned in \Cref{sec:auto-2d-rope}. Unlike 2D-RoPE, Auto 2D-RoPE does not rely on an explicitly specified line-break token to define rows and columns. Instead, it learns to construct 2D position IDs directly from the input sequence.

Concretely, Auto 2D-RoPE represents the 2D position ID of each token through the cumulative product of a learnable 2D transformation matrix along the sequence. Intuitively, this allows the model to automatically decide how positional information should evolve as it scans the input, rather than using a fixed rule such as ``increase the row index after every newline and reset the column index.'' Similar cumulative-product constructions have appeared in prior attention designs~\citep{tan2025scaling,yang2026path}, and the idea of making position IDs learnable has also been explored in recent works~\citep{li2025repo,lahoti2026mamba}. We provide the full architectural details of Auto 2D-RoPE in \Aref{app:adaprope}.

Empirically, Auto 2D-RoPE achieves strong length generalization on copy tasks where we replace the newline token $\symn$ with another character $*$. In this setting, the original 2D-RoPE can no longer identify row boundaries and effectively degenerates to standard RoPE, showing little or no length generalization. In contrast, Auto 2D-RoPE still learns a useful 2D positional structure and generalizes to longer sequences as shown in \Cref{fig:copy-task-acc}.

%% file: conclusion.tex
\section{Conclusion}
In this paper, we reveal a surprising fact that current LLMs still fail on copy. Motivated by this observation, we introduce 2D-RoPE to solve this problem. We further introduce Auto-2D-RoPE to reduce the reliance on explicit line breaks. Overall, our results show that the way positional information is represented can strongly affect even the most basic algorithmic behaviors of language models. We hope this work encourages further exploration of positional encodings that better capture the latent structure of text and support more reliable length generalization.

Several directions remain open for future study. (i) A thorough theoretical understanding of why multi-layer RoPE-based transformers fail to copy is lacking. (ii) How to generalize and scale up Auto-2D-RoPE to a more practical LLM training regime may be of significant value. (iii) Also, a further theoretical understanding of why 2D-RoPE can learn binary copy in the training process is underexplored.

\section{Acknowledgement}
We would like to thank Zixuan Wang, Shaowen Wang, and Tingqiang Xu for their kind feedback and insightful comments. We would also like to give special thanks to Hongxun Wu for initial contributions to this project.

%% file: appendix.tex
\input{appendix-api-details}

\input{appendix-synthetic-exp-detail}

\input{appendix-attn-pattern}

\input{appendix-implementation-2D-RoPE}

\input{appendix-details-llm-experiments}

\input{appendix-adap-2drope}

\input{appendix_prel}

\input{appendix-RoPE-theory}
\input{appendix-2dRoPE-theory}

\input{margin}

\input{appendix-prompt-templates}

%% file: appendix-api-details.tex
\section{Test Details of the LLMs in Binary Copy and Python List Conversion}\label{app:api-details}

We give the details of the closed-model Copy benchmarks reported in \Cref{fig:api-acc} and \Cref{fig:repeat-acc}. All prompt templates are provided in \Aref{appendix:prompts}. All closed-model API evaluations were run between June 11 and June 14, 2026. Unless stated otherwise, each sample was attempted at most three times in total, including the initial call. In the final evaluation logs used for the reported numbers, every requested sample received a valid API response within the allowed retry budget. No sample was discarded or replaced due to an API failure; hence all reported failures in the accuracy tables correspond to incorrect model outputs rather than API failures.

\subsection{Frontier LLMs Copying Test}
Here we give details for \Cref{fig:api-acc}.

\paragraph{Models and API endpoints.}
We evaluate five closed frontier models: \texttt{gpt-5.5}, \texttt{claude-opus-4-7}, \texttt{gemini-3.1-pro-preview}, \texttt{deepseek-v4-pro}, and \texttt{qwen3.6-plus}. We use the OpenAI Responses API for \texttt{gpt-5.5}, the Anthropic Messages API for \texttt{claude-opus-4-7}, Gemini Vertex Express for \texttt{gemini-3.1-pro-preview}, and OpenAI-compatible chat-completion APIs for \texttt{deepseek-v4-pro} and \texttt{qwen3.6-plus}. We do not set explicit stop sequences. We use temperature \(0\) where supported; for \texttt{gpt-5.5}, we do not pass an explicit temperature parameter and use the provider default. For Gemini, we set the thinking level to \texttt{LOW}; for DeepSeek and Qwen, we disable thinking mode.

\paragraph{Recursive-Flip Binary Copy.}
For the first binary-copy evaluation, we test the Recursive-Flip distribution defined in \cref{sec:copy-benchmarks}. We use \(50\) samples for each \(k\in\{7,8,9,10,11\}\), with global random seed \(20260414\). For each \(k\), the logical target length \(n\) satisfies
\[
    2^k \leq n \leq 2^{k+1}-1.
\]
Here, length refers to the number of logical binary units. For each sample, the length is uniformly and randomly selected from this region. To align tokenization across models, we use model-dependent encodings: for \texttt{gpt-5.5}, \texttt{claude-opus-4-7}, and \texttt{deepseek-v4-pro}, $\cV=\{\texttt{000}, \texttt{111}\}$; for \texttt{gemini-3.1-pro-preview} and \texttt{qwen3.6-plus}, $\cV=\{\texttt{0},\texttt{1}\}$. This ensures that each binary unit is one token for the tested model.

\paragraph{Imbalanced Binary Copy.}
For the second binary-copy evaluation, we use an imbalanced distribution with $p$ fixed to be \(0.95\) (see the original definition of imbalanced data in \Cref{sec:copy-benchmarks}). Here let $\cV=\{\texttt{" a"},\texttt{" b"}\}$, so that
\[
    \Pr[\mathtt{a}]=0.95,
    \qquad
    \Pr[\mathtt{b}]=0.05.
\]
Both symbols are one token in all five tested models. Thus, the model is asked to copy a space-separated symbolic sequence rather than a bare binary string. For output normalization, we extract all letters matching \texttt{[abAB]}, map \texttt{A} to \texttt{a}, and map \texttt{B} to \texttt{b}. Exact match compares this extracted and mapped string with the target.

\paragraph{Binary-copy API configuration.}
For Recursive-Flip binary copy, the models are evaluated concurrently with one model-level worker per model. Within each model, the maximum number of concurrent sample requests is \(10\) for \texttt{gpt-5.5}, \texttt{claude-opus-4-7}, \texttt{deepseek-v4-pro}, and \texttt{qwen3.6-plus}, and \(5\) for \texttt{gemini-3.1-pro-preview}. The request timeout is \(2000\) seconds, and the sleep times before the second and third attempts are \(2\) and \(4\) seconds, respectively. For OpenAI-compatible clients, we set \(\texttt{max\_retries}=0\) so that retries are controlled by our explicit retry loop. The output-token budgets for Recursive-Flip binary copy are
\[
\begin{array}{|c|ccccc|}
\hline
\text{Model} & \texttt{gpt} & \texttt{claude} & \texttt{gemini} & \texttt{deepseek} & \texttt{qwen} \\
\hline
\text{Max output tokens} & 50000 & 8192 & 16384 & 8192 & 8192 \\
\hline
\end{array}
\]
For imbalanced symbolic copy, the maximum output-token budget is computed as \(4\cdot \mathrm{target\_length}+128\).

\paragraph{Python List Conversion.}
We also evaluate the same five models on a Python List Conversion task. Each example is generated from a synthetic one-dimensional physical signal. We sample one scenario uniformly from nine signal families: simple harmonic motion, triangle wave motion, sawtooth scan motion, rectified AC signal, clipped sensor oscillation, relaxation oscillator, Fourier periodic waveform, pulse train, and pendulum-like oscillation. For each example, we sample a period \(p\), generate one cycle of length \(p\), and then obtain the length-\(n\) sequence by periodic indexing:
\[
    y_k=\mathrm{cycle}_{k\bmod p},
    \qquad
    k=0,\ldots,n-1.
\]
The target output is the same sequence written as a Python list,
\[
    [y_0,y_1,\ldots,y_{n-1}].
\]
All values are rounded to two decimal places. In the API evaluation, we use \(2\leq p\leq 10\) and evaluate lengths
\[
    n\in\{200,300,400,500\}.
\]
For each length, we generate \(50\) examples using random seed \(2030\). The datasets are saved before model evaluation, and all models are evaluated on the same saved examples. The API endpoints and output-token budgets are the same as in Recursive-Flip binary copy, with a global request timeout of \(6000\) seconds. For evaluation, we strip Markdown code fences from both the model output and the gold output, then extract all numbers matching the regular expression \verb|-?\d+(?:\.\d+)?|. Sequence exact match is true if and only if the extracted numeric sequence from the model output exactly equals that from the gold output.

\subsection{Repeat-structure Copy Test for Frontier LLMs}\label{sec:repeat-test}
Here we give details for \Cref{fig:repeat-acc}.

\paragraph{Models.}
We evaluate \texttt{gpt-5.5}, \texttt{gemini-3.1-pro-preview}, and \texttt{qwen3.6-plus}. The resources are same as the previous subsection.

\paragraph{Data distribution.}
We use random seed \(1320\), target length \(1000\), base lengths
\[
    m \in \{4,6,8,10\},
\]
and flip probabilities
\[
    p \in \{0.02,0.05,0.10,0.15,0.20,0.25,0.30\}.
\]
For each valid pair \((m,p)\), we generate \(50\) samples. A pair is included only if
\[
    (1-p)^m \geq 0.3.
\]
We first sample a base string of length \(m\) uniformly from the alphabet \(\{\texttt{" a"},\texttt{" b"}\}\), repeat it until length \(1000\), truncate the final repeat if necessary, and then independently flip each token with probability \(p\). The \(x\)-axis reports \((1-p)^m\), the probability that an entire repeated block remains unflipped.

\paragraph{API and evaluation.}
GPT-5.5 uses the OpenAI Responses API with \(\texttt{max\_output\_tokens}=5000\) and no explicit temperature parameter. Gemini uses temperature \(0.0\), \(\texttt{max\_output\_tokens}=10000\), and thinking level \texttt{LOW}. Qwen uses temperature \(0.0\), \(\texttt{max\_output\_tokens}=2500\), request timeout \(300\) seconds, request gap \(1\) second, \(\texttt{enable\_thinking=False}\), and \(\texttt{max\_retries}=0\). Final plotted data are computed from merged logs keyed by
\[
    (\texttt{model},\texttt{base\_length},p).
\]
Output normalization keeps only letters \texttt{a} and \texttt{b} from the raw response and maps them back to the symbolic tokens. Exact match is true only if the normalized prediction has length \(1000\) and agrees with the corrupted input sequence at every position.

%% file: appendix-synthetic-exp-detail.tex
\section{Details of Synthetic Experiments on Binary Copy}\label{app:synthetic-details}

We conduct training-from-scratch experiments on the synthetic binary copy dataset. For each binary string $s \in \{0,1\}^*$, the input sequence takes the form
\[
    \cS := (\texttt{A}_1,\texttt{A}_2, s, \texttt{B}_1=\symn,\texttt{B}_2, s),
\]
where the model is trained to predict the next token autoregressively, including the copied string after $\texttt{B}_2$ and the final $\eos$ token at the end of the sequence.

\paragraph{Data Distributions.}
As introduced in \cref{sec:copy-benchmarks}, we consider three data distributions: \emph{uniform}, \emph{imbalanced}, and \emph{recursive-flip}. All training is conducted on imbalanced data with string lengths from $1$ to $100$. For the imbalanced distribution, the probability of sampling the first binary token is chosen from
\[
    \{0.05,0.15,0.3,0.5,0.7,0.85,0.95\},
\]
and tokens are then sampled independently according to this probability. At test time, we evaluate on all three distributions. Thus, uniform and recursive-flip test strings are out-of-distribution relative to the training distribution.

\paragraph{Model Architectures.}
We train Transformers from scratch under two model sizes. The shallow model has $1$ layer and $2$ attention heads, where each head has dimension $512$. The deeper model has $12$ layers and $12$ attention heads, where each head has dimension $128$. Thus, the embedding dimensions are $1024$ and $1536$, respectively. We set dropout to $0$ and use bias-free linear layers. For each architecture and each positional encoding, we train $8$ models with different random seeds and report the average accuracy. We average over seeds because training on the copy task shows noticeable variance across runs, especially near the boundary between interpolation and extrapolation.

\paragraph{Training Configuration.}
The two model sizes use different optimization settings. The different hyperparameters are in \Cref{tab:pe_hyperparams}. For both settings, we use weight decay $0.01$, $\beta_2=0.95$, and a warmup length of $100$ iterations. The training batch size is $64$ with gradient accumulation over $4$ steps. The maximum context length is set to $260000$ for $1$-layer model and $26000$ for $12$-layer model, which is large enough for all training and evaluation sequences considered in these experiments.

\begin{table}[t]
\centering
\caption{Training hyperparameters for each positional encoding in the length generalization experiments. 
Peak LR and Final LR denote the maximum and minimum learning rates used in the cosine decay schedule.}
\label{tab:pe_hyperparams}
\small
\begin{tabular}{llccc}
\toprule
Setting & Positional Encoding & Peak LR & Final LR & Training Steps \\
\midrule
\multicolumn{5}{c}{\textbf{12 layers, 12 heads}} \\
\midrule
12L & RoPE              & $5\times 10^{-5}$ & $1\times 10^{-6}$ & 10000 \\
12L & ALiBi             & $5\times 10^{-5}$ & $1\times 10^{-6}$ & 10000 \\
12L & NoPE              & $5\times 10^{-5}$ & $1\times 10^{-6}$ & 10000 \\
12L & RNoPE             & $5\times 10^{-5}$ & $1\times 10^{-6}$ & 10000 \\
12L & 2D-RoPE           & $5\times 10^{-5}$ & $1\times 10^{-6}$ & 10000 \\
12L & Auto-2D-RoPE      & $5\times 10^{-5}$ & $1\times 10^{-6}$ & 30000 \\
12L & Hybrid RoPE       & $5\times 10^{-5}$ & $1\times 10^{-6}$ & 30000 \\

\specialrule{1.2pt}{2pt}{2pt}

\multicolumn{5}{c}{\textbf{1 layer, 2 heads}} \\
\midrule
1L & RoPE              & $5\times 10^{-4}$ & $5\times 10^{-5}$ & 60000 \\
1L & ALiBi             & $5\times 10^{-4}$ & $5\times 10^{-5}$ & 60000 \\
1L & NoPE              & $5\times 10^{-4}$ & $5\times 10^{-5}$ & 60000 \\
1L & 2D-RoPE           & $5\times 10^{-4}$ & $1\times 10^{-6}$ & 60000 \\
1L & Auto-2D-RoPE      & $5\times 10^{-4}$ & $1\times 10^{-4}$ & 10000 \\
\bottomrule
\end{tabular}
\end{table}

\paragraph{Evaluation.}
During training, strings have length at most $100$. At evaluation time, we test models on longer strings to measure length generalization. We evaluate lengths beyond the training range up to the longest test length used in the corresponding figure, with test lengths are uniformly selected in each interval. For each run, we sample 10 strings in each interval independently.

\paragraph{Results.}
\Cref{fig:length-generalization-combined-full} shows the full length-generalization results for all test distributions. The main text reports the Recursive-Flip setting because it is both out-of-distribution and especially challenging due to repeated substrings. Across the full results, standard positional encodings fail to extrapolate reliably beyond the training length range, while 2D-RoPE-based methods show substantially stronger length generalization.



\paragraph{Training Configuration.}
As discussed before, we train transformers from scratch on the binary copy task under two model sizes.
The first is a shallow model with $1$ layer and $2$ attention heads, where each head has dimension $512$.
The second is a deeper model with $12$ layers and $12$ attention heads, where each head has dimension $128$.
For each architecture and each positional encoding, we train $8$ models with different random seeds and report the average accuracy.
We average over seeds because training on the copy task shows noticeable variance across runs, especially near the boundary between interpolation and extrapolation.

\begin{figure}[tb]
    \centering
    \includegraphics[width=\linewidth]{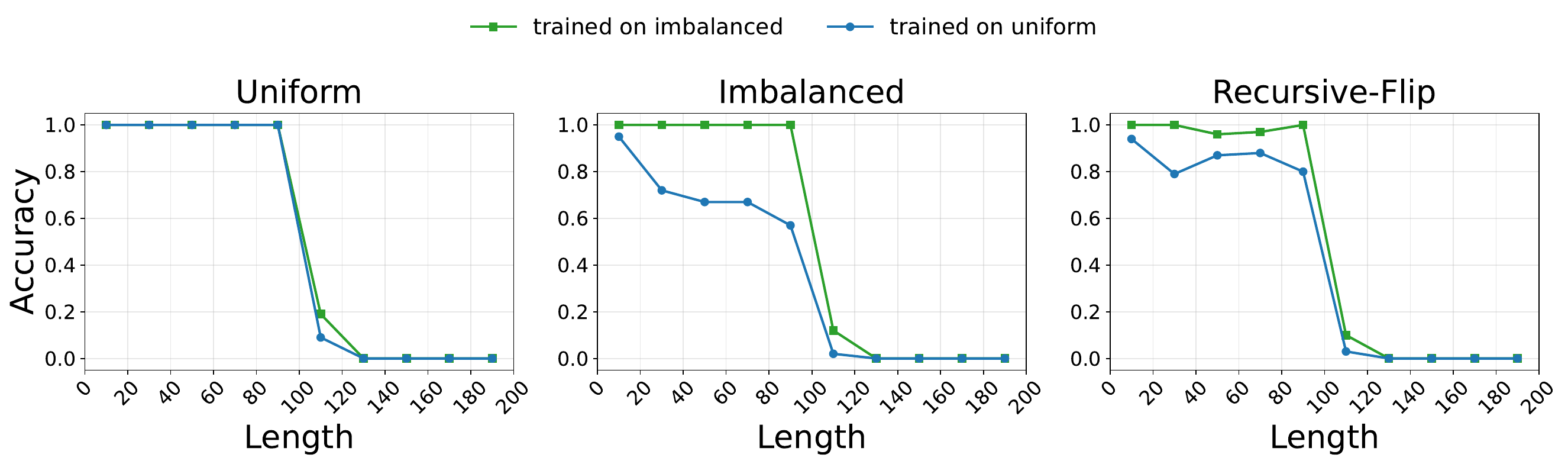}
    
    \caption{Length generalization for training-from-scratch experiments in 2 distributions.}
    \label{fig:rope-2distribution}
\end{figure}

%% file: appendix-attn-pattern.tex
\section{Further Discussion of Multi-Layer RoPE Learns Length-Dependent Pattern}\label{app:discussion-attn-pattern}

\begin{figure}[tb]
    \centering
    \includegraphics[width=\linewidth]{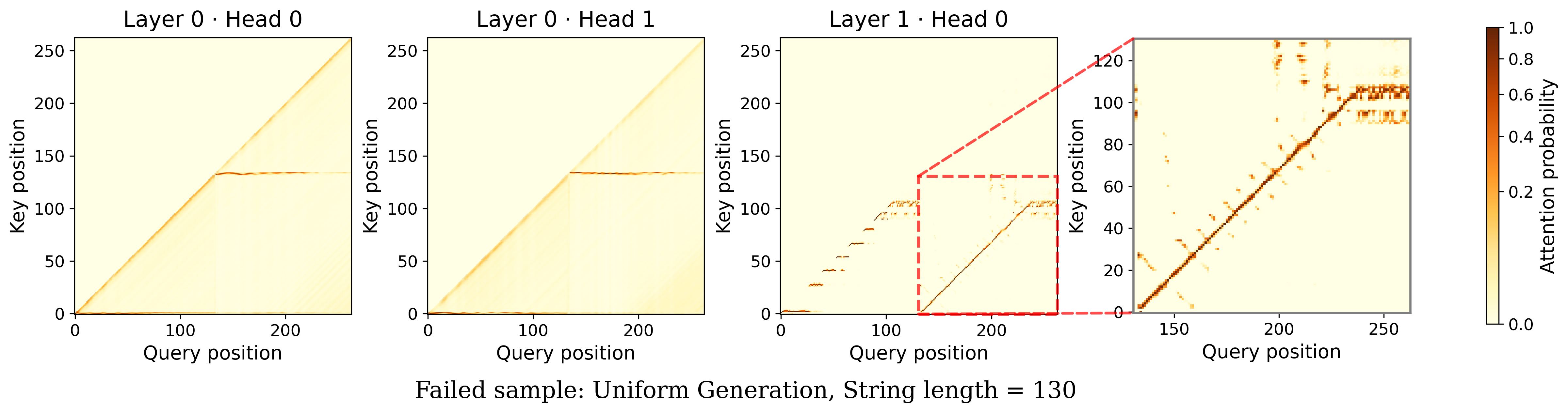}
    \caption{
    Attention map of a RoPE-based model on a failed sample with sequence length beyond the training length.
    In Layer $0$, attention heads exhibit both \textbf{horizontal patterns} (attending to special tokens in between the input string and the output, and the tokens before the string as a locator) and \textbf{diagonal patterns} (semantic/induction-style behavior).
    In Layer $1$, one can see that RoPE would attend to wrong tokens near themselves periodically.
    }
    \label{fig:rope-attn-copy-beyond-100}
\end{figure}

We provide a more detailed discussion of the attention patterns in \Cref{fig:rope-attn-copy} and \Cref{fig:rope-attn-copy-beyond-100}. The main text shows that RoPE does not learn a clean position-based copy map even on a failed example within the training length range. Here, we further examine what happens beyond the training length.

\paragraph{Within the Training Length.}
In \Cref{fig:rope-attn-copy}, the last-layer attention should ideally align each output position with its corresponding input position. However, the zoomed attention map is not a clean diagonal alignment. The head attends not only to the correct source positions, but also periodically assigns large attention scores to nearby wrong tokens around the diagonal. This suggests that the final copy decision is affected by local-context matching: when repeated or similar substrings appear, locally similar positions can compete with the correct source position. Thus, even within the training length range, RoPE can fail because the learned copy mechanism is not purely position-based.

\paragraph{Beyond the Training Length.}
\Cref{fig:rope-attn-copy-beyond-100} further shows why this learned mechanism does not length-generalize. In Layer $0$, the two heads contain horizontal locator patterns that attend to special tokens around the boundary between the input string and the output string, such as the delimiters in $[\symn,\symI,\symO]$. These patterns provide a position-related signal and can help the model infer source-target alignment within the seen length range. However, when the sequence length goes beyond the maximum training length, the locator pattern no longer remains stable: after the copied position exceeds the training range, the horizontal attention pattern becomes discontinuous and eventually breaks. This indicates that the model's position-based component is not a true length-invariant copy algorithm, but a locator-style shortcut tied to the training lengths.

At the same time, Layer $0$ also contains diagonal patterns, which resemble local-context or induction-style behavior~\citep{olsson2022context,singhneeds,jelassi2024repeat}. These patterns can attend to nearby or locally similar tokens, but they cannot uniquely identify the correct source position when repeated substrings are present. Therefore, the trained RoPE model combines two imperfect mechanisms: a locator-style positional shortcut that does not generalize in length, and a context-based matching mechanism that becomes ambiguous on repeated structures. This explains why the last-layer attention becomes noisy and why the model fails to implement the ideal position-based copy rule suggested by \Cref{thm:2L-RoPE-represent}.

%% file: appendix-implementation-2D-RoPE.tex
\section{Implementation Details of 2D-RoPE}\label{app:2drop-imp}

\subsection{Rotational Frequency Hyperparameter.}
In RoPE, there is a hyperparameter $\theta$ that determines the rotational frequency. For 2D-RoPE, the two dimensions can have separate rotational frequencies, $\theta_x, \theta_y$. However, for simplicity, we keep $\theta_x=\theta_y$. Moreover, we use the default hyperparameter from the original RoPE paper, which is $\theta=10,000$, and we set $\theta_x=\theta_y =1,000$. Our preliminary experiments show that setting $\theta_x=\theta_y=100$ or $\theta_x=\theta_y=10,000$ does not yield significantly different results. 

\subsection{Implementation Details and Comparison with Vision 2D-RoPE}
\label{app:detail-2drop-imp-cv}
RoPE-style 2D positional encodings have been explored in computer vision, where each image patch has a natural spatial coordinate. For example, prior vision models apply RoPE to patch tokens using their height and width coordinates, so that attention can encode relative spatial displacement on the image grid~\citep{jeevan2022resource,heo2024rotary,chu2024visionllama}. In this setting, the 2D coordinates are given by the image layout itself.

As we mentioned in the previous section, our implementation differs in both motivation and coordinate construction. In language modeling, the input is normally treated as a 1D token sequence, and there is no fixed 2D grid analogous to image patches. We instead use line break tokens to induce a 2D structure in text. Each token receives a row ID, determined by the number of preceding line breaks, and a column ID, determined by its offset within the current line. Thus, the 2D coordinates are not externally given but are constructed from the textual structure. Further, we introduce an alternative termed Auto 2D-RoPE for 2D-RoPE to learn the 2D structure in context adaptively in \Aref{app:adaprope}.

In our implementation, 2D-RoPE follows an axial decomposition: part of the head dimension is rotated according to the row ID, and the remaining part is rotated according to the column ID. We keep the standard causal attention mask for language modeling. This is also different from most vision applications, which typically use bidirectional attention over image patches.

%% file: appendix-details-llm-experiments.tex
\section{LLM Experiment Details}
\label{sec:llm-experiment-details}
Unless otherwise specified, the model training in the \Cref{sec:llm-experiments} follows the following settings.

\begin{table}[!ht]
    \centering
    \small
    \caption{The hyperparameters of the three model of different sizes involved in Section~\ref{sec:llm-experiments}.}
    \begin{tabular}{l|ccc}
        \toprule
        \textbf{Hyperparameter} & \textbf{350M} & \textbf{730M} & \textbf{1.4B}  \\
        \midrule
        Vocab size & 32000 & 32000 & 32000 \\
        Layers & 24 & 24 & 28 \\
        Hidden size & 1024 & 1536 & 2048\\
        MLP intermediate dim & 3072  & 4096 & 5120 \\
        MLP activation function & SwiGLU & SwiGLU & SwiGLU \\
        Attention heads & 8 & 12 & 16\\
        KV heads & 8 & 12 & 16\\
        Attention head size & 128 & 128 & 128 \\
        RoPE $\theta$ & 10,000 & 10,000 & 10,000 \\
        2D-RoPE $\theta_x,\theta_y$ & 1,000 & 1,000 & 1,000 \\
        \midrule 
        \multicolumn{4}{c}{\textit{Training Configuration}} \\
        \midrule
        Pretraining batch size & 256 & 256 & 256 \\
        Finetuning batch size & 64 & 64 & 64 \\
        Max length & 2K & 2K & 2K \\
        Pretraining max LR & 5e-4 & 5e-4 & 2e-4 \\
        Finetuning max LR & 5e-5 & 5e-5 & 5e-5 \\
        \bottomrule
    \end{tabular}
    \label{tab:llm-model-configs}
\end{table}


\subsection{LLM Model Configuration Details}
\label{app:llm-model-config-details}

The hyperparameters for the models involved in LLM experiments (Section~\ref{sec:llm-experiments}) are described in Table~\ref{tab:llm-model-configs}. These hyperparameter values are chosen to be similar to popular open-source LLMs such as Qwen3 and Llama3~\citep{grattafiori2024llama3herdmodels}. Each model uses the Llama3 tokenizer.

\subsection{LLM Pretraining Configuration Details}
\label{app:llm-training-config-details}

Here, we describe more details for the LLM pretraining~(Section~\ref{sec:llm-experiments}) for reproducibility.

\paragraph{Learning Rate Scheduler.}
Each model was trained using the WSD LR scheduler~\citep{hu2024minicpm}, with different LR depending on the model size. The 350M, 730M, and 1.4B models use 5e-4, 5e-4, and 2e-4 max LR, respectively. Each of the LRs warms up (from 0) for 100 steps, and decays for the last 10\% steps to $1/10$ of their max LR. 

\paragraph{Data Batching.}
The batch size is always 256, and the maximum sequence length during pretraining is always 2048, so the number of tokens per batch is $512$K. Each pretraining document is concatenated with an \texttt{<end\_of\_sequence>} delimiter until they exceed maximum sequence length, and then truncated down to 2048. The part that was truncated is discarded. Each token can attend to all other tokens in other documents that was concatenated into the same sequence.

\paragraph{Optimizer.}
We use the same optimizer for both pretraining and finetuning, which is AdamW~\citep{adamw} with beta values of (0.9, 0.95), a weight decay of 0.1, and gradient clipping at 1.0.

\paragraph{Hardware.}
All LLM experiments were conducted on machines equipped with NVIDIA A800-80GB GPUs, using PyTorch and HuggingFace Transformer libraries. We used BF16 precision during both training and evaluation.

\subsection{LLM Finetuning Configuration Details}
\label{app:llm-finetuning-config-details}

\paragraph{Learning Rate Scheduler.}
We use the Cosine LR scheduler for finetuning, since the cost of rerunning finetuning is much lower than pretraining. The LR warms up for 100 steps then decays down to 0.

\paragraph{Data Batching.}
For finetuning, each sequence might have different sequence length, and we use attention masking to ensure that padding tokens are not attended to. Since the training dataset is synthesized, it does not exceed 2048 tokens and there is no truncation. We always use batch size of 64 during finetuning.

\paragraph{Optimizer and Hardware.}
The optimizer for finetuning is the same as the one used for pretraining, and all experiments were conducted on the same hardware, based on the same precision and implementation libraries.

\paragraph{Loss Mask.}
During finetuning, we apply a loss mask such that the loss is only computed over the output tokens, which are tokens that make up the assistant's response in the finetuning data.




\subsection{LR Ablation in LLM Experiments}
\label{app:lr-ablation}

\begin{table}[!ht]
    
    \footnotesize
    \centering
    \caption{LR ablation results in the LLM experiments. Each model was trained with $3$ different LR values, and the main content reports the LR that achieves best overall performance in the two kinds of tasks (CSR, Copy with imbalanced and recursive-flip distributions), which is selected to be 5e-5 and is highlighted in \colorbox{orange!10}{orange}.}
    \setlength{\tabcolsep}{3pt}
    \begin{tabular}{lccc|c|cccc|cccc}
        \toprule
         & & PT &  & & \multicolumn{4}{c|}{Copy (Imbalanced)} & \multicolumn{4}{c}{Copy (Recursive-Flip)} \\
        Model & Size & Data & LR & CSR & 1K & 2K & 4K & 8K & 1K & 2K & 4K & 8K \\
                \midrule
                \multicolumn{13}{c}{\textit{Different Model Scales with Chinchilla law}} \\
                \midrule
RoPE & 350M & 7B & 1e-4 & 40.9 & 100.0 & 16.8 & 0.0 & 0.0 & 26.8 & 11.7 & 1.1 & 1.8 \\
\rowcolor{orange!10}
RoPE & 350M & 7B & 5e-5 & 41.4 & 97.2 & 14.8 & 0.0 & 0.0 & 19.6 & 8.7 & 0.0 & 0.0 \\
RoPE & 350M & 7B & 3e-5 & 41.6 & 93.2 & 9.6 & 0.0 & 0.0 & 21.6 & 8.7 & 0.0 & 0.0 \\
H-RoPE & 350M & 7B & 1e-4 & 41.2 & 100.0 & 98.2 & 18.8 & 0.2 & 100.0 & 99.0 & 72.3 & 27.3 \\
\rowcolor{orange!10}
H-RoPE & 350M & 7B & 5e-5 & 41.2 & 99.6 & 57.8 & 0.0 & 0.0 & 99.0 & 91.3 & 39.4 & 1.8 \\
H-RoPE & 350M & 7B & 3e-5 & 41.3 & 84.8 & 33.4 & 0.0 & 0.0 & 100.0 & 98.1 & 55.3 & 5.5 \\
2D-RoPE & 350M & 7B & 1e-4 & 40.1 & 99.8 & 94.8 & 81.8 & 27.2 & 100.0 & 100.0 & 100.0 & 100.0 \\
\rowcolor{orange!10}
2D-RoPE & 350M & 7B & 5e-5 & 40.4 & 97.4 & 69.2 & 33.4 & 2.4 & 97.9 & 92.2 & 75.5 & 43.6 \\
2D-RoPE & 350M & 7B & 3e-5 & 40.5 & 81.2 & 57.2 & 14.8 & 0.4 & 100.0 & 98.1 & 92.6 & 87.3 \\
\midrule
RoPE & 730M & 15B & 1e-4 & 43.3 & 100.0 & 20.8 & 0.0 & 0.0 & 20.6 & 6.8 & 0.0 & 1.8 \\
\rowcolor{orange!10}
RoPE & 730M & 15B & 5e-5 & 44.0 & 97.0 & 11.0 & 0.0 & 0.0 & 12.4 & 5.8 & 1.1 & 0.0 \\
RoPE & 730M & 15B & 3e-5 & 44.2 & 88.8 & 4.4 & 0.0 & 0.0 & 14.4 & 7.8 & 0.0 & 0.0 \\
H-RoPE & 730M & 15B & 1e-4 & 44.1 & 93.6 & 13.6 & 0.0 & 0.0 & 100.0 & 100.0 & 78.7 & 3.6 \\
\rowcolor{orange!10}
H-RoPE & 730M & 15B & 5e-5 & 45.5 & 95.4 & 11.4 & 0.0 & 0.0 & 76.3 & 72.8 & 17.0 & 0.0 \\
H-RoPE & 730M & 15B & 3e-5 & 45.4 & 84.0 & 4.0 & 0.0 & 0.0 & 96.9 & 92.2 & 35.1 & 0.0 \\
2D-RoPE & 730M & 15B & 1e-4 & 44.3 & 100.0 & 100.0 & 89.6 & 24.2 & 100.0 & 98.1 & 88.3 & 90.9 \\
\rowcolor{orange!10}
2D-RoPE & 730M & 15B & 5e-5 & 45.2 & 99.4 & 74.6 & 12.8 & 0.0 & 78.4 & 73.8 & 47.9 & 5.5 \\
2D-RoPE & 730M & 15B & 3e-5 & 45.1 & 92.8 & 41.4 & 3.6 & 0.0 & 96.9 & 89.3 & 56.4 & 36.4 \\
\midrule
RoPE & 1.4B & 28B & 1e-4 & 36.8 & 100.0 & 31.0 & 0.0 & 0.0 & 20.6 & 8.7 & 0.0 & 1.8 \\
\rowcolor{orange!10}
RoPE & 1.4B & 28B & 5e-5 & 45.9 & 99.6 & 26.8 & 0.0 & 0.0 & 15.5 & 4.9 & 0.0 & 0.0 \\
RoPE & 1.4B & 28B & 3e-5 & 46.0 & 99.8 & 23.2 & 0.0 & 0.0 & 15.5 & 7.8 & 1.1 & 0.0 \\
H-RoPE & 1.4B & 28B & 1e-4 & 46.8 & 100.0 & 92.4 & 38.8 & 5.2 & 17.5 & 5.8 & 1.1 & 0.0 \\
\rowcolor{orange!10}
H-RoPE & 1.4B & 28B & 5e-5 & 47.2 & 100.0 & 73.6 & 3.6 & 0.0 & 12.4 & 1.9 & 0.0 & 0.0 \\
H-RoPE & 1.4B & 28B & 3e-5 & 47.2 & 99.8 & 65.8 & 2.8 & 0.0 & 14.4 & 5.8 & 0.0 & 0.0 \\
2D-RoPE & 1.4B & 28B & 1e-4 & 45.9 & 100.0 & 100.0 & 100.0 & 99.0 & 99.0 & 95.1 & 93.6 & 69.1 \\
\rowcolor{orange!10}
2D-RoPE & 1.4B & 28B & 5e-5 & 46.8 & 100.0 & 100.0 & 92.0 & 61.8 & 66.0 & 54.4 & 26.6 & 1.8 \\
2D-RoPE & 1.4B & 28B & 3e-5 & 46.9 & 100.0 & 92.0 & 51.8 & 14.6 & 92.8 & 76.7 & 45.7 & 10.9 \\
        \midrule
        \multicolumn{13}{c}{\textit{Overtrained Setting}}\\
        \midrule
RoPE & 730M & 100B & 3e-5 & 46.9 & 99.5 & 23.3 & 0.0 & 0.0 & 15.5 & 7.8 & 1.1 & 1.8 \\
\rowcolor{orange!10}
RoPE & 730M & 100B & 5e-5 & 46.0 & 99.6 & 29.2 & 0.0 & 0.0 & 24.7 & 8.7 & 1.1 & 1.8 \\
RoPE & 730M & 100B & 1e-4 & 34.3 & 100.0 & 22.6 & 0.0 & 0.0 & 63.9 & 15.5 & 0.0 & 1.8 \\
H-RoPE & 730M & 100B & 3e-5 & 46.4 & 100.0 & 41.4 & 0.0 & 0.0 & 92.8 & 86.4 & 6.4 & 0.0 \\
\rowcolor{orange!10}
H-RoPE & 730M & 100B & 5e-5 & 46.2 & 99.2 & 52.1 & 0.0 & 0.0 & 100.0 & 98.1 & 28.7 & 1.8 \\
H-RoPE & 730M & 100B & 1e-4 & 43.3 & 100.0 & 74.6 & 1.7 & 0.0 & 100.0 & 100.0 & 50.0 & 1.8 \\
2D-RoPE & 730M & 100B & 3e-5 & 46.3 & 100.0 & 95.5 & 37.6 & 0.0 & 99.0 & 98.1 & 84.0 & 34.5 \\
\rowcolor{orange!10}
2D-RoPE & 730M & 100B & 5e-5 & 46.2 & 100.0 & 96.3 & 72.8 & 15.8 & 100.0 & 99.0 & 94.7 & 74.5 \\
2D-RoPE & 730M & 100B & 1e-4 & 43.4 & 100.0 & 99.8 & 94.1 & 40.7 & 100.0 & 100.0 & 100.0 & 81.8 \\
        \bottomrule
    \end{tabular}
    \label{tab:lr-ablation-results}
\end{table}

In the LLM experiments~(Section~\ref{sec:llm-experiments}), we swept different max LR in finetuning and chose the best LR that balanced the performance across different tasks. In this section, we report the results for each value of LR ($\{1\times 10^{-4},5\times 10^{-5},3\times 10^{-5}\}$), which is shown in Table~\ref{tab:lr-ablation-results}. Overall, the results show that using a larger LR leads to better copying abilities, but the CSR performance degrades due to catastrophic forgetting. Interestingly, 2D-RoPE is generally more robust to the choice of LR, leading to less catastrophic forgetting in the case of a fairly large LR.

\subsection{Common-Sense Reasoning Detailed Results}
\label{app:csr-result-details}

\begin{table}
    \centering
    \small
    \caption{The results of LLMs on each individual common-sense reasoning tasks. \colorbox{orange!10}{Orange} rows highlight the LR value that is chosen to maximize the performance across common-sense reasoning and binary copy with two distributions.}
    \setlength{\tabcolsep}{3pt}
    \begin{tabular}{lccc|cccccccc|c}
        \toprule
         & & PT & & & & & & & & &  \\ 
        Model & Size & Data & LR & MMLU & LMBD & HS & PIQA & ARC-e & ARC-c & WG & BoolQ & Avg. \\
        \midrule
        \multicolumn{13}{c}{\textit{Different Model Scales with Chinchilla law}} \\
        \midrule
        RoPE & 350M & 7B & 1e-4 & 23.6 & 27.4 & 34.4 & 64.1 & 42.9 & 24.2 & 51.8 & 58.7 & 40.9 \\
\rowcolor{orange!10}
        RoPE & 350M & 7B & 5e-5 & 23.7 & 28.8 & 35.2 & 64.8 & 43.9 & 24.2 & 51.7 & 59.3 & 41.4 \\
        RoPE & 350M & 7B & 3e-5 & 23.8 & 28.7 & 35.3 & 64.9 & 44.0 & 24.6 & 51.9 & 59.4 & \textbf{41.6} \\
        H-RoPE & 350M & 7B & 1e-4 & 25.8 & 28.7 & 35.2 & 64.3 & 43.7 & 24.1 & 50.3 & 57.4 & 41.2 \\
\rowcolor{orange!10}
        H-RoPE & 350M & 7B & 5e-5 & 24.5 & 28.4 & 35.5 & 65.0 & 43.6 & 24.9 & 50.9 & 56.5 & 41.2 \\
        H-RoPE & 350M & 7B & 3e-5 & 24.4 & 28.1 & 35.6 & 65.4 & 43.9 & 25.3 & 51.1 & 56.5 & 41.3 \\
        2D-RoPE & 350M & 7B & 1e-4 & 23.1 & 28.8 & 34.9 & 63.4 & 42.6 & 23.9 & 53.3 & 50.6 & 40.1 \\
\rowcolor{orange!10}
        2D-RoPE & 350M & 7B & 5e-5 & 23.2 & 27.8 & 35.3 & 64.2 & 43.1 & 23.5 & 52.5 & 53.5 & 40.4 \\
        2D-RoPE & 350M & 7B & 3e-5 & 23.5 & 27.4 & 35.3 & 64.2 & 43.4 & 23.5 & 53.2 & 53.5 & 40.5 \\
        \midrule
        RoPE & 730M & 15B & 1e-4 & 23.1 & 38.4 & 42.0 & 67.5 & 45.7 & 25.4 & 53.7 & 50.4 & 43.3 \\
\rowcolor{orange!10}
        RoPE & 730M & 15B & 5e-5 & 23.3 & 38.3 & 43.1 & 69.0 & 48.3 & 26.5 & 53.3 & 50.5 & 44.0 \\
        RoPE & 730M & 15B & 3e-5 & 23.2 & 38.5 & 43.2 & 69.4 & 48.9 & 26.9 & 53.9 & 49.6 & 44.2 \\
        H-RoPE & 730M & 15B & 1e-4 & 23.6 & 37.6 & 41.5 & 66.9 & 48.7 & 26.3 & 51.6 & 56.7 & 44.1 \\
\rowcolor{orange!10}
        H-RoPE & 730M & 15B & 5e-5 & 24.3 & 38.1 & 43.5 & 68.7 & 51.1 & 27.3 & 52.4 & 58.3 & \textbf{45.5} \\
        H-RoPE & 730M & 15B & 3e-5 & 24.2 & 38.1 & 43.8 & 68.3 & 50.9 & 27.8 & 52.1 & 58.1 & 45.4 \\
        2D-RoPE & 730M & 15B & 1e-4 & 26.6 & 38.9 & 42.0 & 67.8 & 47.8 & 26.5 & 52.0 & 53.3 & 44.3 \\
\rowcolor{orange!10}
        2D-RoPE & 730M & 15B & 5e-5 & 26.6 & 38.1 & 43.3 & 68.4 & 47.8 & 27.1 & 52.6 & 57.4 & 45.2 \\
        2D-RoPE & 730M & 15B & 3e-5 & 26.7 & 37.9 & 43.4 & 68.7 & 48.0 & 26.5 & 52.3 & 57.7 & 45.1 \\
        \midrule
        RoPE & 1.4B & 28B & 1e-4 & 23.0 & 18.1 & 31.5 & 51.6 & 31.5 & 24.7 & 52.1 & 61.8 & 36.8 \\
\rowcolor{orange!10}
        RoPE & 1.4B & 28B & 5e-5 & 23.4 & 42.7 & 46.7 & 69.5 & 51.9 & 28.8 & 54.0 & 50.7 & 45.9 \\
        RoPE & 1.4B & 28B & 3e-5 & 24.1 & 42.4 & 46.9 & 70.2 & 52.7 & 29.2 & 52.9 & 49.8 & 46.0 \\
        H-RoPE & 1.4B & 28B & 1e-4 & 24.6 & 42.8 & 45.3 & 68.2 & 50.6 & 27.0 & 53.7 & 61.9 & 46.8 \\
\rowcolor{orange!10}
        H-RoPE & 1.4B & 28B & 5e-5 & 25.3 & 44.5 & 47.1 & 69.3 & 51.3 & 28.0 & 53.9 & 57.9 & \textbf{47.2} \\
        H-RoPE & 1.4B & 28B & 3e-5 & 25.3 & 44.7 & 47.0 & 69.3 & 52.0 & 27.0 & 54.2 & 58.3 & \textbf{47.2} \\
        2D-RoPE & 1.4B & 28B & 1e-4 & 24.4 & 40.5 & 45.3 & 69.5 & 50.8 & 26.4 & 51.7 & 58.9 & 45.9 \\
\rowcolor{orange!10}
        2D-RoPE & 1.4B & 28B & 5e-5 & 23.7 & 41.5 & 46.8 & 70.4 & 51.8 & 26.7 & 55.2 & 58.6 & 46.8 \\
        2D-RoPE & 1.4B & 28B & 3e-5 & 23.7 & 41.5 & 47.1 & 70.2 & 52.3 & 27.2 & 54.2 & 58.8 & 46.9 \\
        \midrule
        \multicolumn{13}{c}{\textit{Overtrained Setting}}\\
        \midrule
        RoPE     & 730M & 100B & 1e-4 & 22.9 & 18.1 & 30.3 & 51.5 & 29.4 & 23.8 & 49.3 & 49.1 & 34.3 \\
\rowcolor{orange!10}
        RoPE     & 730M & 100B & 5e-5 & 23.5 & 43.0 & 46.2 & 68.5 & 48.6 & 26.8 & 51.9 & 59.1 & 46.0 \\
        RoPE     & 730M & 100B & 3e-5 & 24.4 & 44.1 & 47.1 & 69.6 & 50.9 & 28.2 & 52.4 & 58.8 & \textbf{46.9} \\
        H-RoPE     & 730M & 100B & 1e-4 & 23.3 & 37.8 & 41.7 & 66.4 & 46.2 & 26.6 & 51.9 & 52.2 & 43.3 \\
\rowcolor{orange!10}
        H-RoPE     & 730M & 100B & 5e-5 & 24.4 & 43.0 & 46.9 & 69.7 & 52.3 & 28.0 & 52.7 & 52.6 & 46.2 \\
        H-RoPE     & 730M & 100B & 3e-5 & 24.3 & 42.3 & 47.6 & 69.9 & 52.5 & 27.6 & 52.8 & 54.2 & 46.4 \\
        2D-RoPE     & 730M & 100B & 1e-4 & 23.8 & 40.0 & 41.8 & 66.1 & 44.4 & 22.7 & 52.5 & 56.2 & 43.4 \\
\rowcolor{orange!10}
        2D-RoPE     & 730M & 100B & 5e-5 & 23.9 & 42.8 & 46.5 & 68.9 & 49.6 & 26.0 & 54.6 & 57.1 & 46.2 \\
        2D-RoPE     & 730M & 100B & 3e-5 & 24.3 & 42.4 & 47.0 & 69.0 & 50.1 & 25.5 & 54.1 & 57.8 & 46.3 \\
        \bottomrule
    \end{tabular}
    \label{tab:csr-result-details}
\end{table}

For completeness, we report the performance of each model on each common-sense reasoning task mentioned in Appendix~\ref{app:llm-finetuning-config-details}, in Table~\ref{tab:csr-result-details}.
\newpage

%% file: appendix-adap-2drope.tex




\section{Auto 2D-RoPE Discussion}\label{app:adaprope}

We now describe the architecture of Auto 2D-RoPE in more detail. The goal is to keep the main advantage of 2D-RoPE, namely assigning each token a 2D position ID, while removing the hard reliance on a manually specified line-break token. In standard 2D-RoPE, the position ID of each token is updated by a fixed rule: the column index increases within a line, and the row index increases after a line break. In Auto 2D-RoPE, we instead let the model learn how the 2D position IDs should evolve along the sequence.

For each token $i$, let its 2D position ID be
\[
    P_i = \bigl(p_i^{(1)},p_i^{(2)}\bigr).
\]
Intuitively, $p_i^{(1)}$ plays the role of a column-like coordinate, while $p_i^{(2)}$ plays the role of a row-like coordinate. Unlike rule-based 2D-RoPE, these coordinates are not computed from line breaks. Instead, at each layer, we generate two update values from the hidden representation of the current token.

Concretely, let $\vh_i^{(\ell)}$ be the hidden representation of token $i$ at layer $\ell$. We apply a learnable linear projection to obtain
\[
    (a_i^{(\ell)}, b_i^{(\ell)})
    =
    \mW_{\mathrm{pos}}^{(\ell)} \vh_i^{(\ell)} + \vb_{\mathrm{pos}}^{(\ell)}.
\]
We then pass these two scalars through a sigmoid function and a rescaling factor $\alpha$:
\[
    u_i^{(\ell)} = \alpha \cdot S(a_i^{(\ell)}),
    \qquad
    v_i^{(\ell)} = \alpha \cdot S(b_i^{(\ell)}),
\]
where $S(\cdot)$ denotes the sigmoid function. In our experiments, we set $\alpha=1$ by default. This makes the update coefficients bounded and positive.

Given these updated values, the first coordinate is updated by an affine transformation
\[
    p_{i+1}^{(1)}
    =
    u_i^{(\ell)} p_i^{(1)} + v_i^{(\ell)}.
\]
Equivalently,
\[
\begin{bmatrix}
p_{i+1}^{(1)}\\
1
\end{bmatrix}
=
\begin{bmatrix}
u_i^{(\ell)} & v_i^{(\ell)}\\
0 & 1
\end{bmatrix}
\begin{bmatrix}
p_i^{(1)}\\
1
\end{bmatrix}.
\]
The second coordinate is updated by the complementary increment
\[
    p_{i+1}^{(2)}
    =
    p_i^{(2)} + \bigl(1-u_i^{(\ell)}\bigr).
\]
Putting the two coordinates together, we can write the update as
\[
\begin{bmatrix}
p_{i+1}^{(1)}\\
p_{i+1}^{(2)}\\
1
\end{bmatrix}
=
\begin{bmatrix}
u_i^{(\ell)} & 0 & v_i^{(\ell)}\\
0 & 1 & 1-u_i^{(\ell)}\\
0 & 0 & 1
\end{bmatrix}
\begin{bmatrix}
p_i^{(1)}\\
p_i^{(2)}\\
1
\end{bmatrix}.
\]
Thus, the position ID of token $i$ is determined by the cumulative product of these learned affine transformations along the sequence.

This update rule can be viewed as a learnable relaxation of the line-break rule used in 2D-RoPE. For example, if $u_i^{(\ell)}\approx 1$ and $v_i^{(\ell)}\approx 1$, then
\[
    p_{i+1}^{(1)} \approx p_i^{(1)}+1,
    \qquad
    p_{i+1}^{(2)} \approx p_i^{(2)},
\]
which resembles moving to the next column within the same row. On the other hand, if $u_i^{(\ell)}\approx 0$ and $v_i^{(\ell)}\approx 0$, then
\[
    p_{i+1}^{(1)} \approx 0,
    \qquad
    p_{i+1}^{(2)} \approx p_i^{(2)}+1,
\]
which resembles resetting the column coordinate and moving to a new row. Therefore, Auto 2D-RoPE can in principle learn behavior similar to rule-based 2D-RoPE, but it is not restricted to using the newline token as the only row separator.

After the 2D position IDs are computed, we apply the same 2D-RoPE rotation as in \Cref{sec:main}. Specifically, for each layer $\ell$, the query and key vectors are rotated according to the learned coordinates
\[
    \bigl(p_i^{(1)},p_i^{(2)}\bigr),
\]
where one part of the head dimension encodes relative differences in the first coordinate, and the other part encodes relative differences in the second coordinate. The attention layer is then computed in the usual causal way:
\[
\Attn(\mH^{(\ell)})
=
\cS\left(
\MASK\left(
R_{\mathrm{auto}}(\mH^{(\ell)}\mQ)
\bigl(R_{\mathrm{auto}}(\mH^{(\ell)}\mK)\bigr)^\top
\right)
\right)
\mH^{(\ell)}\mV,
\]
where $R_{\mathrm{auto}}$ denotes the 2D-RoPE rotation using the learned Auto 2D position IDs.

Importantly, the quantities $a_i^{(\ell)}$ and $b_i^{(\ell)}$ are not independent learnable parameters tied to absolute positions. They are generated by a shared linear map from the hidden representation of each token. Therefore, the number of parameters does not grow with the input length, and the coordinate-generation rule can be applied to longer sequences at test time. This is the main difference between Auto 2D-RoPE and simply learning a fixed position ID table.

Empirically, this learned coordinate update allows the model to recover useful 2D structure even when the input does not contain explicit newline separators. As discussed in the main text, when we replace the newline token $\symn$ with another character $*$, rule-based 2D-RoPE can no longer identify row boundaries and effectively loses its 2D structure. In contrast, Auto 2D-RoPE can still learn data-dependent position IDs and shows stronger length generalization on the copy task, as shown in \Cref{fig:copy-task-acc}.

%% file: appendix_prel.tex
\section{Preliminary for Theoretical Results}
\label{sec:appendix_prel}
In this section, we formalize our general setting for our theory.

\textbf{Theoretical Definition of Copying.}
First, we define our copying task used in theory. We always have a $0/1$ string $s$ with length at least $1$ and a copying position $1\leq i\leq |s|$.
We fix the positive prompt lengths $(a,b)$.
The input of the model is a string generated by $s,i$, denoted by 
$$\Str(s,i):=(\texttt{A}_1,\cdots,\texttt{A}_a,s,\texttt{B}_1,\cdots,\texttt{B}_b, s_{:i}),$$
where $\texttt{A}_i,\, \texttt{B}_j$ are the different tokens other than 0,1.
Here $s_{:i}$ means the length $(i-1)$ prefix of $s$.

For this input, the last token is $s_{i-1}$ (and when $i=1$, the last token is $\symO$).
The vocabulary set is then defined as 
$$\cV = \{\texttt{A}_1,\texttt{A}_2, \cdots,\texttt{A}_a, \texttt{B}_1,\texttt{B}_2, \cdots,\texttt{B}_b,  0, 1\}.$$
We model the predictor as a function 
\( P_{\vtheta} : \mathcal{V}^+ \to [0,1]^2 \), 
parameterized by \( \vtheta \), which maps an input string 
\( x \in \mathcal{V}^+ \) to a distribution over the next token in \(\{0,1\}\). 
Specifically, for any input \( x \), the output 
\( P_{\vtheta}(x) = \big([P_{\vtheta}(x)]_0, [P_{\vtheta}(x)]_1\big) \) satisfies
\[
[P_{\vtheta}(x)]_0 + [P_{\vtheta}(x)]_1 = 1, 
\quad \text{and} \quad 
0 \le [P_{\vtheta}(x)]_i \le 1 \ \text{for } i \in \{0,1\}.
\]
Thus, \( P_{\vtheta}(x) \) defines a valid probability distribution over the next token.
For each pair of $s,i$, the corresponding loss is defined as
\[
\cL_{\vtheta}(s,i)=-\log \left[P_{\vtheta}(\Str(s,i))\right]_{s_i}.
\]
Now given a distribution $\cD$, the total loss is defined as
\[
\cL(\vtheta)=\E_{(s,i)\sim\cD}[\cL_{\vtheta}(s,i)].
\]
\begin{definition}[Successful Copy]\label{def:suc-copy}
We say a model can successfully do copying in length $L$ if for any $|s|\leq L$ and $1\leq i\leq |s|$, we have $\left[P_{\vtheta}(\Str(s,i))\right]_{s_i}>\frac{1}{2}$.
\end{definition}

\textbf{Formalization of RoPE.}
Given a dimension $d$, RoPE attention is defined as a function $\Attn:\R^{T\times d}\to\R^{T\times d}$ for any $T$ parameterized by $\mK,\mQ,\mV\in \R^{d\times d}$.
For any input $\mX$, the RoPE attention is defined as
\[
\Attn\bigl(\mX\bigr)=\cS\left(\MASK\left(R\left(\mX\mQ\right) \left(R\left(\mX\mK\right) \right)^\top\right)\right)  \mX\mV^\top 
  \;\in\;\R^{T \times d}.
\]
Here, $R:\R^{T\times d}\to \R^{T\times d}$ is the RoPE function. 
Specifically, for an input $\mX\in\R^{T\times d}$, the $i$-th row $X_i$ for $1\leq i\leq L$ is mapped to $X_iR_i^\top$, where
\[
R_i \;=\; 
\begin{bmatrix}
\cos(i\beta_1) & -\sin(i\beta_1) & & & \\
\sin(i\beta_1) & \cos(i\beta_1)  & & & \\
& & \cos(i\beta_2) & -\sin(i\beta_2) & \\
& & \sin(i\beta_2) & \cos(i\beta_2)  & \\
& & & & \ddots
\end{bmatrix}
\;\in\;\R^{d\times d}.
\]
Here, $\beta_j$ is the base frequency associated with the $(2j-1,2j)$ coordinate pair. 
In our setting, these $\beta_j$'s are arbitrary hyperparameters.
Thus, $R_i$ acts block-diagonally on $X_i$, rotating each two-dimensional subspace $(X_{i,2j-1}, X_{i,2j})$ by angle $i\beta_j$.
The causal mask operator $\MASK(\cdot)$ enforces the autoregressive constraint: for $\mZ\in\R^{T\times T}$,
\[
\MASK(\mZ)_{ij} \;=\; 
\begin{cases}
\mZ_{ij}, & j \leq i, \\[3pt]
-\infty, & j > i,
\end{cases}
\]
so that position $i$ cannot attend to any future position $j>i$.
The operator $\cS(\cdot)$ applies row-wise softmax normalization: for a matrix $\mA\in\R^{T\times T}$,
\[
\cS(\mA)_{ij}
\;=\;
\frac{\exp(\mA_{ij})}{\sum_{k=1}^L \exp(\mA_{ik})},
\quad 1\leq i,j\leq T.
\]

\section{Shift Invariant Positional Encoding Cannot Solve Binary Copy}

\label{sec:1Lnegative}
In this section, we prove that one-layer shift-invariant positional encoding cannot do copying as stated in \Cref{thm:1Lnegative}.
Remind that our general theoretical setting is in \Aref{sec:appendix_prel}.

\subsection{Shift Invariant PE Formalization}
First, we mathematically define what shift-invariant positional encoding is.

Similar to the RoPE definition in \Aref{sec:appendix_prel}, shift-invariant attention is also parameterized by $\mQ,\mK,\mV\in\R^d$.
For each $\mX\in \R^{T\times d}$, the attention first computes the key vectors and query vectors
\[
\vk_i=\mK^\top \mX_{i,:}^\top,\vq_i=\mQ^\top \mX_{i,:}^\top.
\]
Then, it computes a matrix $\mA\in\R^{T\times T}$ such that
\[
\mA_{i,j}=\left\{\begin{aligned}
&-\infty&\quad i<j\\
&f(i-j,\vq_i,\vk_j)&\quad i\geq j
\end{aligned}
\right.
\]
for some function $f:\R\times\R^d\times\R^d\to\R$.
Then, the output of the attention is
\[
\Attn(\mX)=\cS(\mA)\mX\mV^\top.
\]

When we treat this attention as a model for binary copy (as defined in \Aref{sec:appendix_prel}), we first have a token embedding $\wembed{\ttC}\in\R^d$ for each $\ttC\in\cV$.
For any input $s,i$, we convert the string $\Str(s,i)$ into the input to the attention as $\mX_{j,:}=\vw_{\Str(s,i)_j}^\top$ (denote this matrix as $\mX(s,i)\in\R^{(|s|+i-1+a+b)\times d}$, where $(|s|+i-1+a+b)$ is the length of $\Str(s,i)$).
We also add another matrix parameter $\mW_O\in\R^{d\times 2}$, and the model is defined as a map from $(s,i)$ to a 2-dimensional vector:
\begin{equation}\label{eq:shift-invariant-pe}
P_\vtheta(s,i)=\softmax\left((\Attn(\mX(s,i))\mW_O)_{|s|+i-1+a+b,:}\right).
\end{equation}

\subsection{Failure of Shift Invariant PE for Binary Copy}


In this subsection, we organize and provide proof for our arguments in \Cref{thm:1Lnegative}. First, to make our results more organized and easy to understand, we give a more detailed statement of \Cref{thm:1Lnegative}.
\begin{theorem}[Detailed Statement of \Cref{thm:1Lnegative}]\label{thm:1lnegative-formal}
    For binary copy with vocabulary set $\cV = \{\texttt{A}_1,\texttt{A}_2, \cdots,\texttt{A}_a, \texttt{B}_1,\texttt{B}_2, \cdots,\texttt{B}_b,  0, 1, \symn,\symO\}$, no one-layer Transformer with shift-invariant positional encodings in \Cref{eq:shift-invariant-pe} can successfully copy all input strings of length $n \le 5$. The definition of successful copy is following \Cref{def:suc-copy}.
\end{theorem}
\paragraph{Notations.} We first recall some notations used in \Aref{sec:appendix_prel} and introduce some other necessary notations. We denote the input prefix string $(\texttt{A}_1,\texttt{A}_2, \cdots,\texttt{A}_a, s,\texttt{B}_1,\texttt{B}_2, \cdots,\texttt{B}_b, s_{:i})$ as $\Str(s,i)$. We consider the last token of $\Str(s,i)$ and denote it by $C(s,i)$.
For each word $\ttC\in \cV$, notice that $\wembed{\ttC}\mV^\top\mW_O = (x_1,x_2)^\top$ is a 2-dimensional vector, we denote $\Delta_\ttC := x_1 - x_2$ as the first coordinate minus the second coordinate of it.
Moreover, we define $\qembed{\ttC} :=\wembed{\ttC}\mQ^\top\in\R^{1 \times d},\kembed{\ttC}:=\wembed{\ttC}\mK^\top\in\R^{1 \times d}$.
To avoid confusion, when we refer to $\vq_{\ttC}$ and $\vk_{\ttC}$ for $\ttC=0,1$, we write $\vq_{(0)}$ and $\vq_{(1)}$, whose subscripts are not $0$, $1$ but $(0)$, $(1)$ to indicate that the index is a token rather than a position.

For any input $s,i$, we convert the string $\Str(s,i)$ into the input to the attention as $\mX_{j,:}=\vw_{\Str(s,i)_j}^\top$. We denote this matrix as $\mX(s,i)\in\R^{(|s|+i-1+a+b)\times d}$, where $(|s|+i-1+a+b)$ is the length of $\Str(s,i)$.
We also add another matrix parameter $\mW_O\in\R^{d\times 2}$, and the model is defined as a map from $(s,i)$ to a 2-dimensional vector:
\begin{equation}\label{eq:shift-invariant-pe}
P_\vtheta(s,i)=\softmax\left((\Attn(\mX(s,i))\mW_O)_{(|s|+i-1+a+b),:}\right).
\end{equation}

Therefore, a model can successfully do binary copy at the $i$-th index as $\left[P_{\vtheta}(\Str(s,i))\right]_{s_i}>\frac{1}{2}$ has an equivalent criterion as the following lemma states.
\begin{lemma}\label{lm:eqi-form}
The inequality 
\begin{align*}
    \left[P_{\vtheta}(\text{\Str}(s,i))\right]_{s_i}>\frac{1}{2}
\end{align*}
is equivalent to that:
\begin{equation}
\label{eq:1Lfail}
(-1)^{s_i}\sum_{j=1}^{|s|+i-1+a+b}\exp(\mA_{|s|+i-1+a+b,j})\Delta_{\text{\Str}(s,i)_j}>0,
\end{equation}
where $\mA$ is the attention matrix for input $\mX(s,i)$ and $s_i \in \{0,1\}$.
\end{lemma}

\begin{proof}[Proof of \Cref{lm:eqi-form}]
    Let 
\[
    T:=|\Str(s,i)|=|s|+i-1+a+b,
    \qquad x:=\Str(s,i).
\]
Write the two logits at the last position as
\[
    \vz=(z_0,z_1)
    :=(\Attn(\mX(s,i))\mW_O)_{T,:}.
\]
By the definition of the model in \cref{eq:shift-invariant-pe},
\[
    P_{\vtheta}(\Str(s,i))
    =
    \softmax(\vz).
\]
Since the output space only contains the two binary tokens \(0\) and \(1\), for any
\(b\in\{0,1\}\) we have
\[
    [\softmax(\vz)]_b
    =
    \frac{\exp(z_b)}{\exp(z_0)+\exp(z_1)}.
\]
Therefore, if \(s_i=0\), then
\begin{align*}
     [P_{\vtheta}(\Str(s,i))]_0>\frac12
    &\quad\Longleftrightarrow\quad
    \frac{\exp(z_0)}{\exp(z_0)+\exp(z_1)}>\frac12\\
    &\quad\Longleftrightarrow\quad
    \exp(z_0)>\exp(z_1)\\
    &\quad\Longleftrightarrow\quad
    z_0-z_1>0.
\end{align*}
Similarly, if \(s_i=1\), then
\[
    [P_{\vtheta}(\Str(s,i))]_1>\frac12
    \quad\Longleftrightarrow\quad
    z_1-z_0>0
    \quad\Longleftrightarrow\quad
    z_0-z_1<0.
\]
Combining the two cases, and using \(s_i\in\{0,1\}\), we obtain the equivalent condition
\[
    (-1)^{s_i}(z_0-z_1)>0.
\]

It remains to rewrite \(z_0-z_1\) in terms of the attention scores. By the definition of the attention layer,
\[
    \Attn(\mX(s,i))_{T,:}
    =
    \sum_{j=1}^{T}
    \cS(\mA)_{T,j}\,
    \vw_{x_j}\mV^\top,
\]
where \(x_j=\Str(s,i)_j\). Since the last position \(T\) can causally attend to all positions
\(1,\ldots,T\), the softmax-normalized attention weight is
\[
    \cS(\mA)_{T,j}
    =
    \frac{\exp(\mA_{T,j})}
    {\sum_{\ell=1}^{T}\exp(\mA_{T,\ell})}.
\]
Moreover, under the shift-invariant attention model, each attention score in this last row has the form
\[
    \mA_{T,j}
    =
    f(T-j,\vq_{x_T},\vk_{x_j})
    =
    f(T-j,\vq_{C(s,i)},\vk_{\Str(s,i)_j}),
\]
where \(C(s,i)=x_T\) is the last token of \(\Str(s,i)\).

Now multiplying the last hidden vector by \(\mW_O\), we get
\[
    \vz
    =
    \sum_{j=1}^{T}
    \cS(\mA)_{T,j}\,
    \vw_{x_j}\mV^\top\mW_O.
\]
For each token \(\ttC\in\cV\), recall that
\[
    \Delta_{\ttC}
    =
    \bigl(\wembed{\ttC}\mV^\top\mW_O\bigr)_0
    -
    \bigl(\wembed{\ttC}\mV^\top\mW_O\bigr)_1 .
\]
Thus,
\[
    z_0-z_1
    =
    \sum_{j=1}^{T}
    \cS(\mA)_{T,j}
    \Delta_{x_j}
    =
    \frac{
    \sum_{j=1}^{T}
    \exp(\mA_{T,j})\Delta_{x_j}
    }{
    \sum_{\ell=1}^{T}\exp(\mA_{T,\ell})
    }.
\]
The denominator is strictly positive, so it does not affect the sign. Hence
\[
    (-1)^{s_i}(z_0-z_1)>0
    \quad\Longleftrightarrow\quad
    (-1)^{s_i}
    \sum_{j=1}^{T}
    \exp(\mA_{T,j})\Delta_{x_j}>0.
\]
Substituting \(T=|s|+i-1+a+b\) and \(x_j=\Str(s,i)_j\) gives exactly
\[
    (-1)^{s_i}
    \sum_{j=1}^{|s|+i-1+a+b}
    \exp(\mA_{|s|+i-1+a+b,j})\Delta_{\Str(s,i)_j}>0.
\]
This proves the equivalence.
\end{proof}
With the above lemma, now we are ready to give the proof of \Cref{thm:1lnegative-formal} by contradiction.
\begin{proof}[Proof of \Cref{thm:1lnegative-formal}]
Now, assume that the model can successfully solve binary copy with any string of length at most $5$.
We prove it cannot be true by contradiction.
We consider the following four specific pairs of $(s,i)$:
\begin{itemize}
\item $s_1=(00010),i_1=2$
\item $s_2=(01000),i_2=2$
\item $s_3=(100),i_3=3$
\item $s_4=(001),i_4=3$.
\end{itemize}
We first consider
\begin{align*}
u&:=\Str(s_1,i_1)=(\texttt{A}_1,\texttt{A}_2, \cdots,\texttt{A}_a,0,0,0,1,0,\texttt{B}_1,\texttt{B}_2, \cdots,\texttt{B}_b,0),\\
v&:=\Str(s_2,i_2)=(\texttt{A}_1,\texttt{A}_2, \cdots,\texttt{A}_a,0,1,0,0,0,\texttt{B}_1,\texttt{B}_2, \cdots,\texttt{B}_b,0).
\end{align*}
By our assumption that the model can successfully solve binary copy with any string of length at most $5$ and \Cref{lm:eqi-form}, \cref{eq:1Lfail} holds for $u$ and $v$. We further rewrite \cref{eq:1Lfail} for these two strings as
\begin{align*}
&\sum_{\substack{1\leq j\leq 6+a+b\\j\neq a+2,a+4}}\exp(f(6+a+b-j,\vq_{(0)},\vk_{u_{j}}))\Delta_{u_j}+\exp(f(4+b,\vq_{(0)},\vk_{(0)}))\Delta_{(0)}\\
&+\exp(f(2+b,\vq_{(0)},\vk_{(1)}))\Delta_{(1)}> 0,\\
&\sum_{\substack{1\leq j\leq 6+a+b\\j\neq a+2,a+4}}\exp(f(6+a+b-j,\vq_{(0)},\vk_{v_{j}}))\Delta_{v_j}+\exp(f(4+b,\vq_{(0)},\vk_{(1)}))\Delta_{(1)}\\
&+\exp(f(2+b,\vq_{(0)},\vk_{(0)}))\Delta_{(0)}< 0.
\end{align*}
The above inequalities use the property of our shift-invariant model that $$\mA_{|s|+i-1+a+b,j}=f(|s|+i-1+a+b-j,\vq_{|s|+i-1+a+b},\vk_j).$$
By comparing the above two formulas, and noticing that $u_j=v_j$ for $j\neq 4,6$, we have
\begin{equation}
\label{eq:L1negativemiddle}
\begin{aligned}
&\exp(f(4+b,\vq_{(0)},\vk_{(0)}))\Delta_{(0)}+\exp(f(2+b,\vq_{(0)},\vk_{(1)}))\Delta_{(1)}\\
>&\exp(f(4+b,\vq_{(0)},\vk_{(1)}))\Delta_{(1)}+\exp(f(2+b,\vq_{(0)},\vk_{(0)}))\Delta_{(0)}.
\end{aligned}
\end{equation}

Now, we similarly compare \cref{eq:1Lfail}, and denote
\begin{align*}
u'&:=\Str(s_3,i_3)=(\texttt{A}_1,\texttt{A}_2, \cdots,\texttt{A}_a,1,0,0,\texttt{B}_1,\texttt{B}_2, \cdots,\texttt{B}_b,1,0),\\
v'&:=\Str(s_4,i_4)=(\texttt{A}_1,\texttt{A}_2, \cdots,\texttt{A}_a,0,0,1,\texttt{B}_1,\texttt{B}_2, \cdots,\texttt{B}_b,0,0).
\end{align*}
We have
\begin{align*}
&\sum_{\substack{1\leq j\leq 5+a+b\\j\neq a+1,a+3}}\exp(f(5+a+b-j,\vq_{(0)},\vk_{u_{j}'}))\Delta_{u_j'}+\exp(f(4+b,\vq_{(0)},\vk_{(1)}))\Delta_{(1)}\\
&+\exp(f(2+b,\vq_{(0)},\vk_{(0)}))\Delta_{(0)}> 0,\\
&\sum_{\substack{1\leq j\leq 5+a+b\\j\neq a+1,a+3}}\exp(f(5+a+b-j,\vq_{(0)},\vk_{v_{j}'}))\Delta_{v_j'}+\exp(f(4+b,\vq_{(0)},\vk_{(0)}))\Delta_{(0)}\\
&+\exp(f(4+b,\vq_{(0)},\vk_{(1)}))\Delta_{(1)}< 0.
\end{align*}
Therefore,
\begin{align*}
&\exp(f(b+4,\vq_{(0)},\vk_{(0)}))\Delta_{(0)}+\exp(f(b+2,\vq_{(0)},\vk_{(1)}))\Delta_{(1)}\\
<&\exp(f(b+4,\vq_{(0)},\vk_{(1)}))\Delta_{(1)}+\exp(f(b+2,\vq_{(0)},\vk_{(0)}))\Delta_{(0)},
\end{align*}
which contradicts to \cref{eq:L1negativemiddle}.
Thus, we proved \cref{thm:1lnegative-formal}.
\end{proof}

\section{2-Layer RoPE Transformer can Represent Binary Copy}
\label{subsec:ropepositive}
In this section, we show the existence of a Transformer that can exactly perform binary copy for sequence lengths significantly larger than the norm constraint.



\subsection{Prompt Template}

We recall the binary copy task in \Aref{sec:appendix_prel} defined over the vocabulary set 
\[
\cV = \{\texttt{A}_1,\texttt{A}_2, \cdots,\texttt{A}_a, \texttt{B}_1=\symn,\texttt{B}_2, \cdots,\texttt{B}_b,  0, 1\}.
\]
The full copy task is specified by a binary string $s\in\{0,1\}^*$. In the next-token prediction
formulation, copying $s$ is decomposed into predicting each copied token one by one. More precisely,
when we analyze the prediction of the $i$-th copied token, where $1\leq i\leq |s|$, the model has only
seen the source string $s$ and the already generated prefix of the output string. We denote this prefix
by $s_{:i}$, which is the length $(i-1)$ prefix of $s$. Thus, the input sequence seen by the model at
this prediction step is
\[
\Str(s,i)
:=
(\texttt{A}_1,\texttt{A}_2, \cdots,\texttt{A}_a,\, s,\,\texttt{B}_1,\texttt{B}_2, \cdots,\texttt{B}_b,\, s_{:i}).
\]
The target token for this input is $s_i$.
Therefore, when $i=1$, the output prefix $s_{:i}$ is empty and the last token in
$\Str(s,i)$ is $\texttt{A}_a$; when $i>1$, the last token is the previously copied token $s_{i-1}$.

Equivalently, for any strict prefix $s'$ of $s$, the model input has the form
\[
(\texttt{A}_1,\texttt{A}_2, \cdots,\texttt{A}_a,\, s,\,\texttt{B}_1,\texttt{B}_2, \cdots,\texttt{B}_b, s'),
\]
and the objective is to predict the next token of $s$ after the prefix $s'$. In the notation above,
$s'=s_{:i}$ and the target output is $s_i$.

\subsection{Transformer Architecture}
We consider the following simplified transformer architecture. 

Let the RoPE dimension $d>50$ be an even number. The embeddings $\wembed{0}, \wembed{1}, \wembed{A1}, \wembed{A2}, \wembed{B1}, \wembed{B2}, \wembed{\symO}, \wembed{\symn}\in \R^d$ are unit vectors and orthogonal to each other. The input is converted to $\mX^{(0)}\in\R^{T\times d}$, where $T$ is the total input length and $d$ is the embedding size.
Then we have one RoPE attention layer
\[
  \mH^{(1)} 
\;=\;\Attn_{1}\bigl(\mX^{(0)}\bigr)
    =\cS\left(
    \MASK \left(
        R\left(\mX^{(0)}\mQ^{(0)}\right) 
        \left(R\left(\mX^{(0)}\mK^{(0)}\right) \right)^\top
    \right)
    \right)  \mX^{(0)}(\mV^{(0)})^\top,
\]
where $\mH^{(1)} \in\R^{T \times d}$. Here, $\mK^{(0)},\mQ^{(0)},\mV^{(0)}\in\R^{d\times d}$, and $R:\R^{T\times d}\to \R^{T\times d}$ is the RoPE function. 
Specifically, for an input $\mX\in\R^{T\times d}$, the $i$-th row $X_i$ for $1\leq i\leq T$ is mapped to $X_iR_i^\top$, where
\[
R_i \;=\; 
\begin{bmatrix}
\cos(i\beta_1) & -\sin(i\beta_1) & & & \\
\sin(i\beta_1) & \cos(i\beta_1)  & & & \\
& & \cos(i\beta_2) & -\sin(i\beta_2) & \\
& & \sin(i\beta_2) & \cos(i\beta_2)  & \\
& & & & \ddots
\end{bmatrix}
\;\in\;\R^{d\times d}.
\]
Here, $\beta_j$ is the base frequency associated with the $(2j-1,2j)$ coordinate pair. 
In our setting, these $\beta_j$'s are arbitrary hyperparameters.
Thus, $R_i$ acts block-diagonally on $X_i$, rotating each two-dimensional subspace $(X_{i,2j-1}, X_{i,2j})$ by angle $i\beta_j$.
The causal mask operator $\MASK(\cdot)$ enforces the autoregressive constraint: for $\mZ\in\R^{T\times T}$,
\[
\MASK(\mZ)_{ij} \;=\; 
\begin{cases}
\mZ_{ij}, & j \leq i, \\[3pt]
-\infty, & j > i,
\end{cases}
\]
so that position $i$ cannot attend to any future position $j>i$.
The operator $\cS(\cdot)$ applies row-wise softmax normalization: for a matrix $\mA\in\R^{T\times T}$,
\[
\cS(\mA)_{ij}
\;=\;
\frac{\exp(\mA_{ij})}{\sum_{k=1}^L \exp(\mA_{ik})},
\quad 1\leq i,j\leq T.
\]

Then, the attention layer output come through a normalization layer (here we don't consider MLP layer for simplicity), and then added to $\mX^{(0)}$ by a residual link. Formally, we have
\begin{align*}
    \bar{\mH}^{(1)} &= \RMSNorm(\mH^{(1)}) \\
    \mX^{(1)} &= \mX^{(0)} + \bar{\mH}^{(1)}
\end{align*}
where
\begin{align*}
    \RMSNorm([\vh_1, \dots, \vh_L]^\top) = \left[\frac{\lambda \, \vh_1}{\|\vh_1\|_2}, \dots, \frac{\lambda \, \vh_T}{\|\vh_T\|_2}\right]^\top.
\end{align*} 

The second attention layer doesn't use positional encoding. Formally,
\[
  \mH^{(2)} 
\;=\;\Attn_{2}\bigl(\mX^{(1)}\bigr)
    =\cS\left(
    \MASK \left(
        (\mX^{(1)}(\mK^{(1)})^\top) 
        \left((\mX^{(1)}(\mQ^{(1)})^\top) \right)^\top
    \right)
    \right)  \mX^{(1)}(\mV^{(1)})^\top 
  \;\in\;\R^{T \times d}.
\]
Finally, the hidden representation $\mH^{(2)} \in \R^{T \times d}$ is projected to the vocabulary space 
via a linear map $\mW_O \in \R^{d \times |\cV|}$:
\[
  \mY \;=\; \mH^{(2)} \mW_O \;\in\; \R^{T\times |\cV|},
\]
where $\mY$ contains the output logits. 

Let $\cV^+ = \cV \cup \cV^2 \cup \cV^3 \cup \cdots$ denote the set of all nonempty finite sequences over $\cV$. 
We denote the Transformer architecture as a function
\[
F_{\vtheta}^{(d,\vbeta)} : \cV^+ \to \R^{L \times |\cV|},\quad  
F_{\vtheta}^{(d,\vbeta)}(x)=\mY,
\]
where $x \in \cV^+$ is the input sequence, 
$\vtheta$ are the learnable parameters, 
$d$ is the embedding dimension, and $\vbeta$ denotes the hyperparameters $\beta_{i},1\leq i\leq d/2$.

The prediction distribution is obtained by applying the row-wise softmax:
\begin{align*}
    P_{\vtheta}^{(d,\vbeta)}(x) 
    &= \softmax\!\left(F_{\vtheta}^{(d,\vbeta)}(x)\right), \\[4pt]
    \left[P_{\vtheta}^{(d,\vbeta)}(x)\right]_{i,j} 
    &= \frac{\exp\!\left(F_{\vtheta}^{(d,\vbeta)}(x)_{i,j}\right)}
            {\sum_{k=1}^{|\cV|} \exp\!\left(F_{\vtheta}^{(d,\vbeta)}(x)_{i,k}\right)},
\end{align*}
where $x \in \cV^+$ is the input sequence, $i$ indexes positions, 
and $j$ indexes vocabulary symbols.

\subsection{Proof of \Cref{thm:2L-RoPE-represent}: Existence of a Length-Generalizing Solution}
In this section, we show the existence of a transformer architecture that can solve binary copy for length larger than some polynomial of the parameter norm as stated in \Cref{thm:2L-RoPE-represent}. We first give a detailed statement of \Cref{thm:2L-RoPE-represent}, and we prove it directly.

\begin{theorem}[Detailed Statement of \Cref{thm:2L-RoPE-represent}]\label{thm:2L-RoPE-represent-formal}
Let $a\geq 2,b\geq 2$ be constant.
There exists a constant embedding dimension $d$ such that for any norm constraint $\rho > 10000$, there exists RoPE freqencies $\vbeta$ and a set of parameters $\vtheta \in \{ \normtwo{\vtheta} \le \rho \}$ such that Transformer with embedding dimension $d$, RoPE freq $\vbeta$, parameters $\vtheta$ can perfectly perform copying for all lengths $1 \le L \le L_0 := \rho^2/(1000\log\rho)$.
More precisely, for every such $s$ and every $i\in[|s|]$, we have
\begin{align*}
    \argmax_{j \in \cV} \left[ P_{\vtheta}^{(d,\vbeta)}(\texttt{A}_1,\texttt{A}_2, \cdots,\texttt{A}_a,\, s,\,\texttt{B}_1,\texttt{B}_2, \cdots,\texttt{B}_b,\, s) \right]_{|s|+a+b+i-1, j} = s_i.
\end{align*}
\end{theorem}
\begin{proof}[Proof of \Cref{thm:2L-RoPE-represent-formal}]

We fix $d=10a+10b+100$. Let $L_0=\frac{\rho^2}{1000\log\rho},\alpha=\frac{\rho}{10}.$
We prove the existence of the parameters that can perfectly do copy for every binary string $s$ with length $1\leq n:=|s|\leq L_0$.

We choose the RoPE frequencies as follows. Let $\beta_1=0$ and $\beta_2=\frac{\pi}{4L_0}$. The remaining frequencies can be chosen arbitrarily. The proof below only uses $\beta_2$.


As there are only $(a+b+2)$ kinds of token, there exists $\mK^{(0)},\mQ^{(0)}$ such that:
\begin{align*}
&\mK^{(0)}\vw_{\texttt{A}_{a-1}} = \alpha\,\mathbf{e}_1,\mK^{(0)}\vw_{\texttt{A}_{a}} = \alpha\mathbf{e}_1 + \mathbf{e}_4,
\mK^{(0)}\vw_{\texttt{B}_{b-1}} = 2\alpha\,\mathbf{e}_1,
\mK^{(0)}\vw_{\texttt{B}_b} = 2\alpha\,\mathbf{e}_1 + \mathbf{e}_4,\\
&\mK^{(0)}\wembed{\ttC}=\mathbf{0}\text{ for any other token }\ttC,\\
&\mQ^{(0)}\wembed{0}=\mQ^{(0)}\wembed{1}=\mQ^{(0)}\vw_{\texttt{B}_b}=\mathbf{e}_1 + \mathbf{e}_3,\\
&\mQ^{(0)}\wembed{\ttC}=\mathbf{0}\text{ for any other token }\ttC.
\end{align*}
Let $\vv_{\texttt{A}_{a-1}},\vv_{\texttt{A}_a},\vv_{\texttt{B}_{b-1}},\vv_{\texttt{B}_b}$ be unit vectors that are orthogonal to $\wembed{\ttC}$ for any token $\ttC$.
Let 
\begin{align*}
\mV^{(0)}\wembed{C}&=\vembed{C},\texttt{C}=\texttt{A}_{a-1},\texttt{A}_a,\texttt{B}_{b-1},\texttt{B}_b\\
    \mV^{(0)}\wembed{\ttC}&=\mathbf{0}\text{ for any other token }\ttC.
\end{align*}

Since the token embeddings are orthonormal and the above linear maps are set to zero on the orthogonal complement of their span, we have
\[
\begin{aligned}
\Vert \mK^{(0)}\Vert_F^2
&=\alpha^2+(\alpha^2+1)+4\alpha^2+(4\alpha^2+1)
=10\alpha^2+2,\\
\Vert \mQ^{(0)}\Vert_F^2
&=4\Vert \mathbf e_1+\mathbf e_3\Vert_2^2=8,\\
\Vert \mV^{(0)}\Vert_F^2
&=4.
\end{aligned}
\]
Therefore,
\[
\Vert \mK^{(0)}\Vert_F^2+\Vert \mQ^{(0)}\Vert_F^2+\Vert \mV^{(0)}\Vert_F^2
=10\alpha^2+14.
\]


Let $s_0'$ be $\texttt{B}_{b}$.
In this construction, for $1\leq i\leq n$, we have 
\begin{align*}
\vw_{\texttt{A}_{a-1}}^\top \left(\mK^{(0)}\right)^\top R_{i+1}\mQ^{(0)}\vw_{s_i}&=\alpha,\\
\vw_{\texttt{A}_{a}}^\top \left(\mK^{(0)}\right)^\top R_{i}\mQ^{(0)}\vw_{s_i}&=\alpha+\sin(i\beta_2),\\
\vw_{\texttt{B}_{b}}^\top \left(\mK^{(0)}\right)^\top R_{i+1}\mQ^{(0)}\vw_{s_{i-1}'}&=2\alpha,\\
\vw_{\texttt{B}_{b-1}}^\top \left(\mK^{(0)}\right)^\top R_{i}\mQ^{(0)}\vw_{s_{i-1}'}&=2\alpha+\sin(i\beta_2).
\end{align*}

We define
\begin{align*}
\vh_{i}&=\frac{1}{\sqrt{1+\exp^2\left(\sin(i\beta_2)\right)}}\vv_{\texttt{A}_{a-1}}+\frac{\exp\left(\sin(i\beta_2)\right)}{\sqrt{1+\exp^2\left(\sin(i\beta_2)\right)}}\vv_{\texttt{A}_a},1\leq i\leq n\\
\tilde{\vh}_{i}&=\frac{1}{\sqrt{1+\exp^2\left(\sin(i\beta_2)\right)}}\vv_{\texttt{B}_{b}}+\frac{\exp\left(\sin(i\beta_2)\right)}{\sqrt{1+\exp^2\left(\sin(i\beta_2)\right)}}\vv_{\texttt{B}_{b-1}},1\leq i\leq n.
\end{align*}
As $\mV^{(0)}\wembed{0,1}=\mathbf{0}$, we have
\[
\bar{\mH}^{(1)}_{i+a}=\lambda\vh_i,1\leq i\leq n.
\]
Now we consider $\bar{\mH}^{(1)}_{n+a+b+i-1}$. Notice that
\begin{align*}
\frac{1}{\lambda}\bar{\mH}^{(1)}_{n+a+b+i-1}=&\frac{\exp\left(-\alpha\right)}{\sqrt{1+\exp^2\left(\sin(i\beta_2)\right)+\exp^2(-\alpha)+\exp^2(\sin((i+3+n)\beta_2)-\alpha))}}\vv_{\texttt{A}_{a-1}}\\
&+\frac{\exp\left(\sin((i+3+n)\beta_2)-\alpha\right)}{\sqrt{1+\exp^2\left(\sin(i\beta_2)\right)+\exp^2(-\alpha)+\exp^2(\sin((i+3+n)\beta_2)-\alpha))}}\vv_{\texttt{A}_{a}}\\
&+\frac{1}{\sqrt{1+\exp^2\left(\sin(i\beta_2)\right)+\exp^2(-\alpha)+\exp^2(\sin((i+3+n)\beta_2)-\alpha))}}\vv_{\texttt{B}_{b}}\\
&+\frac{\exp\left(\sin(i\beta_2)\right)}{\sqrt{1+\exp^2\left(\sin(i\beta_2)\right)+\exp^2(-\alpha)+\exp^2(\sin((i+3+n)\beta_2)-\alpha))}}\vv_{\texttt{B}_{b-1}}.
\end{align*}
Therefore,
\[
\Vert \bar{\mH}^{(1)}_{n+5+i}-\lambda\tilde{\vh}_i\Vert_2^2\leq \lambda^2\exp(-2\alpha)(1+e^2+2e^2(e^2+1)^2)\leq 10000\lambda^2\exp(-2\alpha).
\]

We need to prove the following lemma:
\begin{lemma}
Let $\vh_0=\frac{1}{\sqrt{2}}\vembed{A1}+\frac{1}{\sqrt{2}}\vembed{A2},\vh_{-1}=\vembed{A1}$. There exists $\vbeta$ such that for any $n<L_0$ and $-1\leq i<j\leq n+1$, we have $\langle \vh_i,\vh_j\rangle\leq 1-\frac{1}{72L_0^2}$.
\end{lemma}

\begin{proof}
Let $\beta_2=\frac{\pi}{4L_0}$. For $0\leq i\leq n+1$, we let $a_i=\exp(\sin(i\beta))$. Therefore, for any $0\leq i<j\leq n+1$, as $n<L_0$,
\[
\sin(j\beta)-\sin(i\beta)=\int_{i\beta}^{j\beta}\cos(x)dx\geq \frac{\beta}{\sqrt{2}}(j-i)\geq \frac{1}{2L_0}.
\]
Therefore, for any $0\leq i<j\leq n+1$,
\[
a_j-a_i=\int_{\sin(i\beta)}^{\sin(j\beta)}\exp(x)dx\geq \frac{1}{e}(\sin(j\beta)-\sin(i\beta))\geq \frac{1}{2eL_0}\geq \frac{1}{6L_0}.
\]
This shows
\[
|\arctan a_i-\arctan a_j|\geq \frac{1}{6L_0}.
\]
As $\frac{1}{e}\leq a_i\leq e$, we have
\[
\langle \vh_i,\vh_j\rangle=\cos(\arctan a_i-\arctan a_j)\leq 1-\frac{1}{72L_0^2}.
\]
Moreover, for each $0\leq i\leq n+1$,
\[
\langle \vh_i,\vh_{-1}\rangle=\frac{1}{\sqrt{1+a_i^2}}\leq 1-\frac{1}{72L_0^2}.
\]
Thus, the lemma follows.
\end{proof}

Let $\lambda=10\sqrt{L_0\log L_0}$. In the second layer, we let 
\[
(\mK^{(1)})^\top=10\sqrt{L_0}\left(\vv_{\texttt{A}_{a-1}}\vv_{\texttt{B}_{b}}^\top+\vv_{\texttt{A}_{a}}\vv_{\texttt{B}_{b-1}}^\top\right),
\quad
\mQ^{(1)}=10\sqrt{L_0}\left(\vv_{\texttt{B}_{b}}\vv_{\texttt{B}_{b}}^\top+\vv_{\texttt{B}_{b-1}}\vv_{\texttt{B}_{b-1}}^\top\right),
\]
so
\[
(\mK^{(1)})^\top\mQ^{(1)}=100L_0\left(\vv_{\texttt{A}_{a-1}}\vv_{\texttt{B}_{b}}^\top+\vv_{\texttt{A}_{a}}\vv_{\texttt{B}_{b-1}}^\top\right).
\]
By the previous lemma, for any $1\leq i\leq n+1$ and any $-1\leq j\leq n+1$ with $j\neq i$, we have
\[
\langle \vh_i,\vh_i-\vh_j\rangle
=
1-\langle \vh_i,\vh_j\rangle
\geq \frac{1}{72L_0^2}.
\]
Moreover, $\bar{\mH}^{(1)}_{i+a}=\lambda\vh_i$ and $\bar{\mH}^{(1)}_{n+a+b-1+i}$ is within $100\lambda\exp(-\alpha)$ of $\lambda\tilde{\vh}_i$. Therefore, when $\rho$ is large enough, the second-layer attention score of the correct source position $i+a$ is larger than the score of any incorrect source position by at least $100\log L_0$:
\[
\bar{\mH}^{(1)}_{i+a}(\mK^{(1)})^\top\mQ^{(1)}(\bar{\mH}^{(1)}_{n+a+b-1+i})^\top
\geq
\bar{\mH}^{(1)}_{j+a}(\mK^{(1)})^\top\mQ^{(1)}(\bar{\mH}^{(1)}_{n+a+b-1+i})^\top
+
100\log L_0.
\]
Thus, in the second layer, the attention of the query position $n+a+b-1+i$ to the correct source token $s_i$ is larger than $1/2$ for every $1\leq i\leq n$.





We let $\vembed{0},\vembed{1},\vembed{\symn}$ be unit vectors that are orthogonal to the $\vembed{C},\wembed{C}$ defined before. We let
\[
\mV^{(1)}\wembed{C}=\vembed{C},\quad C\in\{0,1,\symn\}.
\]

At last, we let $\mW_O$ be
\[
\mW_O=\vembed{0}\eembed{0}^\top+\vembed{1}\eembed{1}^\top.
\]
Here $\eembed{C}$ denotes the one-hot vector corresponding to token $C\in\{0,1,\symn\}$. Since the second layer puts more than half of its attention mass on the correct source token, the logit of the correct next token is strictly larger than the logit of the incorrect binary token. Therefore, the prediction is always correct.

For our construction, we have
\begin{align*}
\|\vtheta\|^2&=\Vert \mK^{(0)}\Vert_2^2+\Vert \mQ^{(0)}\Vert_2^2+\Vert \mV^{(0)}\Vert_2^2+\Vert \mK^{(1)}\Vert_2^2+\Vert \mQ^{(1)}\Vert_2^2+\Vert \mV^{(1)}\Vert_2^2+\Vert \mW_O\Vert_2^2+\lambda^2\\
&\leq \frac{\rho^2}{10}+100+400L_0+100L_0\log L_0\leq \rho^2.
\end{align*}
Thus the construction satisfies the norm constraint, so the theorem follows.
\end{proof}


\begin{remark}
One may argue that the above construction does not provide a solution that works for all sequence lengths. 
Indeed, such a solution cannot exist under bounded parameters: if all weights are uniformly bounded, then the attention assigned to any individual token is at most $O(1/L)$, which vanishes as the sequence length $L$ grows. 
Consequently, binary copy cannot be perfectly solved for arbitrarily long sequences.
\end{remark}

%% file: appendix-2dRoPE-theory.tex
\section{One-Layer 2D-RoPE Provably Solve Binary Copy with Length Generalization}
\subsection{Theoretical Setup}
\label{subsec:setup}

To make this section self-contained, in this subsection, we define our setup in detail before we give further theoretical results in the following subsections. 

\paragraph{Binary Copy.} For binary copy task, aligned with the previous theoretical results, we still consider the vocabulary set
\[
\cV = \{\texttt{A}_1,\texttt{A}_2, \cdots,\texttt{A}_a, \texttt{B}_1,\texttt{B}_2, \cdots,\texttt{B}_b,0,1\}.
\]
Here, to use the linebreaking in our 2D-RoPE, we assume that $\texttt{B}_1=\symn$, and other tokens are not linebreaking tokens.
For a binary string $s$ and an index $1\leq i\leq |s|$, we denote by $s_{:i}$ the length $(i-1)$ prefix of $s$. In the next-token prediction formulation, when the model predicts the $i$-th copied token, the input sequence seen by the model is
\[
\Str(s,i):=(\texttt{A}_1,\texttt{A}_2, \cdots,\texttt{A}_a,\, s,\,\texttt{B}_1,\texttt{B}_2, \cdots,\texttt{B}_b,s_{:i}).
\]
The target token for this input is $s_i$. Notice that the length of $\Str(s,i)$ is $|\Str(s,i)|=|s|+a+b-1+i$. Equivalently, the full copy sample corresponds to the sequence
\[
(\texttt{A}_1,\texttt{A}_2, \cdots,\texttt{A}_a,\, s,\,\texttt{B}_1,\texttt{B}_2, \cdots,\texttt{B}_b,s),
\]
but in the autoregressive loss, the logit at position $|s|+a+b-1+i$ is used to predict $s_i$. We further denote the set of all nonempty finite sequences over $\cV$ as $\cV^+ := \cV \cup \cV^2 \cup \cV^3 \cup \cdots$.

\paragraph{2D-RoPE Architecture.}
We now define the one-layer 2D-RoPE Transformer used in theory. Let
\[
u=(C_1,\ldots,C_T)\in\cV^+
\]
be an input sequence with length $T$, where each $C_i\in\cV$. We first convert each token into its
embedding and obtain an input matrix $\mX\in\R^{T\times d}$ by setting
\[
\mX_{i,:}=\vw_{C_i}^{\top},\qquad 1\leq i\leq T.
\]
Here, $d$ is the embedding dimension. We set $d$ to be divisible by $4$, so that the 2D-RoPE
rotation can split the coordinates into a column part and a row part. We consider the one-hot
embeddings of the tokens in $\cV$ as
\[
\wembed{0}, \wembed{1}, \wembed{\symO}, \wembed{\symn}, \vw_{A_1},\vw_{A_2},\cdots,\vw_{A_{a}},\vw_{B_1},\vw_{B_2},\cdots,\vw_{B_{b}}\in \R^d.
\]

Next, we assign a 2D position to each token using the line-break rule discussed in \Cref{sec:2drope}.
That is, the column index increases within the same row, and after the line-break token $\symn$, the
row index increases and the column index is reset. Let the 2D position of the $i$-th token be
$(x_i,y_i)$, where $x_i$ is the column index and $y_i$ is the row index.

We now define the 2D-RoPE rotation. For any matrix $\mZ\in\R^{T\times d}$, the 2D-RoPE map
$R:\R^{T\times d}\to\R^{T\times d}$ maps the $i$-th row $\mZ_{i,:}$ to
\[
\mZ_{i,:}\mR_{x_i,y_i}^{\top},
\]
where
\[
\mR_{x_i,y_i}
=
\diag\bigl(\mR_1(x_i),\dots,\mR_{d/4}(x_i),\;\mR_{d/4+1}(y_i),\dots,\mR_{d/2}(y_i)\bigr)
\in\R^{d\times d}.
\]
Here, the first $d/4$ rotation blocks encode the column coordinate, and the last $d/4$ rotation
blocks encode the row coordinate. For each $1\le j\le d/2$, the block
$\mR_j(\cdot)\in\R^{2\times 2}$ is the standard RoPE rotation matrix
\[
\mR_j(t)
\;:=\;
\begin{bmatrix}
\cos(t\beta_j) & -\sin(t\beta_j)\\
\sin(t\beta_j) & \cos(t\beta_j)
\end{bmatrix}.
\]
Here, $\beta_j$ are the base frequencies associated with the $(2j-1,2j)$ coordinate pair. In our
setting, these frequencies are arbitrary hyperparameters unless otherwise specified.

Given the above embedding and 2D-RoPE rotation, we define the one-layer 2D-RoPE attention layer as
\begin{equation}
\label{eq:2drope-attention}
  \mY=\Attn\bigl(\mX\bigr)=\cS\left(
    \MASK \left(
        R\left(\mX\mQ\right) 
        \left(R\left(\mX\mK\right) \right)^\top
    \right)
    \right)  \mX\mV^\top \in\R^{T \times 2}.
\end{equation}
Here, $\mK,\mQ\in\R^{d\times d}$ and $\mV\in\R^{2\times d}$ are the learnable parameters of this
attention layer. The output $\mY$ has two coordinates at each position, corresponding to the logits
for the two binary output symbols $0$ and $1$.

Equivalently, let
\[
\mA:=\MASK \left(R\left(\mX\mQ\right) \left(R\left(\mX\mK\right) \right)^\top\right)
\]
be the masked attention score matrix. If we denote the unrotated query and key vectors at position
$i$ by
\[
\vq_i=(\mX\mQ)_{i,:}^{\top},\qquad \vk_i=(\mX\mK)_{i,:}^{\top},
\]
then the entries of $\mA$ can be written as
\[
\mA_{i,j}=\left\{\begin{aligned}
&-\infty&\quad i<j,\\
&f((x_i,y_i),(x_j,y_j),\vq_i,\vk_j)&\quad i\geq j,
\end{aligned}
\right.
\]
where
\[
f((x_i,y_i),(x_j,y_j),\vq_i,\vk_j)=\vq_i^\top \mR_{x_j-x_i,y_j-y_i}\vk_j.
\]
The causal mask operator $\MASK(\cdot)$ enforces the autoregressive constraint: for
$\mZ\in\R^{T\times T}$,
\[
\MASK(\mZ)_{ij} \;=\; 
\begin{cases}
\mZ_{ij}, & j \leq i, \\[3pt]
-\infty, & j > i.
\end{cases}
\]

For any input sequence $u\in\cV^+$ with length $T$, we denote the above one-layer 2D-RoPE
Transformer architecture as
\begin{equation}
\label{eq:2drope-arc}
    F_{\vtheta}^{(d,\vbeta)}(u)=\mY\in\R^{T\times 2},
\end{equation}
where $\vtheta=(\mQ,\mK,\mV)$ are the learnable parameters, $d$ is the embedding dimension, and
$\vbeta$ denotes the hyperparameters $\beta_j,1\leq j\leq d/2$.

The prediction distribution is obtained by applying the row-wise softmax to these two logits:
\begin{equation}
\label{eq:2drope-predictor}
\begin{aligned}
    P_{\vtheta}^{(d,\vbeta)}(u) 
    &= \softmax\!\left(F_{\vtheta}^{(d,\vbeta)}(u)\right), \\[4pt]
    \left[P_{\vtheta}^{(d,\vbeta)}(u)\right]_{i,b} 
    &= \frac{\exp\!\left(F_{\vtheta}^{(d,\vbeta)}(u)_{i,b}\right)}
            {\sum_{b'\in\{0,1\}} \exp\!\left(F_{\vtheta}^{(d,\vbeta)}(u)_{i,b'}\right)},
\end{aligned}
\end{equation}
where $i$ indexes positions and $b\in\{0,1\}$ indexes binary output symbols. 

\paragraph{Population Loss Function.} Given a distribution $\cD$ of sequences, let $P_{\vtheta}^{(d,\vbeta)}(s)$ be the output distribution matrix given input $(\texttt{A}_1,\texttt{A}_2, \cdots,\texttt{A}_a,\, s,\,\texttt{B}_1,\texttt{B}_2, \cdots,\texttt{B}_b,s)$. The loss is defined as the cross-entropy loss on the target tokens:
\begin{align}
\cL(\vtheta)
:=
\E_{s\sim\cD}\left[
\frac{1}{|s|}
\sum_{i=1}^{|s|}
-\log
\left[
P_{\vtheta}^{(d,\vbeta)}(s)
\right]_{a+b+|s|+i-1,s_i}
\right].\label{eq:population-loss}
\end{align}
If the training length is at most $L$, we may also write the loss as $\cL_L(\vtheta)$ to emphasize the training length.

\subsection{Expressive Power of One-Layer 2D-RoPE}
\label{subsec:2dropeexpress}
Next, we prove the detailed version of \cref{thm:2dropepositive} as the following theorem states.
\begin{theorem}[Detailed Version of \Cref{thm:2dropepositive}]
\label{thm:2dropepositive-formal}
Assume that the 2D-RoPE dimension $d\geq 100+4a+4b$ is divisible by $4$. For any $\rho>d^2$, let $\beta_j\overset{\mathrm{i.i.d.}}{\sim}\Unif[0,2\pi]$ for every $1\leq j\leq d/2$. Then with probability at least $\left(1-\frac{d}{\rho}\right)$ over the randomness of $\vbeta$, there exists a $1$-layer 2D-RoPE transformer $P_{\vtheta}^{(d,\vbeta)}$ defined in \Cref{eq:2drope-attention,eq:2drope-arc,eq:2drope-predictor} such that:
\begin{enumerate}
\item $\|\vtheta\|_{2}\leq \rho$;
\item For any binary string $s$ with length $|s|\leq \rho^{d/12}$ and any $1\leq i\leq |s|$, the transformer correctly predicts the next copied token:
\[
\argmax_{b\in\{0,1\}}
\left[P_{\vtheta}^{(d,\vbeta)}(\Str(s,i))\right]_{|s|+a+b-1+i,b}
=
s_i.
\]
\end{enumerate}
\end{theorem}


\begin{proof}[Proof of \Cref{thm:2dropepositive-formal}]
Let $M=\frac{\rho}{d}$ and $L=M^{d/6}$. Since $\wembed{\ttC}$ are one-hot embeddings, there exists $\mQ,\mK,\mV$ such that, for every $\ttC\in\cV$,
\[
\vq_{\ttC}:=\mQ^\top\wembed{\ttC}
=
\mR_{a+1-b,-1}\begin{bmatrix}
    M \\ M \\\vdots \\ M
\end{bmatrix},
\qquad
\vk_{\ttC}:=\mK^\top\wembed{\ttC}
=
\mR_{0,0}\begin{bmatrix}
    M \\ M \\\vdots \\ M
\end{bmatrix}.
\]
Here $\vq_{\ttC}$ and $\vk_{\ttC}$ are all $d$-dimensional vectors. We also set
\[
\vv_{(0)}:=\mV\wembed{0}=\begin{bmatrix}
    M\\ -M
\end{bmatrix},
\qquad
\vv_{(1)}:=\mV\wembed{1}=\begin{bmatrix}
    -M\\ M
\end{bmatrix},
\]
and
\[
\vv_{\ttC}:=\mV\wembed{\ttC}=\begin{bmatrix}
    0\\ 0
\end{bmatrix}
\quad\text{for any }\ttC\in\cV\backslash\{0,1\}.
\]
By construction, $\|\vtheta\|_2=\sqrt{2(a+b+3)M^2 d+4M^2}\leq \rho$.

We first need the following technical lemma.
\begin{lemma}
\label{lm:cos_tail}
Let $\theta_1,\ldots,\theta_m \sim \Unif[0,2\pi]$ be independent. Then for any
$\delta \in (0,0.1)$, with probability at most $\delta^{m/2}$, we have
\[
\sum_{i=1}^m \cos(\theta_i)
\geq m-\delta.
\]
\end{lemma}

\begin{proof}
First consider one $\theta\sim\Unif[0,2\pi]$.
Notice that $\cos\theta=1-2\sin^2\frac{\theta}{2}$.
Therefore, if $\cos\theta\geq 1-\delta$, then $\sin^2\frac{\theta}{2}\leq \frac{\delta}{2}$.
Since $\delta<0.1$, we must have $0\leq \theta\leq \sqrt{2\delta}$ or $2\pi-\sqrt{2\delta}\leq \theta\leq 2\pi$.
Therefore, the probability of $\sin^2\frac{\theta}{2}\leq \frac{\delta}{2}$ is at most $\sqrt{\delta}$.

Notice that $\sum_{i=1}^m \cos(\theta_i)\geq m-\delta$ indicates that $\cos(\theta_i)\geq 1-\delta$ for each $i$.
Therefore, the result follows by the independency of these $\theta_i$'s.
\end{proof}

Now we consider back our construction of $\vq,\vk$. We define the following good event:
\[
\mathcal{E}:=\left\{
\sum_{k=1}^{d/4}\cos(r\beta_k)\leq \frac{d}{4}-\delta
\text{ for every }1\leq |r|\leq L+a+b+3
\right\}
\cap
\left\{
\sum_{k=d/4+1}^{d/2}\cos(\beta_k)\leq \frac{d}{4}-\delta
\right\}.
\]
For any non-zero integer $r$, $\cos(r\beta_k)$ has the same distribution as $\cos(\theta)$ for $\theta\sim\Unif[0,2\pi]$. Therefore, by \Cref{lm:cos_tail} and a union bound over all possible relative column offsets, we have
\[
\Pr(\mathcal{E})\geq 1-(L+a+b+4)\delta^{d/8}.
\]

We now work on the event $\mathcal{E}$. Fix any binary string $s$ with $n:=|s|\leq L$ and any $1\leq i\leq n$. Let
\[
T_i:=|\Str(s,i)|=n+a+b-1+i
\]
be the query position for predicting $s_i$, and let
\[
p_i:=i+a
\]
be the position of the corresponding source token $s_i$ in the input string. Under the 2D position assignment, the source position $p_i$ and the query position $T_i$ have the same column, and their row difference is $-1$. Therefore, by construction,
\[
\mA_{T_i,p_i}
=d\cdot M^2.
\]

Next, we compare this score with all other positions. First consider a first-row position $1\leq j\leq a+L+1$ with $j\neq p_i$. Then $j$ has row difference $-1$ from the query position, but its column is different from the column of $T_i$. Let $r_j$ be this non-zero column difference. Then
\[
\mA_{T_i,j}
=
2M^2\sum_{k=1}^{d/4}\cos(r_j\beta_k)+\frac{d}{2}M^2.
\]
Thus, on the event $\mathcal{E}$,
\[
\mA_{T_i,p_i}-\mA_{T_i,j}
=
\frac{d}{2}M^2
-
2M^2\sum_{k=1}^{d/4}\cos(r_j\beta_k)
\geq
2M^2\delta.
\]

Now consider a second-row position $a+L+2\leq j\leq T_i$. In this case, we consider $\mA_{T_i,j}$. The first $d/2$ items contributes at most $\frac{d}{2}M^2$, and the last $d/2$ items contributes $2M^2\sum_{k=d/4+1}^{d/2}\cos(r_j\beta_k)$.
Thus
\[
\mA_{T_i,j}
\leq 
2M^2\sum_{k=d/4+1}^{d/2}\cos(\beta_k)+\frac{d}{2}M^2.
\]
On the event $\mathcal{E}$, this implies
\[
\mA_{T_i,p_i}-\mA_{T_i,j}\geq 2M^2\delta.
\]
Therefore, for every $j\neq p_i$, we have the attention score gap
\[
\mA_{T_i,p_i}\geq \mA_{T_i,j}+2M^2\delta.
\]

It follows that the softmax attention weight on the correct source position is at least
\[
\cS(\mA)_{T_i,p_i}
\geq
\frac{1}{1+(T_i-1)\exp(-2M^2\delta)}
\geq
\frac{1}{1+3L\exp(-2M^2\delta)}.
\]
Thus, as long as
\[
3L\exp(-2M^2\delta)\leq 1,
\]
we have
\[
\cS(\mA)_{T_i,p_i}>\frac{1}{2}.
\]

Finally, we translate this attention property into the output prediction. If $s_i=0$, then the value vector at the correct source position is $\vv_{(0)}=(M,-M)^\top$. Since all non-binary tokens have zero value vectors and the attention mass on the correct source position is larger than $1/2$, the first output logit is strictly larger than the second output logit. Therefore,
\[
\argmax_{b\in\{0,1\}}
\left[P_{\vtheta}^{(d,\vbeta)}(\Str(s,i))\right]_{T_i,b}
=0=s_i.
\]
The case $s_i=1$ is symmetric, because $\vv_{(1)}=(-M,M)^\top$. Hence the model correctly predicts $s_i$ for every $s$ with $|s|\leq L$ and every $1\leq i\leq |s|$.

Now let $\delta=M^{-3/2}$ and $L=M^{d/6}$. Since $d\geq 100+4a+4b$, we have
\[
(L+a+b+4)\delta^{d/8}\leq 2M^{d/6-3d/16}=2M^{-d/48}\leq \frac{1}{M}.
\]
Moreover,
\[
3L\exp(-2M^2\delta)
=
3M^{d/6}\exp(-2M^{1/2})
\leq 1
\]
when $M$ is sufficiently large. Therefore, with probability at least $1-\frac{1}{M}=1-\frac{d}{\rho}$, the constructed one-layer 2D-RoPE Transformer can correctly copy every binary string with length at most
\[
L=M^{d/6}=\left(\rho/d\right)^{d/6}\geq \rho^{d/12}
\]
since $\rho>d^2$.
This proves the theorem.
\end{proof}

%% file: margin.tex
The following assumptions will only be used in the global-minimum analysis in \Aref{app:global-min-analysis}:
\begin{itemize}
\item For every $s$ in the support of $\cD$, we have $|s|\le L$.
\item For every length $1\le \ell\le L$, every string $s\in\{0,1\}^{\ell}$ has positive probability under $\cD$. We denote the minimum such probability by $p$.
\end{itemize}

\subsection{Global Minimum under $\ell_{\infty}$ Norm}\label{app:global-min-analysis}

In this section, we prove the positive result of 2D-RoPE for the length generalization on binary copy.
We first clarify our setting. In this section, we only consider the case where $b=a+2$.  Let $\cD$ be the training distribution on some binary strings with length at most $L$.
We make the following assumptions:
\begin{assumption}
\label{ass:prob}
There exists a $p>0$ such that for each binary string $|s|\leq L$, $\Pr_{x\sim \cD}[x=s]\geq L\cdot p$.
\end{assumption}
\begin{assumption}
\label{ass:Lbound}
The training length $L$ and hidden dimension $d$ satisfies $L>d^4$ and $d>\log^3 L$. Also, $L,d>10000+a+b$.
\end{assumption}

We also need the following assumption that the tokens other than 0/1 have zero value vectors. This is a natural assumption because intuitively, other tokens cannot give any guide to predicting the next binary token.
\begin{assumption}
\label{ass:zerovalue}
For every $\texttt{C}\in\cV\setminus\{0,1\}$, $\mV\wembed{\texttt{C}}=0.$
\end{assumption}
The above assumption states that the columns of $\mV$ corresponding to prompt and delimiter tokens are fixed to zero, while the two binary-token columns remain trainable. So in our analysis, we actually consider the following parameter class
\[
\Theta_M
:=
\left\{
(\mQ,\mK,\mV):
\|\vtheta\|_\infty\le M,\quad
\mV\wembed{\texttt{C}}=0
\ \text{for all }\
\texttt{C}\notin\{0,1\}
\right\}.
\]
And we further analyze the constrained global minimizer
\[
\vtheta^*_{M,L}
\in
\argmin_{\vtheta\in\Theta_M}\cL_L(\vtheta),
\]
where $\cL_L$ is the cross-entropy loss with sequence length up to $L$ defined in \Cref{eq:population-loss}. Also, when it is clear from the context, we ignore the subscript $L$. Before entering to the detailed proof, we first give a proof sketch in the following.

\subsubsection{Technical Challenges and Proof Sketch.} 

The proof is technically challenging for two reasons: (1) the RoPE rotation matrices make the population loss difficult to express and analyze directly; and (2) we characterize the global minimizers for finite $M$, rather than only in the asymptotic regime $M\to\infty$.

In the following, we give the proof sketch of this theorem.
First, we characterize the loss separately for every string and output position. Specifically, consider a pair $(s,i)$, where $s$ is a binary string satisfying $|s|\leq L$ and $i$ is a position at which the loss is evaluated, with $1\leq i\leq |s|$.
We define $g_{s,i}^+(\mQ,\mK,\mV)$ as the maximum attention logit assigned to any position whose token value equals $s_i\in\{0,1\}$ when predicting the $i$-th token for copying $s$.
Similarly, let $g_{s,i}^-(\mQ,\mK,\mV)$ be the maximum attention logit assigned to any position whose token value doesn't equal $s_i\in\{0,1\}$ when predicting the $i$-th token for copying $s$.
We then introduce a key intermediate quantity $f_L(\mQ,\mK,\mV)$ to characterize the global minimum $\vtheta_{L,M}^*$, which is defined as
\[
f_L(\mQ,\mK,\mV):=\min_{s,i,|s|\leq L}
\left\{g_{s,i}^+(\mQ,\mK,\mV)-g_{s,i}^-(\mQ,\mK,\mV)\right\}.
\]
Here $f_L(\mQ,\mK,\mV)$ measures the worst-case separation (over all pairs $(s,i)$) between the largest attention logit associated with the correct token and the largest attention logit associated with the incorrect token. A large value of $f_L(\mQ,\mK,\mV)$ therefore indicates that the model consistently prefers positions containing the correct token.

This construction is related to previous results showing that global minimizers approach max-margin solutions as the norm constraint $M\to\infty$~\citep{wei2019regularization,zhang2024implicit,fan2025implicit}. 
However, our setting introduces an additional difficulty that none of $\mQ$, $\mK$, or $\mV$ is fixed, so the relevant notion of margin depends jointly on all three parameter matrices. 
Therefore, our proof relates the margin-like quantity $f_L(\mQ,\mK,\mV)$ to the constrained global minimizer $\vtheta_{M,L}^*$, but it is not exactly defined as a margin.
Also, rather than considering only the asymptotic limit $M\to\infty$, we establish a guarantee for every sufficiently large but finite $M$. 
Consequently, the discrepancy caused by finite $M$ must be controlled explicitly. 

Our proof proceeds in three steps. The first two steps together imply that, once $M$ is sufficiently large, any constrained global minimizer $\vtheta_{M,L}^*$ must attain a sufficiently large value of $f_L(\mQ,\mK,\mV)$; otherwise, its population loss would be larger than that of the constructed feasible solution. The third step then converts this lower bound on $f_L(\mQ,\mK,\mV)$ into the desired length-generalization guarantee. 
\begin{enumerate} 
\item We derive a lower bound on the population loss whenever $f_L(\mQ,\mK,\mV)$ is not sufficiently large. Intuitively, $f_L$ measures the separation between the attention logits assigned to correct and incorrect tokens. Therefore, achieving a small population loss requires a sufficiently large value of $f_L$.
\item We construct a feasible parameter configuration that attains a sufficiently large value of $f_L(\mQ,\mK,\mV)$ and, consequently, a small population loss. The construction relies on a concentration property of the 2D-RoPE frequencies.
\item We show that, with high probability over the random RoPE frequencies, any parameter configuration with sufficiently large $f_L(\mQ,\mK,\mV)$ correctly copies binary strings at lengths far beyond the training range. This result also relies on concentration properties of the RoPE frequencies, together with a careful analysis of the attention patterns at long sequence lengths.
\end{enumerate} 
Combining these three steps establishes the desired length-generalization guarantee for any constrained global minimizer $\vtheta_{M,L}^*$ when $M$ is sufficiently large. We then formulate our proof. 

\subsubsection{Small Population loss Implies Large Margin}
For each $\texttt{C}\in \cV$, write $$\qembed{C}=\mQ^\top \wembed{C},\quad\kembed{C}=\mK^\top \wembed{C}.
$$

Since $\wembed{C}$ is one-hot embedding, $\|\qembed{C}\|_{\infty},\|\kembed{C}\|_{\infty}$ correspond to rows in $\mK,\mQ$. Thus the requirements on $\mQ$ and $\mK$ asks that:
\[
\|\qembed{C}\|_{\infty},\|\kembed{C}\|_{\infty}\leq M.
\]
Notice that $\mV$ is a $2\times d$ matrix. For the column corresponding to the one-hot embedding $\wembed{C}$, we denote it by $(V_{\ttC,0},,V_{\ttC,1})^\top$.
Therefore, the $\ell_{\infty}$ constraint requires $|V_{\ttC,0}|,|V_{\ttC,1}|\leq M$.

Since other rows of $\mQ,\mK$ and other columns of $\mV$ don't affect the loss, we can always assume that they are 0.
Thus, $\|\vtheta\|_{\infty}\leq M$ is equivalent to:
\begin{equation}
\label{eq:infty_requirement}
\|\qembed{C}\|_{\infty},\|\kembed{C}\|_{\infty},|V_{\ttC,0}|,|V_{\ttC,1}|\leq M.
\end{equation}

Now we need to consider the different parts of the loss term.
For each string $s$ and each index $i\le |s|$, we consider the attention distribution of the last token under the input
\[
\Str(s,i)
=(\texttt{A}_1,\ldots,\texttt{A}_a,s,
  \texttt{B}_1,\ldots,\texttt{B}_b,s_{:i}),
\]
where $s_{:i}=(s_1,\ldots,s_{i-1})$ is the length $(i-1)$ prefix of $s$. For this input, the last token is $s_{i-1}$ (and when $i=1$, the last token is $\symO$).
For each $\texttt{C}\in\cV$, let $A_{\texttt{C}}(s,i)$ denote the total attention weight from the last token assigned to all tokens equal to $\texttt{C}$.
Then the logits for predicting $s_i$ is
\[
\mY_{|s|+a+b-1+i}=\left(\sum_{\ttC\in\cV}A_{\ttC}(s,i)V_{\ttC,0},\sum_{\ttC\in\cV}A_{\ttC}(s,i)V_{\ttC,1}\right).
\]
Therefore, the loss for predicting $s_i$ is:
\begin{align*}
\cL(s,i):&=-\log\left[P_{\vtheta}^{(d,\vbeta)}
(\texttt{A}_1,\cdots,\texttt{A}_a,s,\texttt{B}_1,\cdots,\texttt{B}_b, s_{:-1})\right]_{|s|+a+b+i-1,s_i}\\
&=-\log \frac{\exp\left(\mY_{|s|+a+b+i-1,s_i}\right)}{\exp\left(\mY_{|s|+a+b+i-1,0}\right)+\exp\left(\mY_{|s|+a+b+i-1,1}\right)}\\
&=\log\left(1+\exp\left(\sum_{\ttC\in \cV}A_{\ttC}(s,i)\left(V_{\ttC,1-s_i}-V_{\ttC,s_i}\right)\right)\right).
\end{align*}
Without loss of generality, we can always assume that $V_{\ttC,0}+V_{\ttC,1}=0$.
Since we have assumed that $V_{\ttC,0}=V_{\ttC,1}=0$ for any $\ttC\in\cV\backslash\{0,1\}$, we denote $V_0:=V_{0,0}$ and $V_1:=V_{1,1}$.
Therefore,
\[
\cL(s,i)=\log\left(1+\exp\left(-2A_{s_i}(s,i)V_{s_i}+2A_{1-s_i}(s,i)V_{1-s_i}\right)\right).
\]
By definition,
\[
\cL(\vtheta)=\E_{s\sim\cD}\left[\sum_{i=1}^{|s|}\frac{1}{|s|}\cL(s,i)\right].
\]
Therefore, the contribution of each $\cL(s,i)$ to $\cL(\vtheta)$ is at least $p\cdot \cL(s,i)$.

We consider the global minimum under the $\ell_{\infty}$ constraint:
\[
\|\vtheta\|_{\infty}\leq M.
\]
We denote the global minimum as $\vtheta_{M,L}^*$. In the following proof, we just write it as $\vtheta_{M}^*$ since $L$ is always fixed. We rewrite our main result here:
\begin{theorem}
\label{thm:2drope}
Assume that each $\beta_i$ is independently sampled from $\Unif[0,2\pi]$.
Under the Assumptions~\ref{ass:prob},\ref{ass:Lbound},\ref{ass:zerovalue}, with probability at least $1-\frac{3}{d}$, there exists $M_0>0$ such that for all $M>M_0$, the transformer with parameter $\vtheta_M^*$ can correctly copy all binary strings with length at most $L^{\sqrt{d}/2}$.
\end{theorem}

In the remaining part of this subsection, we prove this theorem. We put some technical lemmas in \Aref{subsec:techlemma}.

For every training pair $(s,i)$ and $1\leq j\leq |s|+a+b+i-1$, we use $S(s,i,j)$ to denote the $j$th token in sequence $(\texttt{A}_1,\cdots,\texttt{A}_a,s,\texttt{B}_1,\cdots,\texttt{B}_b,s_{:i})$, and we use $D(s,i,j)$ 
to denote the 2-D relative position between $S(s,i,j)$ and $s_{i-1}$.

We define a function to characterize the global minimum.
For each $s,i$, we define
\begin{align*}
g_{s,i}^+(\mQ,\mK,\mV)&:=\max_{S(s,i,j)=s_i}\left\{\vq_{s_{i-1}}^\top \mR_{D(s,i,j)}\vk_{s_i}\right\};\\
g_{s,i}^-(\mQ,\mK,\mV)&:=\max_{S(s,i,j)\neq s_i}\left\{\vq_{s_{i-1}}^\top \mR_{D(s,i,j)}\vk_{S(s,i,j)}\right\}.
\end{align*}
Notice that it is well defined, because the existence of such $j$ is promised by $S(s,i,i)=s_i$ and $S(s,i,|s|+a+1)\neq s_i$.
Then, we define
\[
f_L(\mQ,\mK,\mV):=\min_{s,i,|s|\leq L}
\left\{g_{s,i}^+(\mQ,\mK,\mV)-g_{s,i}^-(\mQ,\mK,\mV)\right\}.
\]

The following lemma shows that if we have a big $f_L$, then the global minimum is small:

\begin{lemma}
\label{lm:reduction_upperbound}
For each tuple $(\mQ,\mK,\mV)$, let $S:=f_L(\mQ,\mK,\mV)$. Assume that $S>M$, then when $M$ is large enough there exists $\mV'$ such that $\|\mV'\|_{\infty}\leq M$ and
\begin{align*}
\cL\leq &\log(1+\exp(-2M))+\frac{\exp(-2M)}{1+\exp(-2M)}\cdot 4M(L+a+b)\exp(-S)\\
&+\frac{32\exp(-2M)}{1+\exp(-2M)}\cdot M^2(L+a+b)^2\exp(-2S).
\end{align*}
when we replace $\mV$ by $\mV'$.
\end{lemma}

\begin{proof}
We consider $V_0=V_1=M$ (recall that $V_0:=V_{0,0}$ and $V_1:=V_{1,1}$). 

To prove this lemma, we just need to prove the desired upper bound for any $\cL(s,i)$ such that $1\leq i\leq |s|\leq L$.
Without loss of generality, we assume that $s_i=0$ (since token 0 and 1 are symmetric).
Thus
\begin{align*}
\cL(s,i)&=\log\left(1+\exp\left(-2A_{s_i}(s,i)V_{s_i}+2A_{1-s_i}(s,i)V_{1-s_i}\right)\right)\\
&=\log\left(1+\exp\left(-2A_0(s,i)V_0+2A_1(s,i)V_1\right)\right).
\end{align*}
Recall that $V_0=V_{0,0}$ and $V_1=V_{1,1}$.

Notice that $g_{s,i}^+(\mQ,\mK,\mV)$ is the maximum attention score from $s_{i-1}$ to some token 0, and $g_{s,i}^-(\mQ,\mK,\mV)$ is the maximum attention score to other tokens.
Thus, since $f_L(\mQ,\mK,\mV)=S$, for each token other than 0, the different between the attention score to it and the largest attention score for some 0 is at least $S$. 
Thus the corresponding attention is at most $\exp(-S)$.
Therefore, $A_1(s,i)\leq 2(L+a+b)\exp(-S)$, and $A_0(s,i)\geq 1-2(L+a+b)\exp(-S)$.
Therefore,
\begin{align*}
&\quad\cL(s,i)\\
&=\log\left(1+\exp\left(-2A_{s_i}(s,i)V_{s_i}+2A_{1-s_i}(s,i)V_{1-s_i}\right)\right)\\
&\leq\log(1+\exp(-2M+4(L+a+b)\cdot\exp(-S)\cdot M))\\
&\leq \log(1+\exp(-2M))+\frac{\exp(-2M)}{1+\exp(-2M)}\cdot 4M(L+a+b)\exp(-S)\\
&\quad+\frac{32\exp(-2M)}{1+\exp(-2M)}\cdot M^2(L+a+b)^2\exp(-2S).
\end{align*}

Thus, the result follows.
\end{proof}


The following gives a lower bound of the maximum possible $f_L$.
\begin{lemma}
\label{lm:f_lowerbound}
With probability at least $1-\frac{1}{L}$ over the random frequencies, there exists $\mQ,\mK,\mV$ such that $f_L(\mQ,\mK,\mV)\geq M^2 d/2-2M^2\sqrt{d\log(L+a)}$.
\end{lemma}

\begin{proof}
For each $\ttC$, we define
\[
\vk_\ttC=\begin{bmatrix}
M \\ \vdots \\M \\ M \\\vdots\\ M
\end{bmatrix},\vq_\ttC=\begin{bmatrix}
M \\ \vdots \\M \\ -M \\\vdots\\ -M
\end{bmatrix}.
\]
Here, for each $\vq_\ttC$, there are $d/2$ $M$'s and $d/2$ $(-M)$'s.
Notice that for any $i,s$, $D(s,i,i+a)=(0,1)$. Also, for any $j\neq i+a$, there exists $-L\leq t\leq L+a$ such that $D(s,i,j)=(t,0) (t\neq 0)$ or $(t,1)$. If $D(s,i,j)=(t,0)$, 
\begin{align*}
&\vq_{s_{i-1}}^\top \mR_{D(s,i,i+1)}\vk_{S(s,i,i+1)}-\vq_{s_{i-1}}^\top \mR_{D(s,i,j)}\vk_{S(s,i,j)}\\
=&M^2\cdot\frac{d}{2}-2M^2\sum_{i=1}^{d/4}\cos(t\beta_i)
\end{align*}
If $D(s,i,j)=(t,1)$,
\begin{align*}
&\vq_{s_{i-1}}^\top \mR_{D(s,i,i+1)}\vk_{S(s,i,i+1)}-\vq_{s_{i-1}}^\top \mR_{D(s,i,j)}\vk_{S(s,i,j)}\\
=&M^2\cdot\frac{d}{2}-2M^2\sum_{i=1}^{d/4}\cos(t\beta_i)-2M^2\sum_{i=d/4+1}^{d/2}\cos(\beta_i)+M^2\cdot\frac{d}{2}.
\end{align*}
Thus
\begin{align*}
&\vq_{s_{i-1}}^\top \mR_{D(s,i,i+1)}\vk_{S(s,i,i+1)}-\vq_{s_{i-1}}^\top \mR_{D(s,i,j)}\vk_{S(s,i,j)}\\
\geq& M^2 d/2-2M^2\max\left\{\max_{-L\leq t\leq L+a,t\neq 0}\left\{\sum_{i=1}^{d/4}\cos(t\beta_i)\right\},\sum_{i=d/4+1}^{d/2}\cos(\beta_i)\right\}.
\end{align*}

Therefore, with probability at least $1-\frac{1}{L}$, we have
\[
\vq_{s_{i-1}}^\top \mR_{D(s,i,i+1)}\vk_{S(s,i,i+1)}-\vq_{s_{i-1}}^\top \mR_{D(s,i,j)}\vk_{S(s,i,j)}\geq M^2 d/2-2M^2\sqrt{d\log (L+a)}.
\]
Here, the last inequality comes from \Cref{lm:cos_chernoff} where we let $l=2L+a+2$, $m=d/4$ and $c=2$.

\end{proof}

The following lemma shows that if $V_0\leq M-1$ or $V_1\leq M-1$, then the loss is big.

\begin{lemma}
\label{lm:Vnormal}
If $V_0\leq M-1$ or $V_1\leq M-1$, then 
$$
\cL\geq (1-p)\log (1+\exp(-2M))+p\cdot\log (1+\exp(-2M+2)).
$$
\end{lemma}

\begin{proof}
Without loss of generality, let $V_0\leq M-1$, we consider the string $s$ only containing 0, we know that the loss
\[
\cL(s,i)=\log (1+\exp(-2A_0(s,i)V_0))\geq \log(1+\exp(-2M+2))
\]
since there is no token $1$ in the sequence.
Notice that for any other $s',i'$, $\cL(s',i')\geq \log(1+\exp(-2M))$. Since the probability of sampling $(s,i)$ is at least $p$, we have
\[
\cL\geq (1-p)\log (1+\exp(-2M))+p\cdot\log (1+\exp(-2M+2)).
\]
The result is similar for $V_1\leq M-1$. Therefore, the lemma follows.
\end{proof}

The following lemma shows that small loss requires a big $f_L(\mQ,\mK,\mV)$.
\begin{lemma}
\label{lm:loss_lowerbound}
For any $S>M$, if $f_L(\mQ,\mK,\mV)\leq S$, then we have
\[
\cL\geq \log (1+\exp(-2M))+\frac{\exp(-2M)}{1+\exp(-2M)}\cdot\frac{M\exp(-S)}{2L+a+b}\cdot p.
\]
\end{lemma}

\begin{proof}
First, if $V_0\leq M-1$ or $V_1\leq M-1$, by \Cref{lm:Vnormal} and \Cref{lm:inequality},
\begin{align*}
\cL&\geq (1-p)\log (1+\exp(-2M))+p\cdot\log (1+\exp(-2M+2))\\
&\geq \log (1+\exp(-2M))+\frac{\exp(-2M)}{1+\exp(-2M)}\cdot\frac{M\exp(-S)}{2L+a+b}\cdot p.
\end{align*}

In the following, we only consider the case such that $V_0,V_1>M-1$.

By definition, there exists a pair of $(s,i)$ such that
\[
g_{s,i}^+(\mQ,\mK,\mV)-g_{s,i}^-(\mQ,\mK,\mV)\leq S.
\]
Without loss of generality, we also assume that $s_i=0$.
The above inequality means that there exists a token $\ttC\neq 0$ in $s$ such that the attention score from $s_i$ to it is at least the attention score to any other minus $S$.
As there are at most $2L+a+b$ tokens in the input, the attention to token $\ttC$ is at least $\exp(-S)/(2L+a+b)$.
Therefore,
\begin{align*}
\cL(s,i)&=\log\left(1+\exp\left(-2A_{s_i}(s,i)V_{s_i}+2A_{1-s_i}(s,i)V_{1-s_i}\right)\right)\\
&\geq \log\left(1+\exp\left(-2M+M\cdot\frac{\exp(-S)}{2L+a+b}\right)\right)\\
&\geq \log (1+\exp(-2M))+\frac{\exp(-2M)}{1+\exp(-2M)}\cdot M\cdot\frac{\exp(-S)}{2L+a+b}.
\end{align*}
The last inequality is by \cref{lm:inequality}.
Notice that for any other $(s',i')$, we also have
\[
\cL(s',i')\geq \log (1+\exp(-2M)).
\]
Thus the result follows.

\end{proof}

\subsubsection{Large Training Margin Generalizes to Longer Strings}

We next show that, if $f_L(\mQ,\mK,\mV)$ is sufficiently large, then for inputs far longer than those seen during training, most of the attention is assigned to positions containing the correct token value, which nearly guarantees correct prediction.

\begin{lemma}
\label{lm:upper_cornercase}
Assume that $M>L^2$.
Then, with probability at least $1-\frac{2}{d}$, for any parameter
$\vtheta=(\mQ,\mK,\mV)$ such that
\[
f_L(\mQ,\mK,\mV)
\geq
M^2d/2-3M^2\sqrt{d\log L},
\]
and any pair $(s,i)$ with $|s|\leq L^{\sqrt{d}/2}$, the total attention to token $s_i$ over token $1-s_i$ satisfies
\[
\frac{A_{s_i}(s,i)}{A_{1-s_i}(s,i)}
\geq
\exp(M^2).
\]
\end{lemma}

\begin{proof}
We consider an arbitrary pair $(s,i)$ with $|s|\leq L^{\sqrt{d}/2}$.
Without loss of generality, we assume that $s_i=0$.
When $i=1$, we replace $\vq_{s_{i-1}}$ below by the query vector of the last prompt token $\vq_{\symO}$.

Fix a position attaining $g_{s,i}^+(\mQ,\mK,\mV)$.
We denote the contributions of the $x$-axis and the $y$-axis to this attention score by
$g_{s,i}^{+x}(\mQ,\mK,\mV)$ and
$g_{s,i}^{+y}(\mQ,\mK,\mV)$, respectively.
Thus,
\[
g_{s,i}^+(\mQ,\mK,\mV)
=
g_{s,i}^{+x}(\mQ,\mK,\mV)
+
g_{s,i}^{+y}(\mQ,\mK,\mV).
\]

Now consider an arbitrary position $j^*$ such that
\[
S(s,i,j^*)=1.
\]
We denote this token by $s^*$.
We want to show that the attention score to $s^*$ is much smaller than
$g_{s,i}^+(\mQ,\mK,\mV)$.

We construct a training pair $(s',i')$ as follows.
If $i=1$, we let
\[
|s'|=L,\qquad i'=1,
\]
and set
\[
s'_1=s_i,\qquad
s'_2=\cdots=s'_L=1-s_i.
\]
If $i>1$, we let
\[
|s'|=L,\qquad i'=L,
\]
and set
\[
(s'_{L-1},s'_L)=(s_{i-1},s_i),
\qquad
s'_1=\cdots=s'_{L-2}=1-s_i.
\]

For this construction, every position containing token $s_i$ in
$\Str(s',i')$ has the same relative position as a position containing
token $s_i$ in $\Str(s,i)$.
Therefore,
\[
g_{s',i'}^+(\mQ,\mK,\mV)
\leq
g_{s,i}^+(\mQ,\mK,\mV).
\]

We first consider the case where the relative column position between
$s^*$ and the query token is zero.
In this case, $s^*$ can only be the query position itself, and hence
$i>1$ and $s_{i-1}=1-s_i$.
The same token with the same relative position $(0,0)$ also appears in
$\Str(s',i')$.
Therefore, by the definition of $f_L$,
\[
g_{s,i}^+(\mQ,\mK,\mV)
-
\vq_{s_{i-1}}^\top
\mR_{(0,0)}
\vk_{1}
\geq
f_L(\mQ,\mK,\mV)
\geq
M^2d/2-3M^2\sqrt{d\log L}.
\]
Thus, the desired score gap already holds in this case.

In the following, we consider the case where the relative column
position of $s^*$ is nonzero.
In the constructed pair $(s',i')$, there are at least $L-2$ tokens equal
to $1$ in the same row as $s^*$.
For every such position $j$, we have
\[
g_{s',i'}^-(\mQ,\mK,\mV)
\geq
\vq_{s_{i-1}}^\top
\mR_{D(s',i',j)}
\vk_{1}.
\]
Together with
\[
g_{s,i}^+(\mQ,\mK,\mV)
\geq
g_{s',i'}^+(\mQ,\mK,\mV),
\]
this gives
\begin{equation}
\label{eq:2dattention}
g_{s,i}^+(\mQ,\mK,\mV)
-
\vq_{s_{i-1}}^\top
\mR_{D(s',i',j)}
\vk_{1}
\geq
M^2d/2-3M^2\sqrt{d\log L}.
\end{equation}

Notice that
\[
\vq_{s_{i-1}}^\top
\mR_{D(s',i',j)}
\vk_{1}
\]
consists of two parts corresponding to the $x$-axis and the $y$-axis.
We denote them by
\[
\left(
\vq_{s_{i-1}}^\top
\mR_{D(s',i',j)}
\vk_{1}
\right)_x
\quad\text{and}\quad
\left(
\vq_{s_{i-1}}^\top
\mR_{D(s',i',j)}
\vk_{1}
\right)_y.
\]

Let
\[
B:=
\left\{
\begin{matrix}
\left(
\vq_{s_{i-1}}^\top
\mR_{(0,1)}
\vk_{1}
\right)_y,
&
s^*\text{ is in the first row},
\\[2mm]
\left(
\vq_{s_{i-1}}^\top
\mR_{(0,0)}
\vk_{1}
\right)_y,
&
s^*\text{ is in the second row}.
\end{matrix}
\right.
\]
Thus, $B$ is the contribution of the $y$-axis to the attention score
of $s^*$.
Moreover, all the $L-2$ positions considered above have the same
$y$-axis contribution $B$.

By \cref{lm:cossum}, taking $\delta=3/d^2$ and applying a union bound
over the $d/4$ RoPE blocks corresponding to the $x$-axis, with
probability at least $1-\frac{1}{d}$, the absolute value of every
relevant interval sum of sines and cosines is at most $d^2$ as for any $1\leq j\leq d/4$,
\[
\left|\sum_{k=2}^{L-1} \cos(k\theta_j)\right|,\left|\sum_{k=2}^{L-1}\sin(k\theta_j)\right|,\left|\sum_{k=1}^{L-2} \cos(k\theta_j)\right|,\left|\sum_{k=1}^{L-2}\sin(k\theta_j)\right|\leq d^2.
\]

Since every coordinate of $\vq$ and $\vk$ is bounded by $M$, this gives
\[
\left|
\frac{1}{L-2}
\sum_j
\left(
\vq_{s_{i-1}}^\top
\mR_{D(s',i',j)}
\vk_{1}
\right)_x
\right|
\leq
\frac{M^2d^3}{L-2},
\]
where the sum is over the $L-2$ positions considered above.

Taking the average of \cref{eq:2dattention} over these positions, we
obtain
\[
g_{s,i}^{+x}(\mQ,\mK,\mV)
+
g_{s,i}^{+y}(\mQ,\mK,\mV)
-
B
\geq
M^2d/2
-
3M^2\sqrt{d\log L}
-
\frac{M^2d^3}{L-2}.
\]

We next bound the $x$-axis contribution to the attention score of
$s^*$.
For every nonzero relative column position $t$, the random variables
$t\beta_j\bmod 2\pi$ are independently distributed as
$\Unif[0,2\pi]$.
Therefore, by \cref{lm:cos_concentration}, for each such $t$, with
probability at least $1-L^{-\sqrt d}$,
\[
\left(
\vq_{s_{i-1}}^\top
\mR_{D(s,i,j^*)}
\vk_{1}
\right)_x
\leq
M^2
\left(
\frac{\sqrt{2}}{\pi}d
+
2d^{3/4}\sqrt{\log L}
\right).
\]
Applying a union bound over all possible relative column positions for
strings of length at most $L^{\sqrt d/2}$, with probability at least
$1-\frac{1}{d}$, this inequality holds simultaneously for every such
position.

Combining the above inequalities, the attention score gap between the
maximal position containing token $0$ and the position $j^*$ is at least
\begin{align*}
&
g_{s,i}^+(\mQ,\mK,\mV)
-
\vq_{s_{i-1}}^\top
\mR_{D(s,i,j^*)}
\vk_{1}
\\
\geq\;&
M^2d/2
-
3M^2\sqrt{d\log L}
-
\frac{M^2d^3}{L-2}
-
M^2
\left(
\frac{\sqrt{2}}{\pi}d
+
2d^{3/4}\sqrt{\log L}
\right)
\\
\geq\;&
0.01M^2d.
\end{align*}
The last inequality follows from \cref{ass:Lbound}, after increasing
the universal lower bound on $d$ and $L$ in that assumption if
necessary.

Since $j^*$ is arbitrary, the attention score to the maximal position
containing token $0$ is larger than the attention score to every
position containing token $1$ by at least $0.01M^2d$.
The total number of tokens in the input is at most
\[
2|s|+a+b
\leq
3L^{\sqrt d/2}.
\]
Therefore,
\[
\frac{A_{s_i}(s,i)}{A_{1-s_i}(s,i)}
\geq
\frac{1}{3L^{\sqrt d/2}}
\exp(0.01M^2d)
\geq
\exp(M^2),
\]
where the last inequality follows from $M>L^2$ and $d>10000$.

Combining the two high-probability events above, the result holds with
probability at least $1-\frac{2}{d}$.
\end{proof}

\subsubsection{Proof of \cref{thm:margin}}
\begin{proof}
Let
\[
S_0
=
M^2d/2-2M^2\sqrt{d\log(L+a)}
\]
and
\[
S
=
M^2d/2-3M^2\sqrt{d\log L}.
\]

By \cref{lm:reduction_upperbound} and \cref{lm:f_lowerbound}, with
probability at least $1-\frac{1}{L}$, there exists a feasible parameter
set such that
\begin{align}
\cL
\leq\;&
\log(1+\exp(-2M))
\nonumber\\
&+
\frac{\exp(-2M)}{1+\exp(-2M)}
\cdot
6ML\exp(-S_0)
\nonumber\\
&+
\frac{72\exp(-2M)}{1+\exp(-2M)}
\cdot
M^2L^2\exp(-2S_0).
\label{eq:global-min-upper}
\end{align}

Notice that
\[
S_0-S
=
M^2
\left(
3\sqrt{d\log L}
-
2\sqrt{d\log(L+a)}
\right)
>0
\]
when $L$ is sufficiently large.
Therefore, when $M$ is sufficiently large,
\begin{align}
&
\frac{\exp(-2M)}{1+\exp(-2M)}
\cdot
\frac{M\exp(-S)}{2L+a+b}
\cdot p
\nonumber\\
>\;&
\frac{\exp(-2M)}{1+\exp(-2M)}
\cdot
6ML\exp(-S_0)
\nonumber\\
&+
\frac{72\exp(-2M)}{1+\exp(-2M)}
\cdot
M^2L^2\exp(-2S_0).
\label{eq:loss-comparison}
\end{align}

Let
\[
\vtheta_M^*
=
(\mQ^*,\mK^*,\mV^*)
\]
be a global minimizer under the constraint $\|\vtheta\|_\infty\leq M$.
Suppose that
\[
f_L(\mQ^*,\mK^*,\mV^*)\leq S.
\]
Then, by \cref{lm:loss_lowerbound} and
\cref{eq:loss-comparison},
\begin{align*}
\cL(\vtheta_M^*)
\geq\;&
\log(1+\exp(-2M))
\\
&+
\frac{\exp(-2M)}{1+\exp(-2M)}
\cdot
\frac{M\exp(-S)}{2L+a+b}
\cdot p
\end{align*}
is strictly larger than the upper bound in
\cref{eq:global-min-upper}.
This contradicts the global optimality of $\vtheta_M^*$.
Therefore,
\[
f_L(\mQ^*,\mK^*,\mV^*)
>
M^2d/2-3M^2\sqrt{d\log L}.
\]

We next show that
\[
V_0,V_1>M-1.
\]
Suppose otherwise.
Then, by \cref{lm:Vnormal},
\[
\cL(\vtheta_M^*)
\geq
(1-p)\log(1+\exp(-2M))
+
p\log(1+\exp(-2M+2)).
\]
Moreover, by \Cref{lm:inequality},
\begin{align*}
&
(1-p)\log(1+\exp(-2M))
+
p\log(1+\exp(-2M+2))
\\
\geq\;&
\log(1+\exp(-2M))
+
\frac{2p\exp(-2M)}{1+\exp(-2M)}.
\end{align*}
When $M$ is sufficiently large, we have
\begin{align*}
\frac{2p\exp(-2M)}{1+\exp(-2M)}
>\;&
\frac{\exp(-2M)}{1+\exp(-2M)}
\cdot
6ML\exp(-S_0)
\\
&+
\frac{72\exp(-2M)}{1+\exp(-2M)}
\cdot
M^2L^2\exp(-2S_0).
\end{align*}
Thus, $\cL(\vtheta_M^*)$ would again be strictly larger than the loss
of the feasible parameter set in \cref{eq:global-min-upper}, which
contradicts the global optimality of $\vtheta_M^*$.
Therefore,
\[
V_0,V_1>M-1.
\]

Applying \cref{lm:upper_cornercase}, with probability at least
$1-\frac{2}{d}$, for every pair $(s,i)$ satisfying
$|s|\leq L^{\sqrt d/2}$,
\[
\frac{A_{s_i}(s,i)}{A_{1-s_i}(s,i)}
\geq
\exp(M^2)
\geq
2.
\]
Since
\[
\vv_{\ttC}=0
\qquad
\text{for every }
\ttC\in\cV\backslash\{0,1\},
\]
the Transformer correctly predicts $s_i$ if
\[
A_{s_i}(s,i)V_{s_i}
-
A_{1-s_i}(s,i)V_{1-s_i}
>0.
\]
Using
\[
V_{s_i}>M-1,
\qquad
V_{1-s_i}\leq M,
\]
we have
\begin{align*}
A_{s_i}(s,i)V_{s_i}
-
A_{1-s_i}(s,i)V_{1-s_i}
&>
2A_{1-s_i}(s,i)(M-1)
-
A_{1-s_i}(s,i)M
\\
&=
A_{1-s_i}(s,i)(M-2)
\\
&>0
\end{align*}
for $M>3$.
Therefore, the Transformer correctly copies every binary string with
length at most $L^{\sqrt d/2}$.

Finally, since $L>d^4$, we have $\frac{1}{L}\leq\frac{1}{d}$.
Combining the probability $1-\frac{1}{L}$ from
\cref{lm:f_lowerbound} with the probability
$1-\frac{2}{d}$ from \cref{lm:upper_cornercase}, the result holds with
probability at least
\[
1-\frac{1}{L}-\frac{2}{d}
\geq
1-\frac{3}{d},
\]
which completes the proof.
\end{proof}

\subsection{Technical Lemmas}
\label{subsec:techlemma}
In this subsection, we provide technical lemmas used in the previous subsection.

\begin{lemma}
\label{lm:inequality}
Consider the function $g(x)=\log(1+\exp(-2M+x))$.
Then for any $x>0$,
\[
g(x)\geq \log(1+\exp(-2M))+\frac{\exp(-2M)}{1+\exp(-2M)}\cdot x.
\]
Moreover, when $0\leq x\leq 0.1$, we have
\[
g(x)\leq \log(1+\exp(-2M))+\frac{\exp(-2M)}{1+\exp(-2M)}\cdot x+\frac{2\exp(-2M)}{1+\exp(-2M)}\cdot x^2
\]
\end{lemma}

\begin{proof}
Direct calculation shows that
\[
\frac{d g(x)}{dx}=\frac{\exp(-2M+x)}{1+\exp(-2M+x)},
\]
which increases with $x$.

Therefore,
\[
g(x)\geq \log(1+\exp(-2M))+\frac{d g(x)}{dx}\bigg|_{x=0}\cdot x=\log(1+\exp(-2M))+\frac{\exp(-2M)}{1+\exp(-2M)}\cdot x.
\]
On the other hand, notice that when $x\leq 0.1$, $\exp(x)\leq 1+x+x^2$. Therefore,
\begin{align*}
g(x)&\leq \log(1+\exp(-2M))+x\cdot \frac{\exp(-2M+x)}{1+\exp(-2M+x)}\\
&\leq \log(1+\exp(-2M))+x\cdot \frac{\exp(-2M)(1+x+x^2)}{1+\exp(-2M)(1+x+x^2)}\\
&\leq \log(1+\exp(-2M))+x\cdot \frac{\exp(-2M)}{1+\exp(-2M)}(1+x+x^2)\\
&\leq \log(1+\exp(-2M))+\frac{\exp(-2M)}{1+\exp(-2M)}\cdot x+\frac{2\exp(-2M)}{1+\exp(-2M)}\cdot x^2.
\end{align*}
Thus the result follows.
\end{proof}

\begin{lemma}
\label{lm:cos_chernoff}
Let $\theta_1,\ldots,\theta_m \sim \Unif[0,2\pi]$ be independent. Then for any
$\delta \in (0,1)$, with probability at least $1-\delta$,
\[
\left|\sum_{i=1}^m \cos(\theta_i)\right|
\leq \sqrt{2m\log\frac{2}{\delta}}.
\]
In particular, for any constant $c,l>0$, with probability at least $1-l^{-c}$,
\[
\left|\sum_{i=1}^m \cos(\theta_i)\right|
\leq \sqrt{2m(c\log l+\log 2)}.
\]
\end{lemma}

\begin{proof}
Let $X_i=\cos(\theta_i)$. 
Then $X_1,\ldots,X_m$ are independent, $\E[X_i]=0$, and $X_i\in[-1,1]$ for all $i$. 
By Hoeffding's inequality, for any $t>0$, \[ \Pr\left[\sum_{i=1}^m X_i \geq t\right] \leq \exp\left( -\frac{2t^2}{\sum_{i=1}^m (1-(-1))^2} \right) = \exp\left(-\frac{t^2}{2m}\right). \] 
Applying the same inequality to $-X_i$ gives \[ \Pr\left[\sum_{i=1}^m X_i \leq -t\right] \leq \exp\left(-\frac{t^2}{2m}\right). \] 
Therefore, by a union bound, \[ \Pr\left[ \left|\sum_{i=1}^m \cos(\theta_i)\right| \geq t \right] \leq 2\exp\left(-\frac{t^2}{2m}\right). \] 
Taking \[ t=\sqrt{2m\log\frac{2}{\delta}} \] proves the first claim. 
Finally, setting $\delta=l^{-c}$ gives \[ \left|\sum_{i=1}^m \cos(\theta_i)\right| \leq \sqrt{2m\log(2l^c)} = \sqrt{2m(c\log l+\log 2)} \] with probability at least $1-l^{-c}$. 
\end{proof}

\begin{lemma}
\label{lm:cos_concentration}
Let $\theta_1,\ldots,\theta_m \sim \Unif[0,2\pi]$ be independent, and define
\[
h(\theta)=\max\{|\cos\theta|,|\sin\theta|\}.
\]
Then for any $\delta\in(0,1)$, with probability at least $1-\delta$,
\[
\sum_{i=1}^m h(\theta_i)
\le
\frac{2\sqrt2}{\pi}\,m+\sqrt{\frac{m}{2}\log\frac{1}{\delta}}.
\]
\end{lemma}

\begin{proof}
Let
\[
X_i := h(\theta_i)=\max\{|\cos\theta_i|,|\sin\theta_i|\}.
\]
Then $X_1,\ldots,X_m$ are independent and satisfy $0\le X_i\le 1$.

We first compute the mean of $X_i$. By symmetry,
\[
\mathbb E[X_i]
=
\frac{4}{2\pi}\int_0^{\pi/2} \max\{\cos\theta,\sin\theta\}\,d\theta.
\]
On $[0,\pi/4]$, we have $\cos\theta\ge \sin\theta$, while on $[\pi/4,\pi/2]$, we have $\sin\theta\ge \cos\theta$. Hence
\[
\mathbb E[X_i]
=
\frac{2}{\pi}
\left(
\int_0^{\pi/4}\cos\theta\,d\theta
+
\int_{\pi/4}^{\pi/2}\sin\theta\,d\theta
\right).
\]
Evaluating the integrals gives
\[
\int_0^{\pi/4}\cos\theta\,d\theta
=
\sin(\pi/4)-\sin(0)
=
\frac{\sqrt2}{2},
\]
and
\[
\int_{\pi/4}^{\pi/2}\sin\theta\,d\theta
=
-\cos(\pi/2)+\cos(\pi/4)
=
\frac{\sqrt2}{2}.
\]
Therefore,
\[
\mathbb E[X_i]
=
\frac{2}{\pi}\cdot \sqrt2
=
\frac{2\sqrt2}{\pi}.
\]

Now apply Hoeffding's inequality to the independent random variables
$X_1,\dots,X_m\in[0,1]$. For any $t>0$,
\[
\Pr\!\left(
\sum_{i=1}^m X_i - \mathbb E\!\left[\sum_{i=1}^m X_i\right]
\ge t
\right)
\le
\exp\!\left(-\frac{2t^2}{m}\right).
\]
Set
\[
t=\sqrt{\frac{m}{2}\log\frac{1}{\delta}}.
\]
Then
\[
\Pr\!\left(
\sum_{i=1}^m X_i
\ge
\frac{2\sqrt2}{\pi}\,m+\sqrt{\frac{m}{2}\log\frac{1}{\delta}}
\right)
\le \delta.
\]
Equivalently, with probability at least $1-\delta$,
\[
\sum_{i=1}^m h(\theta_i)
=
\sum_{i=1}^m X_i
\le
\frac{2\sqrt2}{\pi}\,m+\sqrt{\frac{m}{2}\log\frac{1}{\delta}}.
\]
This proves the lemma.
\end{proof}

\begin{lemma}
\label{lm:cossum}
Let $\theta\sim\Unif[0,2\pi]$, then for any $\delta<0.1$, with probability at least $1-\delta$, for any positive integer $s<t$,
\[
\left|\sum_{k=s}^t \cos(k\theta)\right|,\left|\sum_{k=s}^t\sin(k\theta)\right|\leq \frac{3}{\delta}.
\]
\end{lemma}

\begin{proof}
For each $m$, we have
\begin{align*}
\left|\sum_{k=s}^{t} \cos(k\theta)\right|&=\left|\frac{\sum_{k=s}^t \left(\sin\left(\frac{2k+1}{2}\theta\right)-\sin\left(\frac{2k-1}{2}\theta\right)\right)}{2\sin(\theta/2)}\right|\\
&=\left|\frac{\sin\left(\frac{2t+1}{2}\theta\right)-\sin\left(\frac{2s-1}{2}\theta\right)}{2\sin(\theta/2)}\right|\\
&\leq \frac{1}{|\sin(\theta/2)|}.
\end{align*}
Similarly,
\begin{align*}
\left|\sum_{k=s}^{t} \sin(k\theta)\right|&=\left|\frac{\sum_{k=s}^t \left(\cos\left(\frac{2k-1}{2}\theta\right)-\cos\left(\frac{2k+1}{2}\theta\right)\right)}{2\sin(\theta/2)}\right|\\
&=\left|\frac{\cos\left(\frac{2s-1}{2}\theta\right)-\cos\left(\frac{2t+1}{2}\theta\right)}{2\sin(\theta/2)}\right|\\
&\leq \frac{1}{|\sin(\theta/2)|}.
\end{align*}

As $\theta$ is uniformly sampled, with probability at least $1-\delta$, we have $|\sin(\theta/2)|\leq \delta/3$.
Therefore,
\[
\left|\sum_{k=s}^t \cos(k\theta)\right|,\left|\sum_{k=s}^t\sin(k\theta)\right|\leq \frac{3}{\delta},
\]
and the result follows.
\end{proof}

%% file: appendix-prompt-templates.tex
\newpage
\section{Prompt Templates of Copy Test for Frontier LLMs}
\label{appendix:prompts}

This section lists the prompt templates used for evaluating frontier LLMs on the copy tests. 
For all tasks, the model is instructed to output only the copied sequence, with no explanation or additional formatting.

\subsection{Recursive-Flip Binary Copy}

For recursive-flip binary copying task, the prompt template consists both system prompt and user prompt as shown below.

The following is the system prompt:
\begin{Verbatim}[
    fontsize=\footnotesize,
    breaklines=true,
    breakanywhere=true
]
"You are taking a copying test.\nYour task is to copy the binary sequence exactly.\nOutput only the copied binary sequence.\nDo not add any explanation, quotes, punctuation, or formatting.\nDo not add spaces or newlines inside the sequence.\n"
\end{Verbatim}

The following is the user prompt:
\begin{Verbatim}[
    fontsize=\footnotesize,
    breaklines=true,
    breakanywhere=true
]
"Copy the following binary sequence exactly:\n\n{s}"
\end{Verbatim}
where \texttt{\{s\}} is the binary string to be copied.

\subsection{Imbalanced Binary Copy}

For the imbalanced binary copy task, the prompt template consists both system prompt and user prompt as shown below.

The following is the system prompt:
\begin{Verbatim}[
    fontsize=\footnotesize,
    breaklines=true,
    breakanywhere=true
]
"You are taking a copying test.\nYour task is to copy the binary sequence exactly.\nOutput only the copied binary sequence.\nDo not add any explanation, quotes, punctuation, or formatting.\nDo not add spaces or newlines inside the sequence.\n"
\end{Verbatim}

The following is the user prompt:
\begin{Verbatim}[
    fontsize=\footnotesize,
    breaklines=true,
    breakanywhere=true
]
"Copy the following sequence exactly:\n\n{s}"
\end{Verbatim}
where \texttt{\{s\}} is an imbalanced sequence over \texttt{a} and \texttt{b}. The target output is exactly \texttt{\{s\}}.

\subsection{Python List Conversion}

For the Python list conversion task, each input consists of a comma-separated sequence of sensor values. The model is asked to output the same sequence formatted as a single-line Python list. During evaluation, we uniformly sample one prompt from the following list:
\begin{lstlisting}[style=promptlist]
[
    "Copy the following sensor sequence exactly. Output only one single-line Python list, with no line breaks, no spaces, and no extra text:",
    "Rewrite the following sensor measurements as exactly one single-line Python list. Do not add any line breaks, spaces, explanation, or extra characters:",
    "Return the following numbers as one single-line Python list only. Keep the values and order exactly the same. No line breaks, no spaces, no extra text:",
    "Convert the following sensor readings into a Python list on a single line. Output only the list, with no spaces, no line breaks, and no other text:",
    "Repeat the following sequence exactly as a one-line Python list. Do not change any number. Do not insert line breaks, spaces, or commentary:",
    "Write the following measurements as exactly one Python list in a single line. Output only the list itself, with no spaces, no newlines, and no explanation:",
    "Copy the following numeric sequence exactly into a single-line Python list. Preserve every value and its order. No line breaks, no spaces, no extra text:",
    "Output exactly one single-line Python list containing the following numbers in the same order. Do not add spaces, newlines, or any surrounding text:",
    "Turn the following sensor values into a Python list written on one line only. Keep the sequence exactly unchanged. No spaces, no line breaks, no extra text:",
    "Produce exactly the following sequence as a single-line Python list. Output only the list, with no spaces, no line breaks, and no additional content:",
]
\end{lstlisting}

Then, we use such input:
\begin{Verbatim}[
    fontsize=\footnotesize,
    breaklines=true,
    breakanywhere=true
]
"{prompt}:\n{s}"
\end{Verbatim}
Here \texttt{{prompt}} is the sampled prompt and \texttt{\{s\}} is the comma-separated input sequence. If \texttt{\{s\}} is
\texttt{1.55,1.71,\ldots}, then the gold output is
\texttt{[1.55,1.71,\ldots,]}.

\subsection{Repeat-Structure Copy Test}

For the repeat-structure copy test, we use the following prompt:
\begin{Verbatim}[
    fontsize=\footnotesize,
    breaklines=true,
    breakanywhere=true
]
"Copy the following sequence exactly.\nOutput only the copied sequence and nothing else.\nDo not add explanations, quotation marks, or newlines.\n\nSequence:{s}"
\end{Verbatim}
where \texttt{\{s\}} is the corrupted repeat-structure sequence. The gold output is exactly \texttt{\{s\}}.

\section{Prompt Templates for Finetuning Data}

\newsavebox{\mybox} 

Following is the chat template used for finetuning experiments in Section~\ref{sec:llm-experiments}:
\begin{tcolorbox}
\ttfamily\small
\noindent\textbf{\#\#\# user:}\\
\textcolor[RGB]{163,89,121}{\{input\}}

\vspace{0.8em}

\noindent\textbf{\#\#\# assistant:}\\
\textcolor[RGB]{163,89,121}{\{output\}}
\end{tcolorbox}
where \textcolor[RGB]{163,89,121}{\texttt{\{input\}}} is the user prompt, and \textcolor[RGB]{163,89,121}{\texttt{\{output\}}} is the language model response. The response is always the sequence that that needs to be copied, and nothing else. The user prompt is constructed with a set of prompt templates that resembles natural language instructions that a user would send to an AI assistant when the users want the assistant to copy a sequence. This set of prompt templates are listed below as a Python list. They are uniformly and randomly sampled from when generating finetuning data. The \textcolor[RGB]{163,89,121}{\texttt{\{s\}}} indicates the portion that is to be replaced with the binary string to be copied.

\begin{lstlisting}[style=promptlist]
[
    # With newline
    "Please echo the next sequence verbatim (no extra text). Here it is:\n{s}",
    "Copy the text below exactly as-is (keep every character):\n{s}",
    "Task: return the string shown below. Do not add or remove anything.\n{s}",
    "Print the following content exactly, unchanged. Output only the content:\n{s}",
    "Write the following string exactly as it is:\n{s}",
    "I want you to say the following string and nothing else:\n{s}",
    "Say the following string exactly as it is:\n{s}",
    "Here is a string:\n{s}\n\nPlease print the string exactly as it is.",
    "I will give you a string. You need to repeat it exactly as it is. Here it is:\n{s}",
    "{s}\n\nPlease print the above string exactly as it is.",
    "{s}\n\nCan you write out the above string and nothing else.",
    "Can you copy a string of characters? Here is the string I want you to copy:\n{s}",
    "Repeat this string and say nothing else:\n{s}",

    # Without newline
    "Repeat the following string after me (output only the string): {s}",
    "Echo this sequence exactly (no extra characters): {s}",
    "Return the exact same string and nothing else -> {s}",
    "Copy-paste the following payload exactly as written: {s}",
    "Output EXACTLY the string below, unchanged: {s}",
    "Just say \"{s}\" and nothing else.",
    "Say '{s}' and say nothing else.",
]
\end{lstlisting}

\newpage